\documentclass[11pt]{article}

\usepackage[T1]{fontenc}
\usepackage{microtype}
\usepackage{algorithm, algorithmic}

\usepackage[colorinlistoftodos]{todonotes}

\usepackage{dsfont}

\usepackage{environ}

\usepackage{amsthm,amsfonts,amsmath,amssymb,epsfig,color,float,graphicx,verbatim, enumitem, bm}
\renewcommand{\bm}{}
\usepackage{multirow}
\usepackage{algorithm}
\usepackage{caption}
\usepackage{subcaption}
\usepackage{threeparttable}

\usepackage{hyperref}
\hypersetup{
  colorlinks = true,
  urlcolor = {cpurple},
  citecolor = {cpurple}
}
\usepackage{cleveref}
\usepackage{bbm, dsfont}
\definecolor{cpurple}{rgb}{0.6,0,0.6}

\usepackage{enumitem}

\newtheorem{theorem}{Theorem}[]

\newtheorem{lemma}{Lemma}[]
\newtheorem{informal theorem}[theorem]{Theorem (informal statement)}

\newtheorem{corollary}[theorem]{Corollary}
\newtheorem{claim}[theorem]{Claim}

\newtheorem{remark}[theorem]{Remark}

\newtheorem{definition}{Definition}

\newtheorem{question}[theorem]{Question}
\newcommand{\eqdef}{\stackrel{{\mathrm {\footnotesize def}}}{=}}

\usepackage{framed}
\usepackage{nicefrac}

\def\nnewcolor{0}
\ifnum\nnewcolor=1
\newcommand{\new}[1]{{\color{red} #1}}
\newcommand{\inote}[1]{\footnote{{\bf [[Ilias: {#1}\bf ]] }}}
\newcommand{\dnote}[1]{\footnote{{\bf [[Daniel: {#1}\bf ]] }}}
\newcommand{\snote}[1]{\footnote{{\bf [[Sihan: {#1}\bf ]] }}}
\newcommand{\hnote}[1]{\footnote{{\bf [[Hanshen: {#1}\bf ]] }}}
\fi
\ifnum\nnewcolor=0
\newcommand{\new}[1]{#1}
\newcommand{\inote}[1]{}
\newcommand{\dnote}[1]{}
\newcommand{\snote}[1]{}
\newcommand{\hnote}[1]{}
\fi

\newcommand{\R}{\ensuremath{\mathbb R}}
\newcommand{\Z}{\ensuremath{\mathbb Z}}

\newcommand{\pr}{{\bf {\rm Pr}}}

\newcommand{\floor}[1]{\ensuremath{\left\lfloor#1\right\rfloor}}
\newcommand{\poly}{\operatorname{poly}}

\newcommand{\E}{\mathbf{E}}

\newcommand{\junk}[1]{}

\newcommand{\eps}{\epsilon}

\newcommand\snorm[2]{\left\| #2 \right\|_{#1}}

\newcommand{\abs}[1]{\left | #1 \right |}

\newcommand{\lp}{\left}
\newcommand{\rp}{\right}

\newcommand{\cov}{\textbf{Cov}}

\newlength\myindent
\newcommand{\iid}{\text{i.i.d.}}

\newcommand{\Binary}{\text{Binary}}

\newcommand{\erf}{\text{erf}}

\DeclareMathOperator{\argmax}{argmax}

\DeclareMathOperator{\LRS}{Left}
\DeclareMathOperator{\RRS}{Right}

\newcommand{\cl}{\preccurlyeq}

\newcommand\blfootnotea[1]{%
  \begingroup
  \renewcommand\thefootnote{}\footnote{#1}%
  \endgroup
}

\title{Online Robust Mean Estimation\blfootnotea{Author last names are in randomized order.}}

\author{
Daniel M. Kane\thanks{Supported by NSF Medium Award CCF-2107547, NSF Award CCF-1553288 (CAREER),  and a grant from
CasperLabs.}\\
UC San Diego\\
{\tt dakane@ucsd.edu }\\
\and
Ilias Diakonikolas\thanks{Supported by NSF Medium Award CCF-2107079,
a Sloan Research Fellowship, and a DARPA Learning with Less Labels (LwLL) grant.}\\
UW Madison\\
{\tt ilias@cs.wisc.edu}\\
\and
Hanshen Xiao\\
MIT\\
{\tt hsxiao@mit.edu } \\
\and
Sihan Liu\\
UC San Diego\\
{\tt sil046@ucsd.edu}
}
\date{}

\begin{document}

\maketitle

\begin{abstract}
We study the problem of high-dimensional robust mean estimation in an online setting. 
Specifically, we consider a scenario where $n$ sensors are measuring some common, ongoing phenomenon. 
At each time step $t=1,2,\ldots,T$, the $i^{th}$ sensor reports its readings $x^{(i)}_t$ for that time step. 
The algorithm must then commit to its estimate $\mu_t$ for the true mean value 
of the process at time $t$. We assume that most of the sensors observe 
independent samples from some common distribution $X$, 
but an $\eps$-fraction of them may instead behave maliciously. 
The algorithm wishes to compute a good approximation $\mu$ to the true mean $\mu^\ast := \E[X]$.
We note that if the algorithm is allowed to wait until time $T$ to report its estimate, 
this reduces to the well-studied problem of robust mean estimation. 
However, the requirement that our algorithm produces partial estimates 
as the data is coming in substantially complicates the situation.

We prove two main results about online robust mean estimation in this model. 
First, if the uncorrupted samples satisfy the standard condition of $(\eps,\delta)$-stability, 
we give an efficient online algorithm that outputs estimates $\mu_t$, $t \in [T],$ 
such that with high probability it holds that $\|\mu-\mu^\ast\|_2 = O(\delta \log(T))$, 
where $\mu = (\mu_t)_{t \in [T]}$.
We note that this error bound is nearly competitive with the best offline algorithms, 
which would achieve $\ell_2$-error of $O(\delta)$. 
Our second main result shows that 
with additional assumptions on the input (most notably that $X$ is a product distribution) 
there are inefficient algorithms whose error does not depend on $T$ at all.
\end{abstract}

\setcounter{page}{-1}

\thispagestyle{empty}

\newpage

\tableofcontents

\thispagestyle{empty}

\newpage

\section{Introduction} \label{sec:intro}

\subsection{Motivation and Background} \label{ssec:background}

One of the most fundamental problems in statistics is that of mean estimation: 
given a collection of $n$ i.i.d.\ samples drawn from an unknown distribution $X$ 
assumed to lie in some known distribution family $\mathcal{F}$, 
the goal is to output an accurate estimate of the unknown mean $\mu^{\ast}$ of $X$.
While this vanilla setting is fairly well understood, it does not capture 
a number of practically pressing real-world scenarios, where 
(i) due to modeling issues, the underlying distribution $X$ we sample from 
does not lie in the known family $\mathcal{F}$ but is only close to it, and 
(ii) a fraction of the samples are arbitrarily corrupted by malicious users.

The field of \emph{robust statistics} aims to design estimators 
that can tolerate up to a \emph{constant} fraction of corruptions, 
independent of the data dimensionality~\cite{Tukey60,Huber64, Huber09}.
Classical works in the field have identified the statistical limits of several problems 
in the robust setting, both in terms of constructing robust estimators and proving 
information-theoretic lower bounds~\cite{Yatracos85,DonLiu88a,Donoho92,Huber09}. 
However, the early estimators proposed in the statistics literature
were not computationally efficient, typically requiring exponential running time 
in the number of dimensions, see, e.g.,~\cite{Bernholt, Huber09}.

A relatively recent line of work, originating in computer science~\cite{DKKLMS16, LaiRV16}, 
has developed the field of \emph{algorithmic} high-dimensional robust statistics, 
aiming to design estimators that not only attain tight robustness 
guarantees, but are also efficiently computable.
This line of research has provided computationally efficient estimators 
for a variety of statistical tasks, including mean and covariance estimation, 
linear regression, and many others, 
under natural distributional assumptions on the uncorrupted data; 
see \cite{DK19-survey, DK23-book} for an overview of this area.

This recent progress notwithstanding, the vast majority of the recent literature on algorithmic 
robust statistics focuses on the {\em offline} setting, where the (corrupted) dataset is given in the input
and the goal is to produce a {\em single} accurate estimate. For example, in (offline) robust mean estimation,
we are given a dataset of $n$ points in $\R^M$, an $\eps$-fraction of which are corrupted, 
and the goal is to estimate the mean of the distribution that generated the uncorrupted samples.

The aforementioned offline setting fails to model some commonly arising situations.
First, we may need to produce estimations for a series of related statistical tasks 
that come in sequentially.
Second, we are often able to identify the \emph{providers} of the data.
This can be modeled abstractly as follows.
Consider the scenario that we have $n$ sensors over which an $\eps$-fraction may be hijacked by an adversary or simply malfunctioning. 
These sensors are collecting information about some common, ongoing stochastic process. 
In particular, if the stochastic process has $T$ stages, we can model it mathematically as a $T$-dimensional distribution $X$ such that $X_t$ encodes the state of the process at time step $t$.
Then, at each time step $t$, each uncorrupted sensor give us a report which is an i.i.d.\ sample from $X_t$ and the corrupted ones may give some arbitrary out-of-distribution reports.
Our goal is then to compute some statistics related 
to $X_t$ {\em at each time step} given the reports received so far.
A concrete scenario is described below.

\paragraph{Online Decision Making with User Feedback}
A company is trying to deploy a series of new features. 
Before deployment, a random set of users are selected for trials. 
After the trial session of each feature ends, 
the development team needs an estimate of a typical user's rating to the feature to
decide whether it is ready for public deployment.
While most feedbacks from the trial users probably do follow a stochastic pattern, 
some may be significantly ``out of distribution''. 
For example, they may originate from a non-typical user who has special demands 
or even a fake user account registered by competitors.
Ideally, we woudly like to \emph{identify} these outlier users 
so as to minimize their \emph{total} impact to our estimations in the long run.

Indeed, \new{similar scenarios arise whenever we face a \emph{sequence} of statistical estimation tasks 
which share the \emph{same} set of data providers that may not be completely trustworthy.
Though the statistical tasks themselves may be independent of each other, the underlying statistical estimation algorithms should not run independently as that will allow adversarial data providers to disturb the outcomes in \emph{every} estimation task.
The more favorable way is always to get rid of the suspicious data providers during early tasks so as to minimize their influence in the future. 
}

\new{At a more philosophical level, we aim at providing a mathematical framework through which one can develop algorithmic ways to establish \emph{trust} over different information sources over time.}
Almost on a daily basis we are required to make decisions or judgement{s} 
based on information collected from different channels, such as social media, 
television or even gossip. 
How much we believe a new story 
we hear may depend upon the degree 
to which we trust the source (based on our judgement of previous data from that source) 
and on how consistent the story is with others.

In this work, we make a concrete step in formulating such scenarios. 
Specifically, we define and study a natural notion of high-dimensional 
robust mean estimation in the online setting.

\vspace{-0.2cm}

\paragraph{Online Robust Mean Estimation: Problem Setup}
Throughout this work, we consider the standard strong contamination model.

\begin{definition}[Strong Contamination Model] \label{def: strong-contamination}
Given a parameter $0 \leq \eps <1/2$ 
and a set $\mathcal{C}$ of $n$ samples, 
the strong contamination adversary operates as follows. 
After observing the entire set $\mathcal{C}$, the adversary 
can remove up to $\eps n$ samples from $\mathcal{C}$ 
and replace them by arbitrary points. 
The resulting set $\mathcal{X}$ is called an $\eps$-corrupted version of $\mathcal{C}$.
\end{definition}

We are now ready to define our notion of
{\em robust online mean estimation.}
Intuitively, our goal is to model the scenario
where a series of mean estimation tasks 
need to be completed sequentially, 
using data collected from a set of sensors 
over which $\eps$-fraction are either hijacked or malfunctioning.
In more detail, we introduce the following definition.
\begin{definition}[Online Mean Estimation under Strong Contamination] \label{def:orm}
Given $M, T \in \Z^+$, such that $M$ is an integer multiple of $T$, 
and $0\leq \eps <1/2$, 
let $\mathcal{X}= \{x^{(1)}, \cdots, x^{(n)}\}$ be an $\eps$-corrupted version 
of a clean set of \iid~samples from a distribution $X$ on $\R^M$
with unknown mean $\mu^{*}$. 
The $M$ coordinates of each datapoint
are divided into $T$ batches, 
each of size $d \eqdef M/T$ \footnote{We remark that we require the division to be even only for convenience. The model can be generalized to work with any kind of partition depending on specific application scenarios and most of our algorithmic ideas are still applicable.}, 
i.e., $x^{(i)}$ is the concatenation of $x^{(i)}_1, \cdots, x^{(i)}_T$, 
where $x^{(i)}_t \in \R^d$, $t \in [T]$. 
The interaction with the learner proceeds in $T$ rounds as follows:
\begin{enumerate}[leftmargin=*]
\item In the $t$-th round, the $t$-th batch of coordinates 
$x^{(1)}_t, \cdots x^{(n)}_t \in \R^{d}$ are revealed. \new{(See Figure~\ref{fig:divide} for an illustration of this process).}
\item After the $t$-th round, the algorithm is required 
to output $\mu_t \in \R^{d}$ as an estimate of $\mu_t^{\ast}$ -- the $t$-th batch of coordinates of $\mu^{\ast}$.
\end{enumerate}
At the end of this process, we say that 
the algorithm estimates the mean of $X$ under $\eps$-corruption 
in the $T$-round online setting with error $\eps' > 0$, 
failure probability $\tau \in (0,1)$ and sample complexity $n$, 
if with probability at least $1 - \tau$ the following holds 
$$\snorm{2}{ \mu - \mu^{\ast} } = 
\sqrt{ \sum_{t=1}^T \snorm{2}{ \mu_{t} - \mu_{t}^{\ast} }^2 } \leq \eps' \;.$$
\end{definition}

\begin{figure}[H]
\centering
\includegraphics[width=8cm]{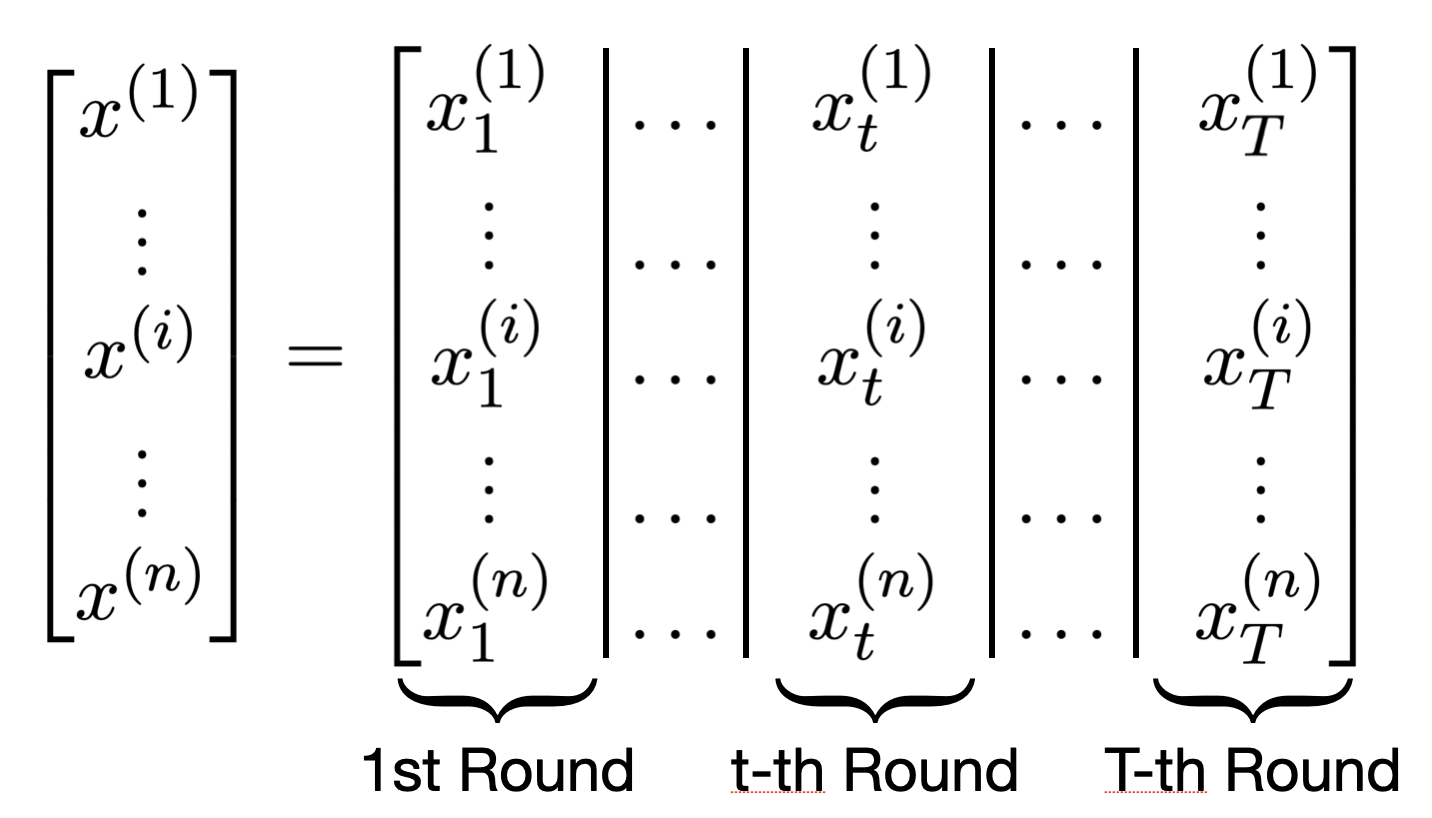}
\caption{Each sample $x^{(i)} \in \R^{M}$ collected is divided into $T$ sub-vectors 
$x^{(i)}_1 ,  \cdots,  x^{(i)}_T  \in \R^{d}$ for $d = M/T$. 
Only the $x^{(i)}_t$'s are revealed in the $t$-th round.}
\label{fig:divide}
\end{figure}

Before we proceed, some remarks are in order.
We start by noting that the task of online mean estimation without contamination
(corresponding to the special case of $\eps=0$ in Definition~\ref{def:orm})
is not significantly more difficult than offline mean estimation. Indeed, in the noise-free setting, one can simply compute the sample mean and computing the $t$-th coordinate of the sample mean only requires the $t$-th coordinate of each sample. 
The contamination, however, dramatically complicates the situation. 
Specifically, all known robust mean estimators (even inefficient ones!) 
--- including the Tukey median and its generalizations~\cite{Tukey75}
or filtering based methods~\cite{DKKLMS16, DK19-survey, DK23-book} --- 
that achieve dimension-independent errors 
require looking at all coordinates of the sample at the same time.
Prior to the current work, even the information-theoretic
aspects of online robust  mean estimation were not understood
(i.e., what is the optimal error achievable 
when the sample size goes to infinity without computational considerations).
Second, we remark that our formulation allows the mean estimation tasks across different time steps to be \emph{correlated}.
In particular, even though the sampling process between 
any two non-adversarial sensors are independent, 
the sampling results 
of one sensor (adversarial or not) at different time steps, namely the variables $x_t^{(i)}$ and $x_{t'}^{(i)}$ for $t \neq t'$, can be correlated.

In the rest of this section, we provide additional 
motivation for our robust online distribution learning model.

\paragraph{Federated Learning under Byzantine Failure}
Federated learning is the practice of training an ML model {in a distributed fashion} 
on multiple decentralized worker devices containing local data;  
see \cite{kairouz2021advances} for an overview of the field. 
The typical framework is the following.
At the $t$-th round, the central server broadcasts $w^{(t)} \in \R^d$ 
-- the parameters of the central model -- to all the worker devices. 
Then, each worker device makes updates to the central model 
received with their local data and sends back to the central 
server $w^{(t)}_i \in \R^d$ -- the parameters of the updated local model. 
After receiving the responses from all local devices, 
the central server updates the central model by aggregating all the local models, 
{producing a new estimate} $ w^{(t+1)} \leftarrow Aggregate( w_i^{(t)} ) $. 
A commonly used aggregation rule is called the \textsf{FederatedAveraging} algorithm: 
the central server simply computes the arithmetic 
mean of the updates~\cite{mcmahan2017communication}.
One then iterates the training process until the central model reaches 
high accuracy on some validation set prepared in advance.

The distributed nature of the learning framework makes the task 
particularly vulnerable to {\em Byzantine failures}~\cite{lamport2019byzantine} --- a subset 
of malicious machines that behave adversarially in the computing network. 
As noted in the work of~\cite{pillutla2019robust}, 
the \textsf{FederatedAveraging}~algorithm, despite being one of the most 
commonly used aggregation protocols in practice, 
is especially vulnerable to Byzantine errors: 
even if only one worker device is controlled by the adversary, 
it can ruin the entire training process 
by giving wildly off local parameters in just one iteration.

The inherently unpredictable and possibly colluding adversarial behaviors 
of the Byzantine devices make them hard or even impossible to distinguish. 
This is especially true in an online or iterative learning procedure, 
including that of learning from streaming data or running distributed 
SGD, see, e.g.,~\cite{chen2017distributed}, \cite{blanchard2017machine}.

Proposed solutions usually involve performing robust estimation 
at each iteration \emph{independently} \cite{pillutla2019robust}, \cite{li2019rsa}. 
This means that, though the adversary cannot corrupt the aggregation by too much 
at a single round, it can steadily create consistent errors 
in the training process.
This scenario motivates our setup of considering robust mean estimation 
of multiple rounds in a holistic manner. In particular, our goal 
is to minimize the {\em total error} incurred in all rounds. 
As we will see later, it is indeed possible to design an efficient estimator 
such that the total errors only grow logarithmically with the number of rounds. 
This then opens up the hope of limiting the influence of Byzantine failures on online systems in the long run.

The above discussion illustrates two key principles 
for dealing with untrustworthy data. On the one hand, 
we can use \emph{outlier detection}, to flag datapoints that might be erroneous. 
On the other hand, for sources that have been around for a while, 
we can additionally develop \emph{trust} in a source 
based on the accuracy of previous predictions. 
In this paper, we will see how the interplay of these ideas 
can be used to maintain accurate estimates during ongoing data collection.

\subsection{Our Results}
We study the problem of high-dimensional online robust mean estimation 
in the setting of Definition~\ref{def:orm}. 
Our main results consist of 
(i) a computationally efficient robust online algorithm 
that achieves nearly optimal error rate (up to a factor of $\log T$, where $T$ is the number of rounds), 
and (ii) an inefficient robust online estimator that achieves 
the information-theoretically optimal error (within a constant factor). 

{Our} first main result is a statistically and computationally efficient robust online algorithm 
that works generically for families of distributions commonly studied 
in the robust statistics literature (see \Cref{thm:filter_main} for a more general statement).

\begin{theorem}[Efficient Online Robust Mean Estimation] \label{thm:main1}
Let $\eps \leq \eps_0$ for a sufficiently small universal constant $\eps_0>0$.
Suppose $X$ is an $M$-dimensional distribution 
with unknown mean vector $\mu^{\ast} \in \R^M$.
There exists a computationally efficient algorithm 
which robustly estimates the mean of $X$ 
under $\eps$-corruption in the $T$-round online setting 
with failure probability $1/10$, 
sample complexity $n = \poly(M, 1/\eps)$, 
and achieves the following error guarantees:
\begin{itemize}[leftmargin=*]
    \item If $X$ has unknown identity-bounded covariance (i.e., $\Sigma_X \preceq I$), 
    then the algorithm achieves error $O(\sqrt{\eps} \log T)$.
    \item If $X$ is subgaussian with identity covariance, 
    then the algorithm achieves error $O(\eps \sqrt{ \log(1/\eps)} \log T )$.
\end{itemize}
\end{theorem}

Note that except for the $\log(T)$ factors above, 
the error bounds in Theorem \ref{thm:main1} are optimal even for offline algorithms, 
i.e., algorithms allowed to observe the entire sample set $\mathcal{X}$ before 
having to make any predictions. 
\new{Moreover, even though it is not explicitly 
specified in our theorem statements, we note that the 
sample complexity of our algorithm is near-optimal, 
matching that of the best known offline algorithm.}

It is natural to ask whether this extra factor of $\log T$ is information-theoretically 
necessary for online robust mean estimation. In our second main contribution, 
we show that the $\log T$ factor can be removed
for certain families of {\em product} distributions 
(albeit using an inefficient algorithm). See Theorem~\ref{thm:optimal-gaussian}, 
Corollaries~\ref{cor:bounded-k},~\ref{cor:subgaussian} for details, 
and \Cref{thm:f-tail} for a more general statement.

\begin{theorem}[Optimal Error for Product Distributions] \label{thm:products-intro}
Let $\eps \leq \eps_0$ for a sufficiently small universal constant $\eps_0>0$.
Fix two positive integers $M, T$ such that $T$ divides $M$.
Suppose $X$ is an $M$-dimensional product distribution.
Then there exists an (inefficient) algorithm which robustly estimates the mean of $X$ 
under $\eps$-corruption in the $T$-round online setting 
with failure probability $1/10$, sample complexity
$ n =  2^{M} \cdot \poly( M, 1/\eps )$,
and achieves the following error guarantees:
\begin{itemize}[leftmargin=*]
\item If  $X$ is a Gaussian with identity covariance, 
     then the algorithm achieves error $O( \eps  )$.
    \item If  $X$ has bounded $k$-th moments for $k\geq4$, 
            then the algorithm achieves error $O(\eps^{1-1/k} )$.
    \item If $X$ is subgaussian with identity covariance,  
           then the algorithm achieves error $O(\eps\sqrt{ \log(1/\eps)})$.
\end{itemize}
\end{theorem}

Finally, we obtain a generalization of Theorem~\ref{thm:products-intro} 
that allows the independence of coordinates assumption to be slightly relaxed. 
In particular, we assume there is an unknown distribution $X_t$ chosen for the $t$-th round 
and the overall distribution $X$ is exactly the product of these $T$ unknown distributions.
For this relaxed setting, we establish the following 
(see Theorem~\ref{thm:block-extension} for a more general statement).

\begin{theorem}[Optimal Error for Round-wise Independent Distributions] \label{thm:block-intro}
Let $\eps \leq \eps_0$ for a sufficiently small universal constant $\eps_0>0$.
Let $X_1, \cdots, X_T$ be $T$ unknown $d$-dimensional distributions.
Suppose $X$ is the product distribution of $X_1, \cdots, X_T$.
Then there exists an (inefficient) algorithm which robustly estimates the mean of $X$ 
under $\eps$-corruption in the $T$-round online setting 
with failure probability $1/10$, sample complexity
$n = 2^{O(Td^2)} \cdot \poly( 1/\eps )$, 
and achieves the following error guarantees:
\begin{itemize}[leftmargin=*]
    \item If  each $X_i$ has bounded $k$-th moments for $k\geq4$, 
     then the algorithm achieves error $O( \eps^{1-1/k} )$.
    \item If each $X_i$ is subgaussian with identity covariance, 
     then the algorithm achieves error $O( \eps \sqrt{ \log(1/\eps) } )$.
\end{itemize}
\end{theorem}

We remark that the aforementioned assumption on the distribution $X$
is a more general condition than the assumption
that the coordinates of $X$ are mutually independent. 
In particular, this assumption holds as long as the estimation tasks 
{\em for different rounds} are independent of each other.

\subsection{Overview of Techniques} \label{ssec:techniques}

The starting point of our efficient online algorithm is 
the weighted filtering algorithm for the offline robust mean estimation problem. 
The (offline) filtering algorithm works by assigning each sample a non-negative weight (initially set at $1/n$) 
that expresses our confidence that it is uncorrupted. 
Then, assuming that the uncorrupted samples satisfy 
a high probability \emph{stability} assumption 
(see Definition \ref{def: stability}), it applies a polynomial-time 
filtering technique in order to de-weight the worst outliers. 
In particular, assuming stability of the uncorrupted points, 
this filtering algorithm has two important properties. 
First, the total weight removed from uncorrupted points 
is at most the weight removed from corrupted ones 
(thus guaranteeing that at most $O(\eps)$ weight is removed overall). 
Second, after applying the filter, the weighted mean of the samples 
will be close to the true mean.

Our efficient algorithm essentially maintains an online version of this filter. 
This requires some new ideas that we explain in the proceeding discussion.
In the online setting, we maintain a set of weights 
for each sample along with that sample's
currently revealed coordinates. In each round, we add the information 
about the newly revealed coordinates to each sample and re-apply our filter. 
We then return the weighted sample mean as our estimate 
for the mean for the newest block of coordinates. 
In particular, letting $\bm{w}_t$ be the weight vectors 
at the end of the $t^{th}$ round and $\mu_t(\bm{w})$ 
be the average of the first $t$ blocks of data using weights $\bm{w}$, 
our algorithm's estimate for the $t^{th}$ block of coordinates 
is just the $t^{th}$ block of $\mu_t(\bm{w}_t)$.

Now we know from the standard properties of the (offline) filter 
that $\|\mu_t(\bm{w}_t) - \mu^\ast_t\|_2 = O(\delta)$, 
where $\mu^\ast_t$ is the first $t$ blocks of the true mean 
and $\delta$ is the stability parameter. 
Unfortunately, this property alone does not suffice: 
it could be the case that $\mu_t(\bm{w}_t)$ agrees exactly 
with $\mu_t^\ast$ in all but the $t^{th}$ block of coordinates, 
in which they differ by $\delta$. This would leave us with $\delta$ error 
on each block of coordinates, for a total error of $\delta \sqrt{T}$ finally. 
To avoid this possibility, we need a new and more subtle structural
property of the filter algorithm. 
In particular, we show (see Lemma   \ref{lemma:bounded-second-difference}) 
that for any $t'>t$ it holds 
$\sum_{j=1}^t \|\mu_j({\bm w}_{t}) - \mu_j(\bm{w}_{t'})\|_2^2 = O(\delta^2 \|\bm{w}_t-\bm{w}_{t'}\|_1/\eps).$ 
Since the total change in the weight vectors 
throughout the entire run of the algorithm is bounded by $O(\eps)$, 
this lemma allows us to show that once we assign values 
to a new block of coordinates, they cannot be changed too much by future reweightings. 
This property along with a careful recursive argument gives our final error bound of $O(\delta \log(T))$.

We now discuss the ideas behind our optimal error (inefficient) estimator. 
We start with the special case of binary product distributions. 
The high-level framework is the following. 
At the $t$-th round, the algorithm divides the samples into groups 
based on the revealed coordinates of the sample 
in the previous $t-1$ rounds (thus producing a total of $2^{t-1}$ groups). 
Equivalently, the samples within a group in some round 
will be divided into two child groups for the next round, 
based on the newly revealed coordinates. 
Given these groups, we then compute the mean 
of the $t^{th}$ coordinates of the samples in each group, 
and use as our final estimate the weighted median 
of these group means (weighted by group size). 
The robustness of this algorithm mainly follows from two observations. 
First, in each round, if the final estimation is $\eta$-far from the true mean, 
it must be the case that at least half of the group estimations 
(weighted by group size) are at least $\eta$-far from the true mean. 
Second, if the mean of a group is far from the true mean, 
then the adversarial samples must be divided unevenly 
among the two child groups in the next round. 
Consequently, as the algorithm accumulates more errors, 
the adversarial samples will become increasingly concentrated 
among a small fraction of groups. Since the final estimation 
is the median of all group estimations, 
it will become harder and harder to get the adversarial examples 
to influence the final mean. To formalize this intuition, 
we define a potential function, 
which is roughly the sum of the squares of the ``adversarial densities'' 
in each group weighted by its relative size 
(see Equation~\eqref{eq:potential-def} for further details). 
In particular, we show (see \Cref{lem:potential-increment-var}) 
that if the algorithm produces an error of $\eta$ in the $t^{th}$ round, 
then the potential function must increase by $\Omega(\eta^2)$ 
between rounds $t$ and $t+1$. 
Combined with the fact that the potential can never exceed $O(\eps^2)$, 
this implies that this algorithm produces $\ell_2$-error of at most $O(\eps)$. 
A slight refinement of this argument shows that if each coordinate 
is known to have mean at most $\gamma$, then 
we can obtain error $O(\min(\eps,\sqrt{\gamma\eps}))$.

To obtain an online robust mean estimator for other families of product distributions, 
we use a reduction to the case of binary product distributions. 
In particular, if we define the indicator variables $Y(q)_t := \mathbbm 1\{ X_t < q \}$, 
we note that for any $q$ that $Y(q)$ is a binary product distribution 
with mean $\E[Y(q)]_t = \pr(X_t < q)$. Applying our binary product estimator to $Y(q)$, 
we can obtain relatively good estimates for the cumulative density functions 
of $X_t$ for all $t$. Using this, 
along with the formula \mbox{$\E[X] = \int_0^\infty \pr(X<t)dt - \int_0^\infty \pr(X<-t)dt$}, 
gives a suitable estimation of the mean of $X$ in an online fashion. 
For the details of this argument, see Section~\ref{sec:gen-product}.

Finally, we discuss how our results can be further generalized 
to the case when the coordinates between rounds are independent 
--- but the coordinates within a round are allowed to have arbitrary correlations.
Once again, we would like to reduce to estimating the mean 
of binary product distributions by trying to estimate tail bounds. 
To see how this might work, we note that in the offline setting 
we can approximate the mean of $Z$ to error $O(\delta)$ 
if we can approximate the mean of $v\cdot Z$ to error $O(\delta)$ 
for every unit vector $v$ (or even for all $v$ in some finite cover of the sphere). 
This suggests the following idea. 
Denoting by $X_t$ the set of coordinates in the $t^{th}$ block, 
if we can estimate the mean of $v\cdot X_t$ for each unit vector $v$ and each $t$, 
this should provide the desired estimates for our mean. 
This idea seems promising as $[v\cdot X_1, v\cdot X_2,\ldots]$ 
is a product distribution. Unfortunately, a naive implementation of this 
will not work, as it might produce error on the order of 
$\sum_t \sup_v \mathrm{EstimationError}(v\cdot X_t)^2$ --- 
while our learner merely guarantees a bound on 
$\sup_v \sum_t \mathrm{EstimationError}(v\cdot X_t)^2$. 
If different $v$'s produce different errors in different rounds, 
this could be much larger than we require. 
To fix this issue, we need a way of combining all of these estimators 
in order to correlate their errors.

To achieve this, we need to modify our binary product estimator. 
To estimate the means of $v\cdot X_t$ for a single $v$, 
we would break our samples into groups,
based on whether or not $v\cdot x_i < s$ for each value of $i$, 
and then compute a mean in each group. 
For the new estimator, we instead break into groups 
based upon whether $v\cdot x_i < s$ for each $i$ \emph{and} each $v$. 
This divides our sample set into many more groups in
each round than the old algorithm did. However, 
if we are interested in estimating the mean of $v\cdot X_t$ 
for some particular vector $v$, we can think of this 
as first splitting into groups based on $v\cdot x_i < s$, 
and then breaking into smaller groups based on the other conditions. 
The thing to note here is that if our estimate of $\E[v\cdot X_t]$ had large error, 
then the first part of the subdivision would lead to 
a correspondingly large increase in our potential function, 
and then the further subdivisions based on other $v$'s 
would make it no smaller (despite not being independent anymore). 
This allows us to bound the errors in the stronger error model that we require. 
The details of this argument can be found in Appendix~\ref{sec:block}.

\subsection{Prior and Related Work}
Here we record related literature that was not discussed
earlier in the introduction. 

\vspace{-0.2cm}

\paragraph{(Offline) Algorithmic Robust Statistics}
The goal of high-dimensional robust statistics is to efficiently obtain 
dimension-independent error guarantees for various statistical tasks
in the presence of a constant fraction of adversarial outliers.
Since the pioneering early work from the statistics community~\cite{Ans60,Tukey60,Huber64,Tukey75}, 
there has been extensive work on designing robust estimators, 
see, e.g.,~\cite{HampelEtalBook86,Huber09} for early textbooks.
Alas, the estimators proposed in the statistics community 
are computationally intractable to compute in high dimensions. 
The first algorithmic progress on high-dimensional robust statistics 
came in two independent works from the 
theoretical computer science community~\cite{DKKLMS16,LaiRV16}.
Since the dissemination of these works, 
which mainly focused on high-dimensional 
robust mean and covariance estimation, 
the body of work in the field has grown rapidly.
Prior work has obtained efficient algorithms with 
dimension-independent guarantees for various robust 
problems, including linear regression~\cite{KlivansKM18,DKS19, BakshiP21}, 
stochastic optimization~\cite{PSBR18,DKK+19-sever}, 
and learning various mixture models~\cite{DKS18-list, KSS18-sos, HL18-sos, BakshiDHKKK20, BK20, DHKK20, LM20-gmm, BD+20-gmm, DKKLT21}. 
For a more detailed account, see the survey \cite{DK19-survey} and the recent 
book~\cite{DK23-book}.
We emphasize that all these prior algorithms work
in the offline setting, where the entire dataset is given in the input and the goal
is to output a single estimate.

\medskip

There are several natural ways to define ``robust online distribution learning'',
based on the underlying scenario to be modeled. Below we summarize
prior work that falls into this general domain along with a comparison 
to our model.

\vspace{-0.2cm}

\paragraph{Distributed Univariate ``Online Robust Mean Estimation''} 
The recent work~\cite{yao2022robust} studies 
the problem of robustly estimating the mean of a single {\em univariate}
distribution when the data is distributed among $n$ clients and arrive in real time. 
At each time step, each agent receives either an \iid~sample from the distribution 
or a corrupted sample with some probability $\eta$.
\cite{yao2022robust} gives a distributed algorithm 
such that the agents' estimations reach consensus 
and converge to the true mean asymptotically.
This contribution is largely orthogonal to our work.
In particular,  we point out two major differences with our setting.
First, in our setting, the samples received at different time steps need not to be independent and identically distributed. In some sense, the setting of \cite{yao2022robust} 
is a special case of our setup 
where the unknown distribution is the product of $T$ identical distributions.
Second, our corruption model is significantly stronger. 
In the setup of \cite{yao2022robust}, each client has a fraction 
of adversarially corrupted samples while in our case 
there are a fraction of adversarial clients 
having only adversarially corrupted samples.
We remark that our setup is closer to the Byzantine error model, 
typically assumed in the context of federated learning.

\vspace{-0.2cm}

\paragraph{Robust Distributed Learning}
A large number of works study distributed SGD in the presence of Byzantine Failures, 
see, e.g.,~\cite{blanchard2017machine}, \cite{su2019securing}, \cite{chen2017distributed}. 
In that setting, a central server collects  stochastic gradients from some worker devices. 
The gradients from most workers are assumed to be computed from \iid~samples 
and a small fraction of Byzantine devices may try to send arbitrary gradient 
updates to corrupt the training process. Typical approaches usually involve 
applying robust estimation techniques to aggregate the gradients received in each iteration. 
A closely related setting is that of robust federated learning; 
see, e.g.,~\cite{pillutla2019robust}, \cite{li2019rsa}, \cite{xie2021crfl}. 
Instead of aggregating the gradient, the central server now tries to directly 
aggregate the model parameters sent from the client device. 
Similarly, a small fraction of Byzantine devices may send arbitrary parameters 
to corrupt the central model parameters. The techniques applied in both settings 
are mostly iteration-independent, which means the accumulated estimation error 
always scales with the number of iterations. This is acceptable in these works, 
as the final goal is just to ensure that the \emph{final model output} in the last round converges. 
We remark that this is different from our setting where the outputs in \emph{all rounds} matter.  

\vspace{-0.2cm}

\paragraph{Robust Online Learning and Bandits}
The works~\cite{lykouris2018stochastic, gupta2019better, bogunovic2021stochastic} 
study robust (linear) stochastic bandits, 
where the data is generated either from some i.i.d.\ distributions 
or adversarially corrupted data. In contrast to the typical contamination 
model assumed in robust statistics, the adversary can corrupt 
the reward of any action at any round, 
and the only restriction is that the difference between the actual reward 
and the corrupted reward needs to be bounded.  

Another type of corruption model, investigated 
in~\cite{altschuler2019best, mukherjee2021mean, kapoor2019corruption, chen2022online}, 
is the contaminated bandit model. Under this model, the rewards 
in most time steps are assumed to follow the underlying reward distributions 
and only a random small fraction of them may be replaced by arbitrary (unbounded) 
corrupted reward prepared by the adversary. This is closer to the corruption model 
considered in the robust statistics literature. 
We remark that our contamination model is still noticeably different. 
In particular, we observe many samples in each round 
and a constant fraction of the samples in each round are corrupted. 
Moreover, the distribution from which the inliers are generated 
can be different from round to round.

\subsection{Discussion and Open Problems}  \label{sec:conc}

This work introduces a natural model of online robust mean estimation
capturing situations where a series of mean estimation tasks need to be completed sequentially, 
using data collected from the same set of sensors of which an $\eps$-fraction are malicious.
We develop two types of algorithms for online robust estimation in this model: 
(i) an efficient algorithm that works for general distributions 
under the stability condition and achieves error which is optimal, up to
a $\log T$ factor, where $T$ is the number of rounds; and 
(ii) an inefficient algorithm that works for more structured distributions (namely product distributions)
and achieves the optimal error --- with no dependence on $T$ whatsoever. 

Our work raises a number of open questions, both technical and conceptual.
First, one may wonder whether there is an algorithm 
that achieves the best of both worlds. Namely, it is statistically and computationally efficient
and achieves error independent of $T$. 
This question is left open, even for identity covariance Gaussians.
In fact, it is not even clear whether there exists an algorithm 
with polynomial {\em sample complexity} and error independent of $T$.

\begin{question}
Are there statistically and/or computationally efficient algorithms 
for online robust mean estimation of identity covariance Gaussians,
within error $\tilde{O}(\eps)$?
\end{question}

Our inefficient algorithms achieving optimal error 
leverage the assumption that the estimation tasks between rounds are independent. 
An interesting direction is to understand the role of ``independence'' 
in online robust mean estimation. 
Concretely, for general Gaussian distributions, it is unclear whether 
the optimal error achievable in the online setting 
is still the same as the offline problem. 

\begin{question}
What is the optimal error of online robust mean estimation of 
an unknown Gaussian distribution $N( \mu^{\ast}, \Sigma^{\ast})$ 
(when the sample size goes to infinity)?
\end{question}

{More generally, it would be interesting to go beyond mean estimation 
and explore the learnability of more general statistical tasks, 
including covariance estimation and linear regression,  
in our robust online learning model. 
The complexity of these tasks has by now been essentially
characterized in the offline model. Understanding the possibilities and limitations 
in our robust online learning setting --- both information-theoretic and computational ---
is a broad challenge for future work.

\section{Preliminaries}
\paragraph{Basic Notation}
We use $\mathbbm Z^+$ to denote the set of positive integers
and $\R^+$ to denote the set of positive reals. 
For $n \in \mathbbm Z^+$, we denote by $[n]$ the set of integers $\{1, \cdots, n\}$.
For $d \in \mathbbm Z^+$, we use $\R^d$ to denote the set of $d$-dimensional real vectors.
For $v \in \R^d$, we write $\snorm{2}{v}$ to denote the $\ell_2$ norm of the vector $v$, i.e., 
$\snorm{2}{v} = \sqrt{ \sum_{i=1}^d v_i^2 } $.
If $M$ is a symmetric matrix, we write $\snorm{2}{M}$ to denote the largest eigenvalue (in absolute value) of $M$.
The asymptotic notation $\tilde{O}$ (resp. $\tilde{\Omega}$) suppresses logarithmic factors in its argument,
i.e., $\tilde{O}(f(n)) = O(f(n)\log^c f(n))$ and
$\tilde{\Omega}(f(n)) = \Omega(f(n)/\log^c f(n))$, where $c>0$ is a universal constant.
Given $x_1, \cdots, x_k \in \R^+$, 
we write $\poly(x_1, \cdots, x_k)$ to denote a sufficiently large constant 
degree polynomial in $\Pi_{i=1}^k x_i$.
For a univariate random variable $X$ and $q \in \R$,
we use $\E[X]$ for its expectation and $\mathbbm 1\{ X > q\}$ 
for the indicator of the event $X > q$. 
Given a set of samples $ \{ x^{(1)}, \cdots, x^{(n)} \} \in \R^{M} $, 
we often write $x^{(1:n)}$ to represent the set.
Let $x \in \R^{ T \cdot d }$ be the concatenation of $T$ sub-vectors $x_1, \cdots x_T \in \R^d$.
We will write $ \bar x_t $ to represent the partition vector that is the concatenation of 
the vectors $x_1, \cdots, x_t$.
Whenever we write $\bar x_t$, the partition of $x$ into the sub-vectors $x_1, \cdots, x_T$ should be clear from the context.

\vspace{-0.2cm}

\paragraph{Stability Condition}
Our efficient algorithm works for any sample set 
satisfying the well-studied \emph{stability} property {(see \cite{DK19-survey})}.

\begin{definition}[($\epsilon, \delta$)-stability] \label{def: stability}
For $\epsilon \in (0, 1/2)$ and $\delta \geq \epsilon$, 
a finite set $S \subset \R^d$
is $(\epsilon, \delta)$-stable with respect to a vector $\mu \in \R^d$ 
if for every unit vector $v \in \mathbb{R}^d$ 
and every subset $S' \subseteq S$, where $|S'| \geq (1-\epsilon)|S|$, 
the following conditions are satisfied:
\begin{enumerate}
    \item $\left| \frac{1}{|S'|}{\sum_{x \in S'}  v^T \cdot \lp(x-\mu \rp) } \right|\leq \delta.$
    \item $\left| \frac{1}{|S'|}  \sum_{x \in S'} \big(v^T \cdot (x-\mu)\big)^2-1 \right| \leq \delta^2/\epsilon.$
\end{enumerate}
\end{definition}

A stable set $S$ satisfies that any sufficiently large subsets of $S$ 
can produce accurate enough first and second moment estimations, 
captured by the parameters $\epsilon$ and $\delta$. 

{This stability condition (or variants thereof) has been proven 
critical for robust mean estimation algorithms even in the offline setting. 
In particular, essentially all known efficient algorithms 
for learning the mean of a distribution from $\eps$-corrupted samples 
to error $\delta$ require some condition on the uncorrupted samples 
at least as strong as $(\eps,\delta)$-stability. As such, it will 
be important for us to know that our sample set satisfies this condition. 
This problem has been extensively studied in the literature 
(see, e.g., \cite{DK19-survey}). 
For example, we have the following results:}
\begin{enumerate}
    \item If $S$ is a set of i.i.d.\ samples from a distribution of identity-bounded covariance 
    $\Sigma \preceq \bm{I}$ and $|S| = \tilde{\Omega}(d/\epsilon)$, 
    then with high probability $S$ is $(\epsilon, O(\sqrt{\epsilon}))$-stable.
    \item If $S$ is a set of i.i.d.\ samples from a subgaussian distribution 
    with identity covariance $\Sigma = \bm{I}$ and 
    $|S| = \tilde{\Omega}(d/\epsilon^2)$, then with high probability 
    $S$ is $(\epsilon, O(\epsilon\sqrt{\log (1/\epsilon) }))$-stable.
\end{enumerate}
We here remark a simple property of stability that is particularly useful for the online setup.
Let $\{ x^{(1)}, \cdots, x^{(n)} \} \subset \R^{T \cdot d}$ be a set of samples satisfying the $(\epsilon, \delta)$-stability condition with respect to some vector $\mu \in \R^{T \cdot d}$.
Consider the partition of coordinates into $T$ parts such that 
$x^{(i)}$ is the concatenation of $x^{(i)}_1, \cdots, x^{(i)}_T \in \R^d$ for $i \in [n]$
and $\mu$ is the concatenation of $\mu_1, \cdots, \mu_T \in \R^d$.
Then, for all $t \in [T]$, the set $\{  \bar x^{(1)}_t, \cdots, \bar x^{(n)}_t \} \subset \R^{d \cdot t}$
is also $(\epsilon, \delta)$-stable with respect to vector $\bar \mu_t \in \R^{d \cdot t}$.

In the rest of the paper, we assume that the initial uncorrupted sample set $\mathcal{C}$
is $(\epsilon, \delta)$-stable. 
We use $\mathcal X$ to denote the $\eps$-corrupted version of $\mathcal C$ under the strong contamination model.
and we use $\mathcal{H}$ to denote the set of clean samples in $\mathcal{X}$, i.e. $\mathcal H = \mathcal C \cap \mathcal X$. 
Consequently, $\mathcal{X} \setminus \mathcal{H}$ represents the corrupted samples.

\section{Efficient Online Robust Mean Estimation} \label{sec:filter-online}
In this section, we describe our computationally efficient algorithm for online
robust mean estimation, thereby establishing Theorem~\ref{thm:main1}.
Before stating our approach, we describe a natural attempt and discuss why it fails. 

We start by observing that the naive approach of applying the 
offline weighted filter to the data $\{ x^{(1)}_t, \cdots, x^{(n)}_t \}$ 
revealed in the $t$-th round \emph{independently} does not suffice.
Such a naive algorithm will incur error $\eps'$ (achievable by the optimal offline filter algorithm) in each round, 
leading to a final $\ell_2$ error of $\sqrt{T \cdot \lp(\eps'\rp)^2} = \eps' \sqrt{T}$.
As hinted at the end of Section~\ref{ssec:background}, 
a key idea in online robust estimation is the interplay between
filtering outliers and establishing ``trust'' over the data providers. 
Hence, a natural idea is to let the filtering algorithm in the $(t+1)$-th round 
to ``inherit'' the information about how likely each sample $x^{(i)}$ 
is an outlier from the filtering result in the $t$-th round. 

Suppose we are using the weighted filtering algorithm 
which produces a set of weights $w^{(i)}_t$ for each sample $i \in [n]$ 
at the end of the $t$-th round. We can then initialize the weights 
for the filtering algorithm in the $(t+1)$-th round as exactly $w^{(i)}_t$.
Unfortunately, the idea to simply maintain an online version 
of the weights achieves little improvement in the worst case.
Consider the case where the unknown distribution $X$ is the product of $T$ 
isotropic Gaussians $X_1, \cdots , X_T$. Then the adversary can 
contaminate the set of samples $\mathcal{C}$ to make them 
look like \iid~samples from another isotropic Gaussian distribution $X'_t$ 
for each round $t \in [T]$, such that $\snorm{2}{\E[X'_t] - \E[X_t]} = c \cdot \eps$ 
for some constant $c$~\cite{DK19-survey}. 
Then the filtering algorithm should not downweight any sample, 
since the revealed coordinates in each round $t$ of contaminated 
samples are statistically indistinguishable from i.i.d.\ samples from $X'_t$. 
As a result, the algorithm's error still grows with $\sqrt{T}$, 
but the weight remains unchanged throughout the process.

We address this issue by considering the aggregation of all historical records. 
At the $t$-th round, we will concatenate the vectors $x_1^{(i)}, \cdots, x_t^{(i)}$ 
together into $\bar x_t^{(i)} \in \R^{t \cdot d}$ 
and perform filtering on the dataset $ \bar x_t^{(1:n)}$.
In particular, after initializing the weight $w^{(1:n)}_0$ as $1/n$, 
our algorithm repeats the following 
two main procedures: 
(a) concatenate the coordinates of each sample revealed so far, 
denoted by $\bar{x}^{(i)}_t=(x^{(i)}_1, \cdots , x^{(i)}_t)$ 
for $i \in [n]$; 
(b) apply filters to iteratively decrease the weights $w^{(1:n)}_{t-1}$ 
inherited from the last round until the set 
$\bar x^{(1:n)}_{t}$ under the new weights $w^{(1:n)}_{t}$ 
satisfies the appropriate second moment condition.
Our proposed efficient online algorithm is presented in pseudocode 
as Algorithm \ref{alg:online_eff_filter} below. 
Intuitively, via operation (a), Algorithm~\ref{alg:online_eff_filter} 
ensures that the estimation made in the $t$-th round properly utilizes 
all historical information; 
and through operation (b) that
the weights adjust to reflect the ``likelihood'' 
of a sample being an outlier as the algorithm collects more information.

Recall that in the offline setting 
(loading the entire data set $\mathcal X$ and computing the estimation once), 
the information-theoretically optimal error guarantee 
under the $(\eps, \delta)$-stability condition is $\Theta(\delta)$.
Somewhat surprisingly, the above technique in the online setting 
yields an error that has only an extra $\log T$ factor.
\begin{theorem}
Suppose that $\mathcal{C}$ is $(\epsilon, \delta)$-stable with respect to $\mu^*$ 
and $X$ is an $\epsilon$-corrupted version of $\mathcal{C}$. 
Then, for $\epsilon$ at most a sufficiently small positive constant, 
there exists some constant $\kappa$ 
such that Algorithm \ref{alg:online_eff_filter}, when $\lambda \new{=}\kappa \delta^2/\epsilon$, outputs a sequence of estimates satisfying 
$$ \snorm{2}{ \mu - \mu^* }  = O(\delta\log T) \;.$$
\label{thm:filter_main}
\end{theorem}
\vspace{-2em}
In the following, we present the proof of Theorem~\ref{thm:filter_main}.
With the stability assumption of clean data in mind, 
we list a few properties of the offline weighted filter algorithm that will be used in the analysis. 
Given a proper selection of filtering threshold $\lambda$ dependent 
upon the stability parameters, in every round the filter always 
removes more weighted mass from the adversarial samples 
compared to that from honest/clean samples. 
As the set of vectors in $\mathcal{C}$ truncated to the coordinates revealed so far
remains stable and Algorithm \ref{alg:online_eff_filter} iteratively applies the filter, 
Algorithm \ref{alg:online_eff_filter} inherits this property. 
Therefore, given a limited budget $\epsilon$ of the adversary, 
Algorithm \ref{alg:online_eff_filter} will finally terminate 
to find a proper weight set  
satisfying the desired second moment bound. 

We formally state this as the following lemma.
\begin{algorithm}
\caption{Online Filter}
\begin{algorithmic}[1]
\STATE \textbf{Input:} The number of samples $n$, Byzantine fraction $\epsilon$, round number $T$, sample coordinates $x^{(1:n)}_{t}$ revealed at the $t^{th}$ round for $t=1,2,\cdots,T$, filter threshold $\lambda$, and initialized weight $w^{(i)}_0=1/n$, for $i=1,2,\cdots,n$.
\FOR{$t=1,2,\cdots,T$}
   \STATE Initialize $\bm{w}_{t} \gets \bm{w}_{t-1}$ and update $\bar{\bm{x}}^{(i)}_t=(x^{(i)}_1, ... , x^{(i)}_t)$.
   \STATE Compute $\Sigma \gets \text{WCov}(\bm{w}_{t}, \bar{\bm{x}}^{(1:n)}_t)$.
    \WHILE{$\|\Sigma \|_2 > 1 + \lambda$}
      \STATE Compute the top eigenvector $v$ of $\Sigma$.
      \STATE Compute empirical weighted mean $\mu(\bm{w}_{t}) \gets \sum_{i=1}^n w^{(i)}_t \bar{\bm{x}}^{(i)}_t.$
      \FOR{$i=1,2,\cdots ,n$}
      \STATE Compute $\rho^{(i)} \gets \langle v, \bm{x}^{(i)}- \mu(\bm{w}_{t}) \rangle^2$.
      \ENDFOR
      \STATE Sort $\rho^{(1:n)}$ into a decreasing order denoted as $\rho^{\pi(1)} \geq \rho^{\pi(2)} \geq \cdots \geq \rho^{\pi(n)}$, and let $\beta$ be the smallest number such that $\sum_{i=1}^{\beta} \rho^{\pi(i)} > 2\epsilon.$
      \STATE Apply WFilter to update weights for $\{\pi(1), \cdots,\pi(\beta)\}$ as $\{w^{\pi(1)}_t,\cdots,w^{\pi(\beta)}_t\} \gets \text{WFilter}(\rho^{\pi(1:\beta)}, w^{\pi(1:\beta)}_t)$; while the remaining weights keep the same, i.e., $w^{\pi(i)}_t = w^{\pi(i)}_t$ for $i >\beta$.
      \STATE Update $\Sigma \gets \text{WCov}(\bm{w}_{t}, \bar{\bm{x}}^{(1:n)}_t)$
    \ENDWHILE
    \STATE \textbf{Output}: $\mu_{t} = \sum_{i=1}^n w^{(i)}_t x^{(i)}_t$.
\ENDFOR
\end{algorithmic}
\hrulefill\\
\textbf{Subroutine 1}: Weighted Filter (WFilter)
\begin{algorithmic}[1]
    \STATE \textbf{Input:} scores $\rho^{\pi(1:\beta)}$,  weights $w^{\pi(1:\beta)}$.
    \FOR{$i=1,2,...,\beta$}
    \STATE $w^{\pi(i)} \gets (1-\frac{\rho^{\pi(i)}}{\max_j \rho^{\pi(j)}})w^{\pi(i)}.$
    \ENDFOR
    \STATE \textbf{Output:} $w^{\pi(1:\beta)}$.
\end{algorithmic}
\hrulefill\\
\textbf{Subroutine2}: Weighted Covariance (WCov)
\begin{algorithmic}[1]
    \STATE \textbf{Input:} weight  $\bm{w}=w^{(1:n)}$ and  samples $\bar{x}^{(1:n)}$.
    \STATE Compute $\mu(\bm{w}) \gets \sum_{i=1}^n \frac{w^{(i)}}{\|\bm{w}\|_1} \bar{x}^{(i)}.$
    \STATE Compute weighted covariance estimation $\Sigma \gets \sum_{i=1}^n \frac{w^{(i)}}{\|\bm{w}\|_1}(\bar{x}^{(i)}-\mu(\bm{w}))(\bar{x}^{(i)}-\mu(\bm{w}))^T.$
    \STATE \textbf{Output:} $\Sigma$.
\end{algorithmic}
\label{alg:online_eff_filter}
\end{algorithm}

\begin{lemma}[Proposition 2.13 of \cite{DK19-survey}] When $\mathcal{C}$ is $(\epsilon, \delta)$-stable, $\mathcal{X}$ is an $\epsilon$ corrupted version of $\mathcal{C}$, and $\epsilon < \epsilon_0$ for some sufficiently small $\epsilon_0$, there exists some constant $\kappa$ such that in Algorithm~\ref{alg:online_eff_filter} when $\lambda \geq \kappa \delta^2/\epsilon$, 
for any $t \in [T]$, $\sum_{i \in \mathcal{H}} w^{(i)}_{t-1} - \sum_{i \in \mathcal{H}} w^{(i)}_{t} \leq \sum_{i \in \mathcal{X} \setminus \mathcal{H}} w^{(i)}_{t-1} - \sum_{i \in \mathcal{X} \setminus \mathcal{H}} w^{(i)}_{t}.$
\label{lem: less_mass}
\end{lemma}

{It is also worth noting that at the end of the filter step in any round, the empirical covariance matrix $\Sigma=\text{WCov}(\bm{w}_{t}, \bar{\bm{x}}^{(1:n)}_t)$ satisfies $\|\Sigma\|_2 \leq 1+\lambda$. In fact, by stability of $\mathcal{C}$, the weighted covariance of just the points in $\mathcal{H}$ must be at least $1-O(\delta^2/\eps)$ in every direction, and so if $\lambda$ is an appropriate multiple of $\delta^2/\eps$, we will have that $\|\Sigma\|_2\leq 1 + \lambda$.
This second moment bound then allows us to control the error in our estimate of the mean. In particular, we have:}
\begin{lemma}[Lemma 2.4 of \cite{DK19-survey}] If $\mathcal{X} = \{x^{(1)}, ... , x^{(n)}\}$ is an $\epsilon$-corrupted version of an $(\epsilon, \delta)$-stable set $\mathcal{C}$ with respect to $\mu$ and $\epsilon < \epsilon_0$ for some sufficiently small $\epsilon_0$, then for any selection of weights $w^{(1:n)}$ such that $\sum_{i=1}^n |1/n-w^{(i)}| \leq 2\epsilon$, 
$$ \|\mu(\bm{w}) - \mu\|_2 \leq O(\delta+ \sqrt{\epsilon \cdot {\max\{(\|\Sigma(\bm{w})\|_2-1), 0\}}}),$$
where the empirical mean is denoted by $\mu(\bm{w})= \sum_{i=1}^n \frac{w^{(i)}}{\|\bm{w}\|_1}x^{(i)}$, and the empirical covariance $\Sigma(\bm{w}) = \sum_{i=1}^n \frac{w^{(i)}}{\| \bm{w}\|_1}(x^{(i)}-\mu(\bm{w}))(x^{(i)}-\mu(\bm{w}))^T$.
\label{lem: signature}
\end{lemma}

Lemma \ref{lem: signature} states that the error in the empirical estimate of the mean is controlled by the empirical covariance. In particular, after filtering we can guarantee that $\|\Sigma(\bm{w})\|_2-1\leq \lambda$, and consequently guarantee a mean estimation error of $O(\delta + \sqrt{\eps\lambda})=O(\delta)$. 
For an offline robust mean estimation algorithm, the above properties would be enough to show the error guarantees of the algorithm. 
To analyze the behavior of the algorithm in the online setting, we need a more subtle property of the filtering algorithm: 
the difference in the estimations of $\mu_t^*$ using weights from two different rounds is proportional to the difference in the weights.
The formal statement is given below.
\begin{lemma}
\label{lemma:bounded-second-difference}
{Given the assumptions of} Theorem \ref{thm:filter_main}, for $1\leq t < t' \leq T$, let $\mu(\bm{w}_t, \bar{\bm{x}}^{(1:n)}_t) = \sum_{i=1}^n \frac{w^{(i)}_t\bar{\bm{x}}^{(i)}_t}{\|\bm{w}_t\|_1}$ and $\mu(\bm{w}_{t'}, \bar{\bm{x}}^{(1:n)}_{t}) = \sum_{i=1}^n \frac{w^{(i)}_{t'}\bar{\bm{x}}^{(i)}_t}{\|\bm{w}_{t'}\|_1}$. Then, when we apply Algorithm \ref{alg:online_eff_filter}, it holds that
\begin{align*}
    \snorm{2}{ \mu(\bm{w}_t, \bar{\bm{x}}^{(1:n)}_t) - \mu(\bm{w}_{t'}, \bar{\bm{x}}^{(1:n)}_{t})}^2
    \leq O(1) \cdot \frac{\delta^2}{\epsilon}  \cdot
    \snorm{1}{\bm{w}_t- \bm{w}_{t'}}.
\end{align*}
\end{lemma}

\begin{proof}
Let $\bm{y}_t = \bm{w}_t/\|\bm{w}_t\|_1$ and $\bm{y}_{t'} = \bm{w}_{t'}/\|\bm{w}_{t'}\|_1$ be the normalized weight outputted at the round $t$ and $t'$, respectively, and $\eta =\|\bm{y}_t - \bm{y}_{t'}\|_1$.
We consider the following decomposition $\bm{y}_t = (1-\eta)\bm{y}_{t'}+\eta\bm{e},$ where $\bm{e}$ is a non-zero weight vector {with} $\|\bm{e}\|_1=1$.
{
Notice that $\bm{y}_t$, $\bm{y}_{t'}$, and $\bm{e}$ can be thought as distributions over our sample vectors $\bm{\bar{x}}^{(i)}_t$.
Essentially, the distribution under $\bm{y}_t$ is a mixture of the distributions under $\bm{y}_{t'}$, and $\bm{e}$ respectively. 
This implies that
$$
\cov(\bm{y}_t) = (1-\eta) \cdot \cov(\bm{y}_{t'})+\eta \cdot \cov(\bm{e})+
\eta(1-\eta) \cdot \Big(\mu(\bm{y}_{t'}) - \mu(e) \Big) \Big(\mu(\bm{y}_{t'}) - \mu(\bm{e}) \Big)^T,
$$
where $\cov(\cdot)$ and $\mu(\cdot)$ denote the covariance and mean over $\bar x^{(i)}_t$ respectively under the argument inside. 
Since both $\|\cov(\bm{y}_t)-I\|_2$ and $\|\cov(\bm{y}_{t'})-I\|_2$ are $O(\delta^2/\eps)$ and since $\cov(\bm{e})\succeq 0$, we have that $\|\mu(\bm{y}_{t'}) - \mu(\bm{e})\|_2^2 = O(\delta^2 / (\eta\eps)).$ Combining this with the fact that
$$
\mu(\bm{y}_t)-\mu(\bm{y}_{t'}) 
= (1-\eta) \cdot \mu(\bm{y}_{t'})
+\eta \cdot \mu(\bm{e})-\mu(\bm{y}_{t'}) =
 -\eta \cdot \Big(\mu(\bm{y}_{t'}) - \mu(\bm{e}) \Big)
$$
then yields
$
\snorm{2}{\mu(\bm{y}_t)-\mu(\bm{y}_{t'})}^2 \leq O\lp(\eta \cdot \delta^2 / \eps \rp).
$
Finally, notice that, by definition, $\mu(y_t)$ is exactly $\mu(w_t, \bar x^{(1:n)}_t)$(and similarly for $\mu(y_t')$ and $\mu(w_t', \bar x^{(1:n)}_t)$), and $\eta = O( \snorm{1}{ w_t - w_t' } )$. Our lemma follows.
}
\end{proof}

The error guarantee of Algorithm~\ref{alg:online_eff_filter} then largely 
follows from the following two observations: 
(i) after the last round, if we were to use the weights $\bm{w}_T$ in 
hindsight to estimate the means in each round, the error will be optimal since 
the algorithm is essentially the same as the one used in the offline setting; 
(ii) $\mu( \bm{w}_t, \bm{\bar{x}}^{(1:n)}_t  )$, the estimation outputted at 
the $t^{th}$ round, will not differ much from 
$\mu(\bm{w}_{t'}, \bm{\bar{x}}^{(1:n)}_t)$, 
the estimation if we were to use the weights $w_{t'}$ 
from some future round $t' >t$. We note that (i) simply follows from the 
standard guarantees of the filtering algorithm while (ii) follows from 
Lemma~\ref{lemma:bounded-second-difference} and Lemma~\ref{lem: less_mass}. 
Formally, we show the cumulative $L^2$ difference 
between the mean estimations produced across $T$ time slots 
and that from the last round can be bounded 
by $O(\log^2 T)$ through a careful recursive argument.
\begin{lemma}
\label{lem:T-difference}
$\sum_{t=1}^T \|\mu(\bm{w}_t, {{x}}^{(1:n)}_t) - \mu(\bm{w}_T, {{x}}^{(1:n)}_t) \|_2^2 \leq O \lp( \log^2 T \cdot \delta^2 \rp)$.
\end{lemma}
\begin{proof}
{We begin by reviewing a few key facts about our algorithm.

Firstly, we note that, since the algorithm only decreases weights it will be the case that $w^{(i)}_1 \geq w^{(i)}_2 \geq \ldots \geq w^{(i)}_T$, for any $i$, and in particular $\|\bm{w}_t - \bm{w}_{t'}\|_1 = \|\bm{w}_t\|_1 - \|\bm{w}_{t'}\|_1$ for any $1\leq t \leq t' \leq T$. It also follows that $\|\bm{w}_1\|_1 \geq \|\bm{w}_2\|_1 \geq \ldots \|\bm{w}_T\|_1$.

Secondly, we note that by Lemma \ref{lem: less_mass} our algorithm removes more mass from bad elements than good. Since the initial mass of the bad elements is only at most $\eps$, this implies $\|\bm{w}_1\|_1 - \|\bm{w}_T\|_1 = O(\eps).$

Finally, we note by Lemma \ref{lemma:bounded-second-difference} that for all $t'>t$ and sufficiently large $C_0$ 
\begin{equation}
\label{vector difference bound equation}
\sum_{l=1}^{t} \snorm{2}{ \mu(\bm{w}_t, {{x}}^{(1:n)}_l) - \mu(\bm{w}_{t'}, {{x}}^{(1:n)}_{l})}^2 = \snorm{2}{ \mu(\bm{w}_t, \bar{\bm{x}}^{(1:n)}_t) - \mu(\bm{w}_{t'}, \bar{\bm{x}}^{(1:n)}_{t})}^2
    \leq C_0 \cdot \frac{\delta^2}{\epsilon}  \cdot
    (\|\bm{w}_t\|_1- \|\bm{w}_{t'}\|_1).
\end{equation}

Our goal will now be to prove that for any two rounds $a < b$ and any sufficiently large $\sqrt{C}>3\sqrt{C_0}$ that:
\begin{equation}
\label{recursive bound equation}
\sum_{t=a}^b \|\mu(\bm{w}_t, {{x}}^{(1:n)}_t) - \mu(\bm{w}_b, {{x}}^{(1:n)}_t) \|_2^2 \leq C \lp( \log^2(1+b-a) \cdot (\| \bm{w}_a\|_1 - \|\bm{w}_b \|_1) \cdot  \frac{\delta^2}{\eps} \rp).
\end{equation}
In particular, we will prove this by strong induction on $b-a$.

The base case here is when $b=a+1$, in which case, Equation \eqref{recursive bound equation} follows immediately from Equation \eqref{vector difference bound equation}.
For the inductive step, consider $c=\lfloor (a+b)/2\rfloor$. Then, we note that
$$
\sum_{t=a}^b \|\mu(\bm{w}_t, {{x}}^{(1:n)}_t) - \mu(\bm{w}_b, {{x}}^{(1:n)}_t) \|_2^2 
\hspace{-0.1em}
= 
\hspace{-0.1em}
\sum_{t=a}^c \|\mu(\bm{w}_t, {{x}}^{(1:n)}_t) - \mu(\bm{w}_b, {{x}}^{(1:n)}_t) \|_2^2 + 
\hspace{-0.5em}
\sum_{t=c+1}^b \|\mu(\bm{w}_t, {{x}}^{(1:n)}_t) - \mu(\bm{w}_b, {{x}}^{(1:n)}_t) \|_2^2.
$$
By the inductive hypothesis, we bound the second term by
$$
C \lp( \log^2(1+b-c) \cdot (\| \bm{w}_c\|_1 - \|\bm{w}_b \|_1) \cdot \frac{\delta^2}{\eps} \rp) \leq C \lp( (\log(1+b-a)-1/2)^2 \cdot (\| \bm{w}_c\|_1 - \|\bm{w}_b \|_1) \cdot \frac{\delta^2}{\eps} \rp).
$$
By triangle's inequality, the first term is at most
\begin{align}
\label{term_A}
\sum_{t=a}^c \left(\|\mu(\bm{w}_t, {{x}}^{(1:n)}_t) - \mu(\bm{w}_c, {{x}}^{(1:n)}_t) \|_2 + \|\mu(\bm{w}_c, {{x}}^{(1:n)}_t) - \mu(\bm{w}_b, {{x}}^{(1:n)}_t) \|_2\right)^2.
\end{align}
For convenience, we will denote
\begin{align*}
\alpha &= \sum_{t=a}^c\|\mu(\bm{w}_t, {{x}}^{(1:n)}_t) - \mu(\bm{w}_c, {{x}}^{(1:n)}_t) \|_2^2 \\
\beta &= \sum_{t=a}^c\|\mu(\bm{w}_c, {{x}}^{(1:n)}_t) - \mu(\bm{w}_b, {{x}}^{(1:n)}_t) \|_2^2. 
\end{align*}
Then, it is easy to see that Equation~\eqref{term_A} is at most $\alpha+\beta+2\sqrt{\alpha\beta}$.

Our inductive hypothesis tells us that
$$
\alpha \leq C \lp( \log^2(1+c-a) \cdot (\| \bm{w}_a\|_1 - \|\bm{w}_c \|_1) \cdot \frac{\delta^2}{\eps} \rp) \leq C \lp( (\log(1+b-a)-1/2)^2 \cdot (\| \bm{w}_a\|_1 - \|\bm{w}_c \|_1) \cdot \frac{\delta^2}{\eps} \rp),
$$
and Equation \eqref{vector difference bound equation} tells us that
$$
\beta \leq C_0 \cdot \frac{\delta^2}{\epsilon} \cdot (\|\bm{w}_a\|_1- \|\bm{w}_{c}\|_1).
$$
Thus, combining the above, we find that $\sum_{t=a}^b \|\mu(\bm{w}_t, {{x}}^{(1:n)}_t) - \mu(\bm{w}_b, {{x}}^{(1:n)}_t) \|_2^2$, as desired, is at most
\begin{align*}
& C \lp( (\log(1+b-a)-1/2)^2 \cdot (\| \bm{w}_c\|_1 - \|\bm{w}_b \|_1) \cdot \frac{\delta^2}{\eps} \rp) \\
&  + C \lp( (\log(1+b-a)-1/2)^2 \cdot (\| \bm{w}_a\|_1 - \|\bm{w}_c \|_1) \cdot \frac{\delta^2}{\eps} \rp) + \beta + 2\sqrt{\alpha\beta}\\
& \leq (\delta^2/\eps) \cdot \lp(\| \bm{w}_a\|_1 - \|\bm{w}_b \|_1\rp)\left( C(\log^2(1+b-a)-\log(1+b-a))+C_0+2\sqrt{CC_0 \log^2(1+b-a) } \right)\\
& \leq (\delta^2/\eps) \cdot \lp(\| \bm{w}_a\|_1 - \|\bm{w}_b \|_1\rp)\left(C\log^2(1+b-a) - C\log(1+b-a) +3\log(1+b-a)\sqrt{CC_0} \right)\\
& \leq C \lp( \log^2(1+b-a) \cdot \lp(\| \bm{w}_a\|_1 - \|\bm{w}_b \|_1\rp) \cdot \delta^2/\eps \rp).
\end{align*}
 Note that the last line above depends on the selection that $\sqrt{C} \geq 3\sqrt{C_0}$.
This completes our proof of Equation \eqref{recursive bound equation}, and plugging in $a=1$ and $b=T$ completes the proof of this lemma.
}
\end{proof}

Finally, with all the above preparation, we can prove the statement in Theorem \ref{thm:filter_main}. First, from Lemma \ref{lem: less_mass}, the weights do not increase and once the second moment criterion is not satisfied, more mass will be removed from the adversary side. Therefore, given the bounded $\epsilon$ budget for the adversary, Algorithm \ref{alg: eff_filter} will finally terminate to find $\bm{w}_T$ such that the empirical covariance $\|\Sigma(\bm{w}_T)-\bm{I}\|\leq \lambda$, for some sufficiently large $\lambda \geq \kappa \delta^2/\epsilon$. 
Moreover, it must terminate in at most $\eps \cdot n$ filter iterations since the weights of at least $1$ point becomes $0$ after each filter iteration.
Then, by Lemma \ref{lem: signature}, it holds
\begin{align}
    \sum_{t=1}^T \|\mu(\bm{w}_T,
 {x}^{(1:n)}_t) - \mu^*_t\|^2 = \|\mu(\bm{w}_T, \bm{\bar{x}}^{(1:n)}_T) -\bm{\mu}^* \|^2 \leq O(\delta^2).
 \label{final-T-error}
\end{align}
{Combining} Lemma \ref{lem:T-difference} with Equation~\eqref{final-T-error}, we then have
\begin{align*}
    \sum_{t=1}^T \|\mu(\bm{w}_t,
 {x}^{(1:n)}_t) - \mu^*_t\|^2 & =  \sum_{t=1}^T \|\mu(\bm{w}_t,
 {x}^{(1:n)}_t) - \mu(\bm{w}_T,  {x}^{(1:n)}_t)+ \big(\mu(\bm{w}_T,  {x}^{(1:n)}_t) - \mu^*_t \big)\|^2\\
 & \leq 2\big(\sum_{t=1}^T \|\mu(\bm{w}_t,
 {x}^{(1:n)}_t) - \mu(\bm{w}_T,  {x}^{(1:n)}_t)\|^2+ \|\mu(\bm{w}_T,  {x}^{(1:n)}_t) - \mu^*_t\|^2\big)\\
 & = O(\delta^2+{\delta^2}\log^2T).
\end{align*}
Thus, Theorem \ref{thm:filter_main} follows.

\begin{remark}
{\em Throughout the analysis, we relied on only two properties of the filter algorithm: (i) the algorithm at each step filters more corrupted sample points than  uncorrupted ones, and (ii) the filter in each round terminates when the rest of samples truncated to the coordinates revealed so far have bounded covariance. 
We note that the extra $\log T$ factor in our analysis of the filter algorithm in Theorem~\ref{thm:filter_main} is nearly tight if one does not leverage any additional property of the filtering algorithm. 
In particular, consider a special case of the theorem where the clean set $C$ is a set that has its covariance bounded by some constant multiple of $I$.
Then, let $X^{(1)}, \cdots X^{(T)}$ be sets of samples such that $X^{(t)}$ represents the samples kept by the filter algorithm until the $t$-th round (assume the filter algorithm always sets the weight of a sample to be either $0$ or $1$).
We show that if $X^{(1)}, \cdots X^{(T)}$ are sets satisfying only properties (i) and (ii), the output of the algorithm can incur an $\ell_2$ error as large as $\Omega( \eps \log T )$.
For more details of this, see Appendix~\ref{sec:lower-bound}.}
\end{remark}

\section{Optimal Error for Product Distributions} \label{sec:gen-product}
In this section, we establish Theorem~\ref{thm:products-intro}.
We start 
with the special case of binary product distributions 
(Section~\ref{sec:binary-prod})  
and develop on optimal error (inefficient) algorithm in this setting. 
We then show that our upper bound for binary products
can be used as a subroutine to perform optimal error 
online robust mean estimation 
for more general families of product distributions, including
identity covariance Gaussians (Section~\ref{ssec:gaussian}) 
and product distributions whose coordinates can even 
come from nonparametric families, 
as long as they satisfy mild concentration properties (Section~\ref{ssec:nonpar}).

\subsection{Binary Product Distributions} \label{sec:binary-prod}
In this section, we present an algorithm which robustly estimates 
the mean of binary product distributions in the online setting 
and achieves the optimal accuracy.
As we will see in proceeding sections, the algorithm can also be used 
as a building block to obtain online mean estimations 
for many other important families of distributions.

For this purpose, it is useful to consider binary product distributions 
whose coordinate-wise means are uniformly bounded 
by some constant $\gamma \in (0,1)$. We will call such distributions 
``$\gamma$-bounded'' binary product distributions.

\begin{definition}[$\gamma$-Bounded Binary Product Distribution] \label{def:gamma-bounded}
Let $X$ be a distribution on the boolean hypercube $\{0,1\}^{M}$. 
We say $X$ is a binary product distribution if its coordinates are mutually independent. 
Additionally, it is $\gamma$-bounded if each coordinate $X_i$ 
satisfies $\E[X_i] \leq \gamma$ for $i \in [M]$.
\end{definition}

We briefly discuss robust mean estimation of such distributions in the offline model.
When $\eps$-fraction of the samples are generated adversarially, 
it is possible to approximate the mean within $\ell_2$ distance $O(\eps)$ 
for any binary product distribution (not necessarily $\gamma$-bounded).
Importantly, this is information-theoretically optimal, 
in the sense that no algorithm can distinguish between 
two binary product distributions whose means differ by $\Omega(\eps)$ 
given a dataset \emph{of any size} such that $\eps$-fraction of the data points are corrupted.

However, when we have the extra condition that the mean 
of the unknown distribution is coordinate-wise bounded by $\gamma$, 
for some $\gamma < \eps$, it turns out we can take advantage 
of the condition to improve our estimation accuracy. 
In particular, it can be shown that any two $\gamma$-bounded 
binary product distributions within total variation distance $\eps$ 
have their mean differ by at most $\sqrt{\eps \gamma}$ in $\ell_2$-distance.
Hence, it is possible to estimate the means of $\gamma$-bounded 
binary product distributions up to accuracy $O( \min( \eps, \sqrt{\eps \gamma}))$ in the offline model. 
Without too much extra effort, it is easy to see the accuracy is also information-theoretically optimal.
As the main theorem of the section, we show this optimal accuracy 
is still achievable in the online setting.

In the following section, we restrict our attention to the setting 
when only one coordinate is revealed in each round, i.e.,
$d = 1$ and $M = T$.
This is a strictly harder setting, 
as we can always simulate the process of revealing the coordinates one 
at a time even when $d > 1$. 
Hence, in the rest of this section, 
we will always have $d = 1$ and the unknown distribution $X$ 
is always a $T$-dimensional distribution.

\begin{theorem} \label{lem:binary-product-estimation}
Let $\eps, \gamma, \tau \in (0,1)$.
Suppose $X$ is a $T$-dimensional $\gamma$-bounded binary product distribution. 
For sufficiently small $\epsilon$,
there exists an algorithm \textbf{Binary-Product-Estimation} which robustly estimates 
the mean of $X$ under $\eps$ corruption in the online setting 
with error $O( \min(\eps, \sqrt{\gamma \eps}) )$, 
failure probability $\tau$, and sample complexity 
$$
n \geq 2^T \cdot \poly(T, 1/\eps, 1/\gamma) \cdot \log(1/\tau) \;.
$$
\end{theorem}

\begin{algorithm}[ht] 
\caption{Binary-Product-Estimation}
\label{alg:binary-product-estimation}
\begin{algorithmic}[1]
\STATE \textbf{Input:} The number of samples $n$, Byzantine fraction $\epsilon$, round number $T$, sample coordinates $\{x^{(1)}_{t}\, \cdots x^{(n)}_t\}$ revealed at $t$-th iteration for $t=1,2,\cdots,T$, bound parameter $\gamma$.
\STATE Initialize the group $S_0^{(1)} = [n]$.
\FOR{$t=1,2,...,T$}
   \STATE In the $t$-th round, $x_t^{(1)}, \cdots, x_t^{(n)}$ are revealed.
   \FOR{$i=0,\cdots,2^{t-1}-1$}
   \STATE Compute the group estimation
   $\mu_{t}^{(i)} \eqdef \min \lp(   \gamma, \frac{1}{ |S_i^{(t)}| }\sum_{j \in S_i^{(t)}  } x^{(j)}_{t}  \rp)$.
   \STATE Split to create the child groups.
   $$S_{2 \cdot i}^{(t+1)} = \{ j \in S_i^{(t)} \text{ such that } x_t^{(j)} = 0   \} \, ,
   S_{2\cdot i + 1}^{(t+1)} = \{ j \in S_i^{(t)} \text{ such that } x_t^{(j)} = 1  \}.
   $$
   \ENDFOR
   \STATE Set $\mu_t$ to be the weighted median over $\mu_t^{(i)}$ where the weights are given by $\abs{ S_i^{(t)} }$.
    \STATE \textbf{Output}: $\mu_t$.
\ENDFOR 
\end{algorithmic}
\end{algorithm}

\paragraph{Preliminary Simplification} 
We will manually add noise to the samples in the following manner. At the $t$-th round, for each sample $i \in [n]$, we change $x_t^{(i)}$ to $1$ with probability $\gamma / 4$, to $0$ with probability $1/2 - \gamma/4$ and leaves it unchanged otherwise. Then, the samples after the preprocessing can be viewed as \iid~samples drawn from another binary product distribution $X'$ satisfying that $\E[X'_t] = \E[X_t]/2 + \gamma/4$. It is easy to see that then we have $\E [X'_t] \in [ \gamma/4, 3 \cdot \gamma/4 ]$. Furthermore, if our algorithm outputs $\mu'$ such that $\snorm{2}{\mu' - \E[X'] }  \leq \xi$. Then, we can easily compute $\mu$ where $\mu_t \eqdef 2 \cdot \mu'_t - \gamma/2$ such that $\snorm{2}{ \mu - \E[X]} \leq O(\xi)$. Hence, with this preprocessing step, we will assume without loss of generality that the unknown distribution $X$ satisfies $\E[X_t] \in [\gamma/4, 3 \cdot \gamma/4]$.

\paragraph{Main Algorithm} At the beginning of the $t$-th round, the algorithm divides the samples into at most $2^{t-1}$ groups based on the $0,1$ patterns of the past observations. In particular, the group $S_i^{(t)}$ consists of all samples of index $j$ satisfying $\bar x^{(j)}_{t-1} = \Binary(i)$. Within each group $i \in [2^{t-1}]$, we compute the group estimation  $\mu^{(i)}_{t} \eqdef \min \lp(   \gamma, \frac{1}{ |S_i^{(t)}| }\sum_{j \in S_i^{(t)}  } x^{(j)}_{t}  \rp)$, which is essentially the empirical mean within the group capped by $\gamma$ - the known upper bound for the true mean $\mu_t^*$. 
Then, we will compute the weighted median $\mu_t$, i.e. the median of the distribution $U$ such that $\Pr \lp[  U = \mu_{t}^{(i)} \rp] \propto \abs{ S_i^{(t)} }$.

At a high level, our algorithm relies on the following simple but useful fact. 
Let $\mathcal C$ be the set of clean samples and $\mathcal C_i^{(t)} \eqdef S_i^{(t)} \cap \mathcal C$ denote the set of clean samples within the group $S_i^{(t)}$. Then, the empirical mean of $S_i^{(t)}$ can only be far from $\mathcal C_i^{(t)}$ if there are far more adversarial samples (in $S_i^{(t)}$) having one label than the other. In other words, the adversarial samples needs to be allocated unevenly among the child group branched off from $S_i^{(t)}$ for the sample mean of $S_i^{(t)}$ to be severely corrupted.
As a result, if we were going to compute the sample means of the two child groups in the $(t+1)$-th round, one of them will be ``cleaner'' as it is less affected by the adversaries.
In long term, if the adversaries keep corrupting the sample mean of each group, the adversarial samples will get increasingly concentrated within a small fraction of groups, leaving the
sample means of the vast majority of groups relatively uncorrupted.
Though we cannot necessarily identify the cleaner groups, we can nonetheless take the median of the sample means of all groups, and the estimator will get increasingly more reliable in the future rounds as it incurs more errors in the past rounds.

 We now outline the proof which formalizes the above high-level intuition.
 First, we show that the empirical mean of the clean samples within each group are well-concentrated around the true mean (\Cref{lem:concentration}). Condition on that, we then formally spell out our observation of the relationship between the errors of the estimation of a group and the distribution of adversarial samples among its child groups and show its correctness. (\Cref{lem:eps-deviation}).
Finally, we define a potential function which intuitively measures how ``concentrated'' the adversarial samples are and couple it with the error guarantees of the algorithm (\Cref{lem:potential-increment-var}).

For the mean estimation task to be possible even without the interference of the adversaries, we need the mean of the clean samples is at least well-concentrated around the true mean. Since our algorithm breaks the samples into many groups, we require the empirical mean of the clean samples within each group to be sufficiently accurate. We show this is true with high probability.
\begin{lemma} \label{lem:concentration}
Let $\hat \mu^{(i)}_t$ be the empirical mean of the group $S_i^{(t)}$ computed from only the clean samples.
In particular, let $\mathcal C$ denote the set of un-corrupted samples. 
We define $ \hat \mu^{(i)}_t \eqdef 
\frac{1}{ \abs{ \mathcal C \cap S_i^{(t)} } }
\sum_{ i \in \mathcal C \cap S_i^{(t)}  } 
x_t^{(i)}.
$
Assume that $n \geq 2^T \cdot \poly(T, 1/\eps, 1/\gamma) \cdot \log(1/\tau)$. With probability at least $1-\tau$, for all $t$ and any group satisfying that 
$ |S_i^{(t)} \cap \mathcal C | \geq n \cdot \eps / 2^{t-1}$
,  it holds $\abs{\hat \mu^{(i)}_t - \mu_t^*  } \leq \min \lp( \eps, \gamma \rp) / T$.
\end{lemma}
\begin{proof}
The guarantee will be violated if there is any group such that (i) $|S_i^{(t)} \cap \mathcal C| \geq n \cdot \eps / 2^{t-1}$ and (ii) $\abs{\hat \mu^{(i)}_t - \mu_t^*  } \geq \min \lp( \eps, \gamma \rp) / T$.
Fix a group $S_i^{(t)}$, we argue the probability such that (i) and (ii) happens at the same time is small and conclude our proof with the union bound.
Since the probability $\Pr[A \cap B]$ for two events is always smaller than $\Pr[A | B]$, it suffices for us to argue that $\abs{\hat \mu^{(i)}_t - \mu_t^*  } \geq \min \lp( \eps, \gamma \rp) / T$ happens with small probability condition on $|S_i^{(t)} \cap \mathcal C| \geq n \cdot \eps / 2^{t-1}$.
Notice that under the condition, $\hat \mu^{(i)}_t$ is exactly the average of $n \cdot \eps / 2^{t-1}$ \iid~copies of a binary variable with mean $\mu_t^*$.
Then, by Chernoff bound, we easily have
$ 
\abs{\hat \mu^{(i)}_t - \mu_t^*  } \leq \min\lp(\eps, \gamma\rp)/T
$
with probability at least $1 - \tau / ( 10 T \cdot 2^T )$
since $n \cdot \eps / 2^{t-1} \geq \poly(T, 1/\eps, 1/\gamma)  \cdot \log(1 / \tau)$. 
Then, by union bound, this holds for all groups with probability at least $1 - \tau$ since there are at most $T \cdot 2^T$ many groups.
\end{proof}
The algorithm is deterministic once the samples are drawn. We will condition on the guarantee in Lemma~\ref{lem:concentration} being true in the proceeding analysis and show that the algorithm always succeeds. A quantity crucial to the analysis of the algorithm is the \emph{Adversarial Density} of each group.
\begin{definition}
At the $t$-th iteration, we define $\eps_i^{(t)}$, the \emph{adversarial density} of a group $S_i^{(t)}$, to be the fraction of adversarial samples within $S_i^{(t)}$.
\end{definition}
Consider the two child groups, $S_{L(i)}^{(t+1)}$ and $S_{R(i)}^{(t+1)}$ , branched from $S_i^{(t)}$ in the next round.
In particular, we have
$$
S_{L(i)}^{(t+1)} = \lp\{ j \in S_i^{(t)} \text{ such that } x_{t}^{(j)} = 1 \rp\} \, ,
S_{R(i)}^{(t+1)} = \lp\{ j \in S_i^{(t)} \text{ such that } x_{t}^{(j)} = 0 \rp\}.
$$
Assume that $S_i^{(t)}$ has enough clean samples ( $\abs{ S_i^{(t)} \cap \mathcal C } \geq \poly(T, 1/\eps, 1/\gamma) \cdot \log(1/\tau)$ )
such that the empirical mean $\hat \mu_t^{(i)}$ computed from the clean samples are close to $\mu_t^*$. Then, if the group estimation $\mu_t^{(i)}$ from the group $S_i^{(t)}$ is still far from the true mean $\mu_t^*$, it must be the case that the adversarial samples are distributed unevenly among the groups $S_{L(i)}^{(t+1)}, S_{R(i)}^{(t+1)}$.
We formalize the intuition in the argument below.
\begin{lemma} \label{lem:eps-deviation}
Let $S_i^{(t)}$ be a group satisfying that 
(i) $ \abs{S_i^{(t)} \cap \mathcal C} \geq n \cdot \eps / 2^{t-1}$ (ii) $\eps_i^{(t)} \leq 10 \eps$. 
Let $S_{L(i)}^{(t+1)}, S_{R(i)}^{(t+1)}$ be the two child groups branched from $S_i^{(t)}$.
Assume the group estimation $\mu_t^{(i)}$ is off by $\eta \eqdef \abs{ \mu_t^{(i)} - \mu_t^* } \geq 2 \cdot \min\lp( \eps, \gamma \rp)/ T$.
Then, if $ \abs{S_{L(i)}^{(t+1)}} / \abs{ S_{i}^{(t)} } \leq 5 \gamma $, it holds
$$
\abs{ \eps_{L(i)}^{(t+1)} -  \eps_{R(i)}^{(t+1)} }
\geq \Omega(\eta) \cdot \abs{ S_i^{(t)} } / \abs{ S_{L(i)}^{t+1} }.
$$
Otherwise, we have
$$
\abs{ \eps_{L(i)}^{(t+1)} -  \eps_{R(i)}^{(t+1)} }
\geq \Omega(1).
$$
\end{lemma}
\begin{proof}
Notice that $\eps_i^{(t)}$ can be viewed as the following convex combination of $\eps_{L(i)}^{(t+1)}$ and $\eps_{R(i)}^{(t+1)}$.
$$
\frac{ \abs{S_{L(i)}^{(t+1)}} }{  \abs{S_{i}^{(t)}}}
\cdot \eps_{L(i)}^{(t+1)}
+
\frac{ \abs{S_{R(i)}^{(t+1)}} }{  \abs{S_{i}^{(t)}} }
\cdot \eps_{R(i)}^{(t+1)}
= \eps_i^{(t)}.
$$
Hence, it must be that
$  \abs{ \eps_{L(i)}^{(t+1)} - \eps_{R(i)}^{(t+1)} } = 
\abs{ \eps_{L(i)}^{(t+1)} - \eps_{i}^{(t)} }
+ 
\abs{ \eps_{R(i)}^{(t+1)} - \eps_{i}^{(t)} }.
$
It hence suffices for us to lower bound $
\abs{ \eps_{L(i)}^{(t+1)} - \eps_{i}^{(t)} }$.
 Since we condition on the guarantee in \Cref{lem:concentration} being true, the first condition $ \abs{ S_i^{(t)} \cap \mathcal C } \geq n \cdot \eps / 2^{t-1} $ ensures that the empirical mean computed from the clean samples is relatively accurate.
 \begin{align} \label{eq:good-sample-fraction}
 \abs{\hat \mu_t^{(i)} - \mu_t^*} \leq \min \lp(\eps, \gamma\rp)/T.    
 \end{align}
 \textbf{Case I: $\abs{S_{L(i)}^{(t+1)}} / \abs{ S_{i}^{(t)} } > 5 \gamma$. }
 In this case, we claim that the group $S_{L(i)}^{(t+1)}$ is mostly made up of adversarial samples. 
 By Equation~\eqref{eq:good-sample-fraction}, we have that there are at most
 $$
 \hat \mu_t^{(i)} \cdot (1 - \eps_i^{(t)}) \cdot \abs{ S_i^{(t)} }
 \leq \lp( \mu_t^* + \gamma / T \rp) \cdot (1 - \eps_i^{(t)}) \cdot \abs{ S_i^{(t)} }
 \leq 2 \gamma \cdot \abs{ S_i^{(t)} }
 $$
 many clean samples. 
 On the other hand, since we have $\abs{S_{L(i)}^{(t+1)}} / \abs{ S_{i}^{(t)} } > 5 \gamma$ in this case, it holds there are at least $ 3 \gamma \cdot \abs{ S_i^{(t)} }$ many adversarial samples. Hence, the adversarial density for $S_{L(i)}^{(t+1)}$ is at least $3/5$.
 Therefore, we have $\abs{ \eps_{L(i)}^{(t+1)}
 - \eps_{i}^{(t)} } \geq \Omega(1)$. \\
 \textbf{Case II: $\abs{S_{L(i)}^{(t+1)}} / \abs{ S_{i}^{(t)} } < 5 \gamma$. }
Let $a_{L(i)}^{(t+1)}$ be the number of adversarial samples within $S_{L(i)}^{(t+1)}$ and $\tilde \mu_t^{(i)}$ as the uncapped empirical mean of the group, i.e. $ \tilde \mu_t^{(i)} = \frac{ \abs{ S_{L(i)}^{(t+1)} } }{ \abs{ S_i^{(t)} } }$.
 We can always write the number of samples in $S_{L(i)}^{(t+1)}$ as the sum of clean samples and adversarial samples.
 \begin{align} \label{eq:clean-adversarial}
\abs{ S_{L(i)}^{(t+1)} }
=  a_{L(i)}^{(t+1)}
+ \hat \mu_i^{(t)}(1 - \eps_i^{(t)}) \cdot \abs{S_{i}^{(t+1)}}.     
 \end{align}
 Rearranging Equation~\eqref{eq:clean-adversarial} then gives
$$
\abs{ a_{L(i)}^{(t+1)} - \hat \mu_i^{(t)}  \cdot  \eps_i^{(t)} \cdot \abs{S_{i}^{(t+1)}} }  =
\abs{ \tilde \mu_t^{(i)} - \hat \mu_i^{(t)}   } \cdot \abs{S_{i}^{(t+1)}}.
$$
We assume that the group estimation is off by $\eta \geq 2 \cdot \min\lp(\eps, \gamma\rp)/T$. The uncapped group mean is  off by at least that much since ``capping'' the group mean always draws it closer to the true mean $\mu_t^*$. This gives us $$
\abs{\tilde \mu_t^{(i)} - \mu_t^* } \geq \abs{\mu_t^{(i)} - \mu_t^* } =  \eta \geq 2 \cdot \min(\eps, \gamma)/T.
$$ 
On the other hand, the empirical mean of the clean samples are accurate enough such that $\abs{ \hat \mu_t^{(i)} - \mu_t^* } \leq \eps/T$.
By triangle's inequality we then have
$
\abs{ \tilde \mu_t^{(i)} - \hat \mu_i^{(t)}   } \geq \eta / 2 \, ,
$
which further implies that
\begin{align} \label{eq:absolute-deviation}
\abs{ a_{L(i)}^{(t+1)} - \hat \mu_i^{(t)}  \cdot  \eps_i^{(t)} \cdot \abs{S_{i}^{(t+1)}} }
\geq \eta/2 \cdot \abs{S_{i}^{(t+1)}}.    
\end{align}
Notice that $\eps_{L(i)}^{t+1}$ and $a_{L(i)}^{t+1}$ have the following relationship
$$
\eps_{L(i)}^{(t+1)} =
\frac{a_{L(i)}^{(t+1)}}{ a_{L(i)}^{t+1} + \hat \mu_t^{(i)} \cdot \lp( 1 - \eps_i^{(t)} \rp) \cdot \abs{S_i^{(t)}}}.
$$
Thus, we can rewrite $ \abs{ \eps_{L(i)}^{(t+1)} - \eps_i^{(t)} } $ as
$$
\abs{ \eps_{L(i)}^{(t+1)} - \eps_i^{(t)} }
= 
\abs{ \frac{a_{L(i)}^{t+1}
-
\eps_i^{(t)} \cdot \lp(  a_{L(i)}^{t+1} + \hat \mu_t^{(i)} \cdot \lp( 1 - \eps_i^{(t)} \rp)  \cdot \abs{S_i^{(t)}} \rp)
}{ a_{L(i)}^{t+1} + \hat \mu_t^{(i)} \cdot \lp( 1 - \eps_i^{(t)} \rp) \cdot \abs{S_i^{(t)}}  }  
}
=
\frac{ \abs{ a_{L(i)}^{t+1} - \eps_i^{(t)} \hat \mu_t^{(i)} \abs{S_i^{(t)}} } \cdot \lp( 1 - \eps_i^{(t)} \rp) }{ a_{L(i)}^{t+1} + \hat \mu_t^{(i)} \cdot \lp( 1 - \eps_i^{(t)} \rp) \cdot \abs{S_i^{(t)}} }.
$$
Notice that the denominator is simply $\abs{ S_{L(i)}^{(t+1)} }$. 
Hence, combining this with Equation~\eqref{eq:absolute-deviation} then gives
$$
\abs{ \eps_{L(i)}^{(t+1)} - \eps_i^{(t)} } \geq
\Omega(\eta) \cdot \frac{ \abs{S_i^{(t)}} }{\abs{ S_{L(i)}^{t+1} } }.
$$
\end{proof}

As the algorithm keeps accumulating errors, the adversarial samples will become increasingly concentrated in a small fraction of groups. Since the final output of the algorithm is given by the weighted median of the estimations from all groups, it therefore gets harder for the adversary to corrupt the estimation as the algorithm accumulates more errors. This then allows us to design a potential function based on the adversarial density to bound the total error incurred.

\paragraph{Potential Function} 
To bound the total error of the estimation, we consider the following potential function, 
\begin{align} \label{eq:potential-def}
\Phi(t) \eqdef \frac{1}{ N } \sum_{i=1}^{2^{t-1}} g_{\gamma}\lp( \eps_i^{(t)} \rp) \cdot \abs{ S_i^{(t)} } \, ,
\end{align}
where $g_{\gamma}: [0,1] \mapsto \R^+$ is the piecewise function
\begin{equation*}
g_{\gamma}(x)=
    \begin{cases}
        x^2 & \text{if } x < 10 \cdot \eps / \gamma \, ,\\
        20 \frac{\eps}{\gamma} \cdot x - 100 \lp( \frac{\eps}{\gamma}\rp)^2 & \text{otherwise.}
    \end{cases}
\end{equation*}
Here, we briefly discuss the reasons for using such a piecewise function $g_{\gamma}$ to construct the potential function.
In essence, $g_{\gamma}$ is designed to have the following properties.
\begin{claim} \label{clm:g-property}
$g_{\gamma}$ is (i) convex within the entire domain $[0,1]$  (ii) $2$-strongly convex within the interval $[0, 10 \eps /\gamma]$ (iii) upper bounded by $O\lp( \min\lp( \eps/\gamma, 1\rp) \rp)$.
\end{claim}
\begin{proof}
One can verify that for all $x,y$ where
$g_{\gamma}'(x), g_{\gamma}'(y)$ is well-defined, we have $g_{\gamma}'(x) < g_{\gamma}'(y)$ as long as $x<y$.
In particular, for all $x < 10 \eps/\gamma$, $g_{\gamma}'(x) = 2x$, which is indeed a monotonically increasing function. 
For all $x 10 \eps/\gamma$, we have $g_{\gamma}'(x) = 20 \eps/\gamma$, which is constant. 
Moreover, $2 \cdot 10 \eps/\gamma = 20 \eps/\gamma$. Hence, for any $x < 10 \eps/\gamma$ and $y \geq 10 \eps/\gamma$, we always have $g'_{\gamma}(x) \leq g_{\gamma}'(y)$.
Besides, at the kink $c = 10 \eps / \gamma$ , we have
$ 
\lim_{x \rightarrow c^{+}} g_{\gamma}(x)
= \lim_{x \rightarrow c^{-}} g_{\gamma}(x)
$, showing that $g_{\gamma}$ is a continuous function. 
The convexity of $g_{\gamma}$ then follows. Within the interval $[0, 10 \eps / \gamma]$, $g_{\gamma}$ is simply the quadratic function $x^2$. Hence, it is $2$-strongly convex. Finally, we derive the upper bound for $g_{\gamma}$ through a case analysis. 
Since $g_{\gamma}$ is monotonically increasing, $\max_x g_{\gamma}(x)$ is always attained at $g(1)$.
When $10 \eps/\gamma > 1$. $g_{\gamma}(x) = x^2$ over the entire domain $[0,1]$.  We then have $\max_x g_{\gamma}(x) = g_{\gamma}(1) = 1$. When $10 \eps/\gamma < 1$, we have $g_{\gamma}(1)
\leq 20 \eps/ \gamma$. Both quantities are of order $O\lp( \min(\eps/\gamma , 1) \rp)$ in their regimes, therefore giving the desired upper bound.
\end{proof}
When a group $S_i^{(t)}$ splits into two child groups $S_{L(i)}^{(t+1)}$, $S_{R(i)}^{(t+1)}$, we always have that 
$$
\eps_i^{(t)} = \frac{ \abs{S_{L(i)}^{(t+1)}} }{ \abs{S_i^{(t)}} } \cdot \eps_{L(i)}^{(t+1)}
+
\frac{ \abs{S_{R(i)}^{(t+1)}} }{ \abs{S_i^{(t)}} } \cdot \eps_{R(i)}^{(t+1)}.
$$
Therefore, the convexity of $g_{\gamma}$ then ensures that the total contribution from the child groups is at least the contribution from the parent group.
, making $\Phi$ a valid non-decreasing ``potential''. 
Besides, $g_{\gamma}(x)$ is locally strongly convex. 
This ensures the contribution to the potential will increase substantially if the adversarial densities between the two child groups differ by a lot (given that their adversarial densities are still within the strongly convex region of $g_{\gamma}$ ). Lastly, the upper bound on $g_{\gamma}$ allows us to derive tight upper bound for the potential function, which is essential in obtaining the optimal error bound for the algorithm. We give the upper bound on $\Phi$ below.
\begin{claim}
\label{lem:modified-potential-bound}
$ \Phi(t) \leq O(1) \cdot \min\lp(\eps, \eps^2/\gamma \rp) $ for all $t \in [T]$.
\end{claim}
\begin{proof}
Notice that we always have the equality
$$
\sum_i \frac{ \abs{ S_i^{(t)} } }{n} \cdot \eps_i^{(t)} = \eps.
$$
Since $g_{\gamma}$ is a convex function, it is not hard to see that the potential function is maximized when we have $\eps$ fraction of groups that are made entirely of adversarial samples. This then gives that
$$
\Phi(t) \leq \eps \cdot g_{\gamma}(1)
\leq  O(1) \cdot \min\lp(\eps, \eps^2/\gamma \rp).
$$
\end{proof}

We next show how we can couple the increment of the potential function and the estimation error incurred.
At a high level, if our algorithm outputs
$\mu_t$ such that it incurs error~$\eta \eqdef \abs{ \mu_t - \mu_t^* }$, more than half of the group estimations $\mu_t^{(i)}$ must also be off by at least $\eta$ since $\mu_t$ is obtained by computing the median over $\mu_t^{(i)}$.
As illustrated in \Cref{lem:eps-deviation}, given that such an erroneous group also satisfies some other technical conditions, the adversarial density for one of its child group must be substantially higher than the other. Then, strong-convexity of $g_{\gamma}$ will ensure the contributions to the potential from the child groups must be significantly higher than that from the parent group.
One slight issue of the above argument is that the adversarial densities of these erroneous groups (and their child groups) may be well above the threshold $10 \eps / \gamma$. For such a group, even if the split of the adversarial samples is vastly uneven between the two child groups, their overall contribution to the potential remains the same since both of them are in the linear regime for $g_{\gamma}$. Fortunately, there cannot be too many groups with high adversarial densities, and it suffices for us to look at only the increments gained from groups with relatively low adversarial densities.
\begin{lemma}
\label{lem:potential-increment-var}
$ \Phi(t+1) - \Phi(t) \geq \Omega(1/\gamma) \cdot \lp( \mu_t - \mu_t^* \rp)^2 $ if $\abs{\mu_t - \mu_t^*} \geq 2 \cdot \min\lp( \eps, \gamma \rp)/T$.
\end{lemma}
\begin{proof}
We first introduce some notations.
Let $S_{L(i)}^{(t+1)}, S_{R(i)}^{(t+1)}$ be the two child groups branched off from the parent group $S_i^{(t)}$.
Let 
$
\eps_{L(i)}^{(t+1)} \, , \eps_{R(i)}^{(t+1)}
$ be their corresponding adversarial densities, and $ \hat \mu_t^{(i)}$ be the empirical mean computed from the clean samples from the parent group $S_i^{(t)}$.
For each group $ S_i^{(t)} $, we define its increment as
\begin{align} \label{eq:def-increment}
    \Delta \lp( S_i^{(t)} \rp)
    \eqdef     
    \lp( \frac{ |S_{L(i)}^{(t+1)}| }{ n }\rp) \cdot
    g_{\gamma}\lp( \eps_{L(i)}^{(t+1)} \rp)
    +
    \lp( \frac{ |S_{R(i)}^{(t+1)}| }{ n } \rp) \cdot 
    g_{\gamma}\lp( \eps_{R(i)}^{(t+1)} \rp)
    - \lp(\frac{ \abs{  S_i^{(t)} } }{ n }\rp)  \cdot
    g \lp( \eps_i^{(t)} \rp).
\end{align}

Consider the groups satisfying the following conditions (i) $\eps_i^{(t)} \leq 5 \eps$ (ii) the estimation $\mu_t^{(i)}$ is off from $\mu_t^*$ by at least 
$ \abs{\mu_t^{(i)} - \mu_t^*} \geq
\eta \eqdef \abs{\mu_t - \mu_t^*}$, which is by our assumption at least $ 2 \cdot \min(\eps, \gamma)/T$, and (iii) the number of clean samples is at least $\abs{ S_i^{(t)} \cap \mathcal C } \geq n \cdot \eps / 2^{t-1}$.
Notice that the total weight of groups satisfying condition (i) is at least $1 - 1/5$, the total weight of groups satisfying condition (ii) is at least $1/2$.
For condition (iii), we claim the total weight of groups satisfying that is at least $1 - 2\cdot \eps$.
Since there are only $\eps$ fraction of adversarial samples, we have
$\sum_i \abs{ S_i^{(t)} \cap \mathcal C } = n \cdot (1 - \eps)$.
Let $\mathcal H$ be the set of groups which satisfy the condition. Then, we have
$$\sum_{i \in \mathcal H} \abs{ S_i^{(t)} \cap \mathcal C } \geq n \cdot (1 - \eps)
- \sum_{i \not \in \mathcal H} \abs{ S_i^{(t)} \cap \mathcal C }
\geq n \cdot (1 - \eps) - n \cdot \eps \, ,
$$
where in the second inequality we use the fact that $\abs{ S_i^{(t)} \cap \mathcal C } \leq n \cdot \eps / 2^{t-1}$ for $i \not \in \mathcal H$ and there are at most $2^{t-1}$ many groups.
Then, our claim easily follows from the fact that
$\abs{ S_i^{(t)} } \geq \abs{ S_i^{(t)} \cap \mathcal C }$.
Denote the set of groups satisfying all three conditions as $G$. 
By union bound, it is not hard to see that the fraction of groups satisfying the above three conditions is at least 
\begin{align} \label{eq:G-fraction}
\frac{1}{n} \sum_{i \in G} \abs{S_i^{(t)}} \geq 1/5.
\end{align}

We will show that, for each group $i \in G$, the contribution to the potential function from the two child groups branched off from $i$ in the next round is significantly higher than the contribution from the $i$-th group at the current round.
In particular, for all $i  \in G$, we claim its increment is at least
\begin{align} \label{eq:increasing}
\Delta\lp( S_i^{(t)} \rp)
\geq \frac{ \abs{  S_i^{(t)} } }{ n } 
\cdot \Omega(\eta^2 / \gamma).
\end{align}
For any other groups $i \not \in G$, we instead show the contributions to the potential are non-decreasing.
In particular, for all $i  \not \in G$, we claim
\begin{align} \label{eq:non-decreasing}
\Delta\lp( S_i^{(t)} \rp)
\geq 0.
\end{align}
Again, $\eps_i^{(t)}$ can be viewed as the following convex combination of $\eps_{L(i)}^{(t+1)}, \eps_{R(i)}^{(t+1)}$.
$$
\frac{ S_{L(i)}^{(t+1)} }{  S_{i}^{(t)} }
\cdot \eps_{L(i)}^{(t+1)}
+
\frac{ S_{R(i)}^{(t+1)} }{  S_{i}^{(t)} }
\cdot \eps_{R(i)}^{(t+1)}
= \eps_i^{(t)} \, ,
$$
and the increment can be rewritten as
\begin{align} \label{eq:increment-rewrite}
    \Delta \lp( S_i^{(t)} \rp)
    =     
    \frac{ \abs{S_i^{(t)}} }{n}
    \cdot \lp( 
    \frac{ |S_{L(i)}^{(t+1)}| }{ \abs{S_i^{(t)}} } \cdot
    g_{\gamma}( \eps_{L(i)}^{(t+1)} )
    +
    \frac{ |S_{R(i)}^{(t+1)}| }{ \abs{S_i^{(t)}} } \cdot 
    g_{\gamma}( \eps_{R(i)}^{(t+1)} )
    - 
    g_{\gamma} ( \eps_i^{(t)} )
    \rp)    \;.
\end{align}
Hence, Equation~\eqref{eq:non-decreasing} immediately follows from the convexity of $g_{\gamma}$.

Next, we proceed to show Equation~\eqref{eq:increasing}. 
By \Cref{clm:g-property}, $g_{\gamma}$ is $2$-strongly convex within the interval $[0, 10 \eps/\gamma]$.
We will show that $\eps_{L(i)}^{(t+1)}, \eps_{R(i)}^{(t+1)}$ are both within the region where $g_{\gamma}$ is strongly convex. 
By our choice of the group $i$, we have 
$ \eps_i^{(t)} \leq 5 \eps$, and $\abs{S_i^{(t)} \cap \mathcal C} \geq n \cdot \eps / 2^{t-1}$. Then, by \Cref{lem:concentration}, it holds that 
$\abs{\hat \mu_t^{(i)} - \mu^*_t} \leq  \gamma/T$.
Then, we can upper bound the adversarial densities by
\begin{align*}
    &\eps_{L(i)}^{(t+1)} \leq \eps_i^{(t)} / (\hat \mu_t^{(i)}) 
    \leq \frac{5 \eps}{ \gamma - \gamma/T }
    \leq 10 \eps  / \gamma \, ,\\
    &\eps_{R(i)}^{(t+1)} \leq \eps_i^{(t)} / (1 - \hat \mu_t^{(i)}) \leq \frac{5 \eps}{ 1 - \gamma/4 - \gamma/T } \leq 10 \eps \, ,
\end{align*}
where we have utilized the facts $\mu^*_t \in [\gamma/4, \gamma ]$ (by our preliminary simplification) and $ \hat \mu_t^{(i)} =  \mu^*_t \pm \gamma/T$.
In this regime, we have $g_{\gamma}(x) = x^2$ is a $2$-strongly convex function. 
This implies that for any $x,y,z \in [0, 10 \eps/\gamma]$ and $\alpha \in [0,1]$ satisfying that
$  z = \alpha \cdot x + (1 - \alpha) \cdot y $, we always have
$$
\alpha \cdot g_{\gamma}(x) 
+ (1 - \alpha) \cdot g(y) - g(z)
\geq \alpha \cdot (1 - \alpha) (x-y)^2.
$$
Applying this fact with $\alpha = \frac{ \abs{S_{L(i)}^{(t+1)} } }{ \abs{S_i^{(t)}} }$, $x = \eps_{L(i)}^{(t+1)}$, $y = \eps_{R(i)}^{(t)}$ and $z= \eps_{i}^{(t+1)}$ to Equation~\eqref{eq:increment-rewrite}
then gives us
the increment for any group $i \in G$ is at least
\begin{align} \label{eq:strong-convexity-bound}
\Delta\lp( S_i^{(t)} \rp)
\geq
\Omega(1) \cdot \frac{ \abs{S_{i}^{(t)}} }{ n }
\cdot \frac{ \abs{S_{L(i)}^{(t+1)}} }{  \abs{S_{i}^{(t)}} }
\cdot 
\frac{ \abs{S_{R(i)}^{(t+1)}} }{  \abs{S_{i}^{(t)}} }
\lp( 
\eps_{L(i)}^{(t+1)} - \eps_{R(i)}^{(t+1)}
\rp)^2.
\end{align}
 \textbf{Case I: $\abs{S_{L(i)}^{(t+1)}} / \abs{ S_{i}^{(t)} } > 5 \gamma$. }
 By \Cref{lem:eps-deviation}, we have
 $\abs{ 
\eps_{L(i)}^{(t+1)} - \eps_{R(i)}^{(t+1)}
} \geq \Omega( 1 )$. 
Besides, we can lower bound
$\abs{S_{R(i)}^{(t+1)}}$ by the number of clean samples in it, which then gives
\begin{align*}
\frac{ \abs{S_{R(i)}^{(t+1)} } }{ \abs{S_{i}^{(t)} } }
&\geq \lp( 1 - \hat \mu_t^{(i)} \rp) \cdot \lp( 1 - \eps_i^{(t)}\rp)
\geq 
\lp( 1 - \mu_t^* - \gamma/T \rp) \cdot \lp( 1 - \eps_i^{(t)}\rp)  \\
& \geq  
\lp( 1 - 3 \cdot \gamma/4 - \gamma/T \rp) \cdot \lp( 1 - 5 \eps \rp) 
\geq \Omega(1) \, ,
\end{align*}
where the second inequality holds since $\abs{\hat \mu_t^{(i)} - \mu_t^* } \leq \gamma/T$, the third inequality holds since $\eps_i^{(t)}$ by our choice of the group and $\mu_t^* \leq 3 \cdot \gamma/4$ by our preliminary simplification step.
By our assumption of the case, we have  
$
\frac{ \abs{S_{L(i)}^{(t+1)} } }{ \abs{S_{i}^{(t)} } } \geq \Omega(\gamma).
$
Therefore, substituting the bounds for 
$\frac{ \abs{S_{L(i)}^{(t+1)} } }{ \abs{S_{i}^{(t)} } }, \frac{ \abs{S_{R(i)}^{(t+1)} } }{ \abs{S_{i}^{(t)} } }, $ and
$\abs{ 
\eps_{L(i)}^{(t+1)} - \eps_{R(i)}^{(t+1)}
}
$
into Equation~\eqref{eq:strong-convexity-bound} then gives
the increment is at least $\frac{ \abs{S_{i}^{(t)}} }{ n } \cdot \Omega(\gamma)$.
On the other hand, since both the estimation $\mu_t$ (since all the group estimations are capped by $\gamma$ ) and the true mean $\mu_t^*$ are upper bounded by $\gamma$, we have 
$ \eta \leq \gamma $. We then have the increment is at least $$ 
\Delta\lp( S_i^{(t)} \rp) \geq \frac{ \abs{S_{i}^{(t)}} }{ n } \cdot \Omega(\gamma) 
=
\frac{ \abs{S_{i)}^{(t+1)}} }{ n } \cdot \Omega(\gamma^2 / \gamma)
\geq 
\frac{ \abs{S_{i}^{(t)}} }{ n } \cdot \Omega(\eta^2 / \gamma).
$$ \\
 \textbf{Case II: $\abs{S_{L(i)}^{(t+1)}} / \abs{ S_{i}^{(t)} } < 5 \gamma$. }
  By \Cref{lem:eps-deviation}, we have
 $$
\abs{
\eps_{L(i)}^{(t+1)} - \eps_{R(i)}^{(t+1)}
} \geq 
\abs{ \mu_t^{(i)} - \mu_t^* }
\cdot \frac{ \abs{S_{i}^{(t)} } }{ \abs{S_{L(i)}^{(t)} } }
\geq
\Omega( \eta ) \cdot 
\frac{ \abs{S_{i}^{(t)} } }{ \abs{S_{L(i)}^{(t)} } } \, ,
$$
where the second inequality follows from our choice of the group $i$ such that 
$\abs{\mu_t^{(i)} - \mu_t^*} \geq \eta$.
Similar to the last case, we always have
$\frac{ \abs{S_{R(i)}^{(t+1)} } }{ \abs{S_{i}^{(t)} } } \geq \Omega(1)$.
Substituting the bounds for 
$\frac{ \abs{S_{R(i)}^{(t+1)} } }{ \abs{S_{i}^{(t)} } }, $ and
$\abs{ 
\eps_{L(i)}^{(t+1)} - \eps_{R(i)}^{(t+1)}
}
$
into Equation~\eqref{eq:strong-convexity-bound} then gives
$$
\Delta\lp( S_i^{(t)} \rp) \geq  \frac{ \abs{S_{i}^{(t)} } }{ n } \cdot \Omega(\eta^2) \cdot \frac{ \abs{S_{i}^{(t)} } }{ \abs{S_{L(i)}^{(t)} } } \geq
\frac{ \abs{S_{i}^{(t)} } }{ n }
\Omega(\eta^2 / \gamma) \, ,
$$
where the last inequality follows from our case assumption $\abs{S_{L(i)}^{(t+1)}} / \abs{ S_{i}^{(t)} } < 5 \gamma$.

As our final step, we can then lower bound the total increment of the potential function as
\begin{align*}
\Phi(t+1) - \Phi(t)
&= \sum_{i=1}^{2^{t-1}} \Delta \lp( S_i^{(t)} \rp)
= \sum_{i \in G} \Delta \lp( S_i^{(t)} \rp)
+
\sum_{i \not \in  G} \Delta \lp( S_i^{(t)} \rp) \\
&\geq 
\sum_{i \in G} \frac{\abs{ S_i^{(t)} } }{n}
\cdot \Omega(\eta^2 / \gamma)
\geq  \Omega(\eta^2 / \gamma) \, ,
\end{align*}
where the first equality follows from our definition of increment in Equation~\eqref{eq:def-increment}, the first inequality follows from Equations~\eqref{eq:non-decreasing} and~\eqref{eq:increasing}, and the second inequality follows from Equation~\eqref{eq:G-fraction}.
\end{proof}
Now, we can conclude the proof of \Cref{lem:binary-product-estimation}.
\begin{proof}[Proof of \Cref{lem:binary-product-estimation}]
From \Cref{lem:potential-increment-var}, we know that
$$
\sum_{t=1}^T \lp(\mu_t - \mu_t^* \rp)^2
\leq \sum_{t=1}^T  O \lp( \min\lp(\eps, \gamma\rp)^2/T^2 \rp) +  O(\gamma) \cdot \lp( \Phi(t) - \Phi(t-1) \rp) 
\leq O(\min\lp(\eps, \gamma\rp)^2/T) + O(\gamma) \cdot \Phi(T).
$$
By \Cref{lem:modified-potential-bound}, we know $\Phi(T) \leq O(1) \cdot \min(\eps, \eps^2/\gamma )$.
Substituting that into the equation above then gives the desired bound on $\snorm{2}{ \mu - \mu^*}$.
\end{proof}

\subsection{Identity Covariance Gaussians} \label{ssec:gaussian}

Estimating the mean of an isotropic Gaussian distribution 
is a widely studied question in the field of algorithmic robust statistics. 
In the offline model, the Tukey median
robustly estimates the mean up to error $O\lp(\eps\rp)$ in $\ell_2$-distance,
which matches the information-theoretic limit of the task up to constant factors. 
In this section, we show that the $O\lp(\eps\rp)$ error is still achievable in the online setting. 

At a high level, our algorithm reduces the problem to estimating the mean of binary product distributions. 
The reduction leverages the following fact about a (1d) Gaussian distribution: 
the cumulative density function of a (1d) Gaussian distribution 
is an invertible function of its mean and is Lipschitz 
within an interval of constant length around the mean. 
That being said, if we are able to robustly estimate the probability 
$\Pr[X_t \leq q_t]$ for some $q_t$ that is within constant distance 
from $\mu^*_t$, we can then feed the estimation into the inverse of the Gaussian CDF function to retrieve a robust estimation of $\mu^*_t$. 
It is not hard to see that estimating $\Pr[X_t \leq q_t]$ for all $t$ is exactly the same as estimating the mean of the binary product distribution defined as $Y_t \eqdef \mathbbm 1 \{ X_t \leq q_t \}$. Therefore, the only thing remaining is for us to find such $q_t$ that is within constant distance from $\mu^*_t$. Fortunately, any robust 1d-estimator (such as the median) achieves the goal easily.
\begin{theorem} \label{thm:optimal-gaussian}
Let $\eps, \tau \in (0,1)$.
Suppose $X$ is a $T$ dimensional Gaussian distribution with an unknown mean vector $\mu^*$ and identity covariance.
Then, for sufficiently small $\epsilon$, there exists an algorithm which robustly estimates the mean of $X$ under $\eps$ corruption in the online setting with accuracy $O\lp( \eps \rp)$, failure probability $\tau$  and sample complexity 
$$
n \geq 2^T \cdot \poly(T, 1/\eps) \cdot \log(1/\tau).
$$
\end{theorem}
\begin{proof}
First, we discuss a preprocessing step that allows us to assume without loss of generality that $\mu^*_t \leq O(\eps)$ for all $t \in [T]$.
To do so, we will reserve $\poly(1/\eps) \cdot \log(T/\tau)$ many samples for ``calibration''.
At the $t$-th round, we can use any robust $1d$ estimators on the reserved samples to output an estimation $\hat \mu_t$ 
satisfying that $\abs{\hat \mu_t - \mu^*_t } \leq O(\eps) $ with probability at least $1 - \tau/(10T)$. Then, we can subtract $\hat \mu_t$ out from the $t$-th coordinate of the rest of the samples. The samples after the subtraction would then follow a Gaussian distribution where the mean of each coordinate is bounded by $O(\eps)$, and it is easy to see that estimating the mean of this Gaussian is equivalent to solving our original estimation problem.

Let $x^{(i)}$ be an un-corrupted sample. 
It is not hard to see that  $\E\lp[ \mathbbm 1 \{  x^{(i)}_t > 0 \}\rp] = \Pr \lp[ X_t > 0  \rp]$. 
At the $t$-th round, we can then feed $y^{(i)}_t \eqdef \mathbbm 1 \{  x^{(i)}_t > 0 \}$ for all $i$ in the remaining samples to Algorithm~\ref{alg:binary-product-estimation}. The result will be an estimator $\tilde Y_t$ satisfying that 
\begin{align} \label{eq:quantile-error}
\sum_{t=1}^T \lp( \tilde Y_t - \Pr \lp[ X_t > 0  \rp] \rp)^2 \leq O(\eps^2).
\end{align}
On the other hand, the quantity 
$\Pr \lp[ X_t > 0  \rp]$ is precisely 
$\erf\lp( \mu^*_t \rp)$ where $\erf$ is the error function defined as 
$$
\erf\lp( u \rp) = \frac{1}{\sqrt{2\pi}} \int_{x=0}^{\infty} \exp \lp( - \lp( x - u\rp)^2/2 \rp) dx.
$$
Hence, if we let the algorithm output 
$ \mu_t =   \erf^{-1}\lp( \tilde Y_t  \rp)$, the error will be at most
\begin{align*}
\sum_{t=1}^T \lp( \mu_t - \mu_t^* \rp)^2
& = 
\sum_{t=1}^T \lp( \erf^{-1}\lp( \tilde Y_t  \rp) - 
\erf^{-1}\lp( \Pr \lp[ X_t > 0  \rp]  \rp)
\rp)^2\\
&\leq 
O(1) \cdot 
\sum_{t=1}^T \lp( 
\tilde Y_t - \Pr \lp[ X_t > 0  \rp]
\rp)^2
\leq O(\eps^2) \, ,
\end{align*}
where the first inequality is by the fact that $\erf^{-1}$ is $\Theta(1)$-Lipchitz within the interval $[1/4, 3/4]$, 
$\Pr[X_t  > 0] \in [1/4, 3/4]$ since $\abs{\mu^*} = O(\eps)$, $\tilde Y_t \in [1/4,3/4]$ 
since $\abs{\tilde Y_t - \Pr[X_t > 0]} \leq O(\eps)$, and the second inequality is by Equation~\eqref{eq:quantile-error}.
\end{proof}

\subsection{More General Product Distributions} \label{ssec:nonpar}
In this subsection, we give an inefficient online robust mean estimation 
algorithm for product distributions whose coordinates come from 
nonparametric distribution families satisfying mild concentration properties.

In particular, we present a meta-algorithm that works with coordinate-wise independent distributions with good tail bounds. After that, we will show how the meta-algorithm can be instantiated to obtain informational theoretically optimal error rates for sub-gaussian distributions and distribution with bounded moments (still assuming each coordinate is independent).
To abstract out the properties of the unknown distribution needed by the algorithm, we give the following definition of $F$-tail bound product distributions.

\begin{definition}[$F$-tail bound product distributions] \label{def:f-tail-bound}
Let $X$ be a $T$-dimensional coordinate-wise independent distributions with mean $\mu^*$. Namely, it is the product of $T$ independent distribution $X_1, \cdots, X_T$.
Let $F$ be some monotonically decreasing function $F: \R^+ \mapsto [0,1]$.
We say $X$ is an $F$-tail bound product distribution if each of the univariate distribution $X_t$ satisfies the tail bound 
$
    \Pr \lp[ 
    \abs{X_t} \geq q
    \rp]
    \leq F(q)
$ \footnote{We define the tail bound assuming the univariate distribution $X_t$ is ``centered'' around $0$.
We remark this is a mild assumption as it is always possible to use a $1$d robust estimator to calibrate the distribution so that it is approximately centered around $0$.}.
\end{definition}

In general, the faster the tail bound $F$ decreases, 
the more concentrated the distribution is and the better our algorithm behaves. 
More specifically, under $\eps$ corruption, the accuracy of the algorithm will be given by
$$
Q_{F,\eps} = \int_{0}^{\infty} \min\lp(\eps, \sqrt{ \eps F(q) } \rp) dq \;.
$$
For this reason, we do require the tail bound $F$ 
to be good enough such that the above integral is at least convergent. 

\begin{theorem} \label{thm:f-tail}
Let $\eps, \tau \in (0,1)$., $F$ be some monotonically decreasing function $F: \R^+ \mapsto [0,1]$ 
such that $Q_{F,\eps} \eqdef \int_{0}^{\infty} \min\lp(\eps, \sqrt{ \eps F(q) } \rp) dq$ 
is convergent. Suppose $X$ is an $F$-tail bound product distribution.
Then, for sufficiently small $\epsilon$, there exists an algorithm \textbf{Non-parametric-Estimation}
(Algorithm~\ref{alg:meta-alg})
which robustly estimates the mean of $X$ under $\eps$ corruption 
in the online setting with accuracy $O\lp( Q_{F, \eps}\rp)$, 
failure probability $\tau$, and sample complexity 
$$
n \geq 
2^T \cdot \poly(T, 1/\eps, 1/F(L)) \cdot \log(L/(Q_{F,\eps}\cdot\tau))\, ,
$$
where $L \eqdef  \inf_{z} \lp( \int_{q=z}^{\infty} F(q) dq \leq \frac{1}{\sqrt{T}} Q_{F, \eps} \rp)$.
\end{theorem}

We next discuss the components for the algorithm.
A key property used in obtaining \Cref{thm:optimal-gaussian} 
is that one can uniquely recover the mean of the unknown distribution 
given its cumulative density function evaluated at a point. 
This is no longer the case for nonparametric families of distributions. 
Nonetheless, we claim it is still possible to approximately recover 
the mean if we have access to the  distribution's cumulative density function 
at many different points. The high-level idea is to rely on the following folklore 
inequality that relates a random variable's mean and its cumulative distribution function. 
\begin{claim} \label{clm:folklore}
Let $U$ be a one dimensional random variable. Then it holds
$$
\E[U] = \int_{0}^{\infty} \Pr[ U \geq q ] dq +
\int_{0}^{-\infty} \Pr[ U \leq q ] dq \;.
$$
\end{claim}

The above integral would directly give us a way of computing the mean 
if the random variable is discrete and of bounded support 
(as the integral would have a closed form that can be evaluated 
with finite many queries to the variable's CDF).
For continuous distributions following proper tail bounds, 
we can nonetheless still try to approximate the integral with its Riemann sum.

\begin{definition}[$n$-Rectangle Riemann Sum]
Let $f: [a,b] \mapsto \R$ be a continuous function. 
The $n$-rectangle left and right Riemann sums of the integral 
$\int_{a}^b f(x) dx$ is defined as 
$$ \LRS(f,a,b,n) = \sum_{i=1}^n f(x_{i-1}) \cdot \lp( b-a\rp)/n \, ,
\RRS(f,a,b,n)
= \sum_{i=1}^n f(x_{i}) \cdot \lp( b-a\rp)/n \;, 
$$
where $x_0, \cdots, x_n$ partition $[a,b]$ into intervals of equal sizes.
\end{definition}

The following result on the approximation error of Riemann Sum is standard.

\begin{lemma} \label{lem:Riemann}
Suppose $f$ is integrable on $[a,b]$ and let $n$ be a positive integer. Then,
if $f$ is monotonically increasing (or decreasing), we have
$$
\abs{ \int_a^b f(x) dx - \RRS \lp( f,a,b,n \rp) }
\leq \abs{ \lp(f(b) - f(a)\rp) \cdot \lp( b-a\rp) }/n \;.
$$
The same bound holds for $\LRS \lp( f,a,b,n \rp)$.
\end{lemma}

Notice that for the equation in \Cref{clm:folklore}, 
$\Pr[ U \geq q ]$ is monotonically decreasing and $\Pr[U \leq q]$ 
is monotonically increasing (with respect to $q$). 
Hence, we can approximate the two parts separately with the Riemann Sum.
One slight issue is that the domain of the integral may be infinite. 
We note that, if the random variable satisfies proper tail bounds, 
we can restrict the domain to some finite interval $[-L,L]$ 
and create only negligible bias to our approximation if $L$ is large enough.

Now, we go back to the problem of robustly estimating the mean of an $F$-tail bound product distribution $X$ in the online setting. 
The high-level idea of the algorithm is the following.
For each $t \in [T]$, we define the indicator variable
$Y(q)_t = \mathbbm 1\{ X_t \geq q \}$ if $ q < 0 $ and $Y_t(q) = \mathbbm 1 \{X_t \leq q\}$ if $q \geq 0$.
Notice that we have exactly $\E \lp[ Y_t(q) \rp] = \Pr[ X_t \leq q ]$ for $q \geq 0$, which corresponds to the $X_t$'s CDF function at $q$.
If we are able to estimate the mean of $Y(q)_t$, this then gives us (noisy) query access to the CDF function of $X_t$. We can then leverage the Riemann Sum approximation of the integral in \Cref{clm:folklore} to further compute the mean of $X_t$.

The remaining task is then to estimate each $Y(q)_t$ up to good accuracy.
This is made possible with the following observation: Fixing some $q \in \R^+$, the variables $Y(q)_1, \cdots, Y(q)_T$ form a binary product distribution.
By \Cref{lem:binary-product-estimation}, we can then compute a series of estimators $\tilde Y(q)_1, \cdots, \tilde Y(q)_T$ in the online setting such that the total error is at most $ \snorm{2}{Y(q) - \tilde Y(q)} = O( \eps )$.
Though the estimation error for one $Y(q)$ is now independent of $T$, 
the total error may still get out of control as the errors for different $Y(q)$ add up linearly in the Riemann Sum. 
Fortunately, we have the extra condition that $Y(q)$ is $F(q)$-bounded by the tail bound of $X$.
By \Cref{lem:binary-product-estimation}, 
our estimation accuracy naturally improves as $F(q)$ becomes smaller. In particular, the accuracy is given by $\snorm{2}{\tilde Y(q) - Y(q)} \leq \min \lp( \eps, \sqrt{ \eps \cdot F(q) } \rp)$.
Therefore, as long as the integral
$\int_{0}^{ \infty } \min \lp( \eps, \sqrt{ \eps \cdot F(q) } \rp) dq $ is convergent, our total estimation error remains a quantity independent of $T$. We now give the algorithm and its analysis, which constitutes the proof of \Cref{thm:f-tail}.
\begin{algorithm}[h] 
\caption{Non-parametric-Estimation}
\label{alg:meta-alg}
\begin{algorithmic}[1]
\STATE \textbf{Input:} corruption $\epsilon$, round number $T$, $n$ samples whose coordinates are revealed one by one in each round.
\STATE Set $Q_{F,\eps} = \int_{q=0}^{\infty} \min\lp(\eps, \sqrt{ \eps \cdot F(q) } \rp) dq$, $L = \inf_z \lp( \int_{q=z}^{\infty} F(q) dq  \leq Q_{F, \eps} / \sqrt{T} \rp)$,
$m = \floor{ L \cdot \sqrt{T} / Q_{F, \eps} }$.
\STATE Choose $q_0, \cdots, q_m$ such that the points partition $[0, L]$ into intervals of equal size.
\FOR{$t=1,2,...,T$}
   \STATE In the $t$-th round, $x_t^{(1)}, \cdots, x_t^{(n)}$ are revealed.
   \FOR{$i=1,\cdots,m$}
   \STATE Compute the samples for $Y(q_i)_t \eqdef \mathbbm 1 \{ X_t \geq q_i \}$,
   $Y(-q_i)_t \eqdef \mathbbm 1 \{ X_t \leq -q_i \}$.
   $$
   y(q_i)_t^{(j)} = \mathbbm 1 \{x_t^{(j)} \geq q_i\} \, \forall j\in [n]\, ,
   y(-q_i)_t^{(j)} = \mathbbm 1 \{x_t^{(j)} \leq -q_i\} \, \forall j\in [n].
   $$
   \STATE Compute the robust estimator
   \begin{align}
   &\tilde Y(q_i) = \textbf{Binary-Estimation}
   \lp( y(q_i)_t^{(1)}, \cdots, y(q_i)_t^{(n)}, \gamma = F\lp(q_i\rp)\rp) \, , \\
   &\tilde Y(-q_i) = \textbf{Binary-Estimation}
   \lp( y(-q_i)_t^{(1)}, \cdots, y(-q_i)_t^{(n)}, \gamma = F\lp(q_i\rp)\rp) .
   \end{align}
   \ENDFOR
   \STATE 
   $\mu_t = 
   \sum_{i=1}^m  
   \tilde Y(q_{i})
   \cdot L/m
   - 
   \sum_{i=1}^m 
   \tilde Y(-q_{i})
   \cdot L/m.
   $
    \STATE \textbf{Output}: $\mu_t$.
\ENDFOR 
\end{algorithmic}
\label{alg: eff_filter}
\end{algorithm}

\begin{proof}[Proof of \Cref{thm:f-tail}]
For $q \in \R$, define the indicator variables
$$
Y(q)_t =
    \begin{cases}
        \mathbbm 1 \{  X_t \geq q \} & \text{if } q \geq 0 \, , \\
        \mathbbm 1 \{ X_t \leq q \} &
        \text{if } q < 0.
    \end{cases}
$$
Recall that in Algorithm~\ref{alg:meta-alg} we take
$$
Q_{F,\eps} = \int_{q=0}^{\infty} \min\lp(\eps, \sqrt{ \eps \cdot F(q) } \rp) dq,L = \inf_z \lp( \int_{q=z}^{\infty} F(q) dq  \leq Q_{F, \eps} / \sqrt{T} \rp),m = \floor{ L \cdot \sqrt{T} / Q_{F, \eps} } \, ,
$$
and $q_0, \cdots, q_m$ such that the points partition $[0, L]$ into intervals of equal size.
Now, consider a hypothetical estimator $\hat \mu$ defined as
$$
\hat \mu_t = 
\sum_{ i=1 }^m
\E [ Y( q_{i-1} )_t ]
\cdot L/m
-
\sum_{ i=1 }^m
\E [ Y( -q_i )_t ]
\cdot L/m
$$
Since $\E[Y(q_i)_t] = \Pr[ X_t \geq q_i ]$ and $\E[ Y(-q_i)_t ] = \Pr[X_t \leq q_i]$,
the two terms correspond to the $m$-rectangle Riemann Sum of 
$\int_{0}^{L} \Pr(X_t \geq q) dq$ and 
$\int_{-L}^{0} \Pr(X_t \leq q) dq$ respectively.
Therefore, 
by \Cref{lem:Riemann},
the approximation error of the Riemann Sum is at most $O(L/m) = O( Q_{F,\eps}/\sqrt{T} )$.
On the other hand, 
$\int_{L}^{\infty} \Pr(X_t \geq q) dq, \int_{L}^{\infty} \Pr(X_t \leq q) dq$ is at most $\int_{L}^{\infty} F(q) dq$ by the tail bound of $X_t$, which is at most $Q_{F,\eps}/\sqrt{T}$ by our definition of $L$.
Therefore, we must have
$\snorm{2}{ \mu^* - \hat \mu }
\leq 
O(1) \cdot
\sqrt{ T \cdot \lp( Q_{F,\eps}/\sqrt{T} \rp)^2}
\leq O(Q_{F,\eps})$.
Then, it suffices to bound $\snorm{2}{ \hat \mu - \mu }$. In particular, we have
\begin{align*}
    \snorm{2}{ \hat \mu - \mu }
&=
\snorm{2}{  \sum_{i=1}^m 
\lp( \tilde Y(q_{i}) - \E \lp[ Y(q_{i}) \rp] \rp)
\cdot L/m
- 
 \sum_{i=1}^m 
\lp( \tilde Y(-q_i) - \E \lp[ Y(-q_i) \rp] \rp)
\cdot L/m
} \\
&\leq
  \sum_{i=1}^m 
  \snorm{2}{
\lp( \tilde Y(q_{i}) - \E \lp[ Y(q_{i}) \rp] \rp)}
\cdot L/m
+
 \sum_{i=1}^m 
 \snorm{2}{
\lp( \tilde Y(-q_i) - \E \lp[ Y(-q_i) \rp] \rp)
}
\cdot L/m
\end{align*}
We will focus only on the first term since the bound for the other term is similar. 
By \Cref{lem:binary-product-estimation}~, as long as the number of samples is at least 
$$
n \geq 2^T 
\poly\lp(T, \eps, 
1/F(L) \rp) \cdot \log(m/\tau) $$ 
where $m = \floor{\sqrt{T} \cdot L / Q_{F, \eps}}$, the estimator $\tilde Y(q_i)$ satisfies the condition 
$\snorm{2}{ \tilde Y(q_i) - Y(q_i) } \leq \min \lp( \eps, \sqrt{\eps \cdot F(q_i)} \rp)$ with probability at least $1 - \tau / (10m)$. 
By union bound, the condition holds for all $\tilde Y(q_i)$ with high probability.
This then gives
\begin{align*}
\sum_{i=1}^{m}
  \snorm{2}{
\lp( \tilde Y(q_{i}) - Y(q_{i}) \rp)} \cdot L/m
&\leq
\sum_{i=1}^{m}
\min \lp( \eps, \sqrt{ \eps \cdot F(q_i) } \rp)
\cdot L/m \\
&\leq 
\int_{q=0}^{\infty}
\min \lp( \eps, \sqrt{ \eps \cdot F(q) } \rp) dq 
= Q_{F, \eps}
\, ,
\end{align*}
where in the second inequality 
we view the sum as the Right Riemann Sum of $\min(\eps, \sqrt{\eps \cdot F(q)})$, which is a monotonically decreasing function of $q$. Therefore, by triangle's inequality, the total error of the algorithm is at most $\snorm{2}{ \mu^* - \mu }
\leq \snorm{2}{ \mu^* - \hat \mu } + \snorm{2}{ \hat \mu - \mu } \leq O(Q_{F, \eps})$.

\end{proof}

As corollaries, we obtain algorithms for estimating 
the mean of many important families of product distributions with optimal accuracy. 
\begin{corollary} \label{cor:bounded-k}
Let $\eps, \tau \in (0,1)$.
Let $X_1, \cdots, X_T$ be distributions satisfying that 
$ \E \lp[  \lp( X_t - \E[ X_t ] \rp)^k \rp] \leq 1$ for some constant integer $k \geq 4$.
Suppose $X$ is the product of the distributions $X_t$.
Then, for sufficiently small $\epsilon$, there exists an algorithm which robustly estimates the mean of $X$ under $\eps$ corruption in the online setting with accuracy 
$ O \lp( \eps^{1-1/k} \rp) $, failure probability $\tau$  and sample complexity 
$$
n \geq 2^T \cdot \poly(T,1/\eps) \cdot \log(1/\tau) \;.
$$
\end{corollary}

\begin{proof}
Similar to the proof of \Cref{thm:optimal-gaussian}, 
we can without loss of generality assume that $\E[X_t] = \sqrt{\eps}$. 
In particular we can always reserve $\poly(1/\eps) \cdot \log(T/\tau)$ 
many samples and use a $1$d robust estimator to estimate $\E[X_t]$ 
up to error $O(\eps^{1-1/k})$, which is bounded above by $\sqrt{\eps}$ 
for sufficiently small $\eps$. Then, we can then use the estimation 
to calibrate the mean and reduce the task into the scenario, 
where $\E \lp[ X_t \rp] \leq \sqrt{\eps}$ for all $t \in [T]$.

By Chebyshev's Inequality (generalized for higher moments), 
it holds that
$ \Pr[ \abs{X_t - \E[X_t] } \geq q ] \leq q^{-k} $. This implies that
$ \Pr[ \abs{X_t } \geq q ] \leq \lp(q - \sqrt{\eps} \rp)^{-k}$ for $q > 1 + \sqrt{\eps}$.
Hence, $X$ is an $F$-tail product distribution with 
$$
F(q) = 
\begin{cases}
&1 \text{ when } q < 1 + \sqrt{\eps} \, , \\
& \lp( q - \sqrt{\eps} \rp)^{-k} \text{ when } q \geq 1 + \sqrt{\eps}.
\end{cases}
$$
Then, the quantity
$Q_{F, \eps}$ is convergence whenever $k \geq 4$. In particular, we have
\begin{align*}
Q_{F, \eps}
&\leq 
\eps \cdot \eps^{-1/k}
+
\sqrt{\eps}  \cdot 
\int_{ \eps^{-1/k} }^{\infty}
\lp( q - \sqrt{\eps} \rp)^{-k} dq \\
&\leq 
\eps \cdot \eps^{-1/k}
+
\sqrt{\eps}  \cdot 
\int_{ \eps^{-1/k} }^{\infty}
\lp( q  / 2 \rp)^{-k} dq \\
&\leq 
\eps \cdot \eps^{-1/k}
+
O \lp( \sqrt{\eps} \rp)  \cdot 
\int_{ \eps^{-1/k} }^{\infty}
\lp( q  \rp)^{-k} dq
= O \lp(  \eps^{1-1/k} \rp).
\end{align*}
Then, by \Cref{thm:f-tail}, the accuracy of the meta-algorithm 
is then given by $O \lp(  \eps^{1-1/k} \rp)$.
Then, the quantity $L$ is given by
\begin{align*}
    L = \inf_{z} \lp(  \int_z^{\infty} F(q) dq \leq \frac{1}{\sqrt{T}} Q_{F, \eps} \rp)
    \leq \inf_{z} \lp( \frac{1}{k-1} z^{1-k} \leq \frac{1}{\sqrt{T}}  \eps^{1 - 1/k} \rp)
    \leq  \poly(T, 1/\eps).
\end{align*}
Accordingly, we have $1/F(L) \leq L^k = \poly(T, 1/\eps)$ for constant $k$.
Hence, the sample complexity is given by
$2^T \cdot \poly(T, 1/\eps) \cdot \log(1/\tau)$.
\end{proof}

\begin{corollary} \label{cor:subgaussian}
Let $\eps, \tau \in (0,1)$.
Let $X_1, \cdots, X_T$ be sub-gaussian distributions with unit variance. 
Suppose $X$ is the product of the distributions $X_t$.
Then, for sufficiently small $\epsilon$, there exists an algorithm which robustly estimates the mean of $X$ under $\eps$ corruption in the online setting with accuracy 
$O\lp( \eps \cdot \sqrt{ \log(1/\eps)  }\rp)
$, failure probability $\tau$  and sample complexity 
$$
n \geq 2^T \cdot \poly(T, 1/\eps) \cdot \log(1/\tau).
$$
\end{corollary}
\begin{proof}
Again, similar to the proof of \Cref{thm:optimal-gaussian}, 
we can without loss of generality assume that $\E[X_t] = \sqrt{\eps}$. 
In particular we can always reserve $\poly(1/\eps) \cdot \log(T/\tau)$ many samples and use a $1$d robust estimator to estimate $\E[X_t]$ up to error $O(\eps \sqrt{ \log(1/\eps) })$, which is bounded above by $\sqrt{\eps}$ for sufficiently small $\eps$. Then, we can then use the estimation to calibrate the mean and reduce the task into the scenario where $\E \lp[ X_t \rp] \leq \sqrt{\eps}$ for all $t \in [T]$.

Since each $X_t$ is a sub-gaussian distribution, by definition,
we have that 
$\Pr[ \abs{X_t - \E[X_t]} > q  ] \leq \exp(-q^2/2)$ for $q \in \R^+$.
Since $\E[X_t] \leq \sqrt{\eps}$, this implies that
$\Pr[ \abs{X_t} > q  ]
\leq \exp \lp( - ( q -\sqrt{\eps} )^2/2 \rp)
$ for $ q > \sqrt{\eps}$.
Hence, $X$ is an $F$-tail product distribution with 
$$
F(q) = 
\begin{cases}
&1 \text{ when } q < \sqrt{\eps} \, ,\\
&\exp \lp( - ( q -\sqrt{\eps} )^2 \rp)  \text{ when } q \geq \sqrt{\eps}.
\end{cases}
$$
In fact, when $q$ is at least $\sqrt{\log(1/\eps)}$, we will further have 
$F(q) \leq \exp(-q^2)/2$.
Then, the quantity
$Q_{F, \eps}$ is convergent. In particular, we have
\begin{align*}
Q_{F, \eps}
&\leq \eps \cdot \sqrt{ \log(1/\eps) }
+ O \lp( \sqrt{\eps} \rp) \int_{  \sqrt{ \log(1/\eps) } }^{\infty}   \exp(-q^2/2) dq \\
& \leq \eps \cdot \sqrt{ \log(1/\eps) } +  O( \sqrt{\eps} ) \cdot \exp(  -\log(1/\eps)/2
)
= O \lp(  \eps \cdot \sqrt{ \log(1/\eps) } \rp).
\end{align*}
Then, by \Cref{thm:f-tail}, the accuracy of the meta-algorithm is $O \lp(  \eps \cdot \sqrt{ \log(1/\eps) } \rp)$.
Then, the quantity $L$ is given by
\begin{align*}
    L \eqdef \inf_{z} \lp(  \int_z^{\infty} F(q) dq \leq \frac{1}{\sqrt{T}} Q_{F, \eps} \rp)
    \leq \inf_{z} \lp( \int_z^{\infty} \exp(-q^2/2) dq \leq \frac{1}{\sqrt{T}} \eps \cdot \sqrt{\log(1/\eps)} \rp)
    \leq  O\lp(\log(T/\eps)\rp) \, , 
\end{align*}
which implies that
$$
\frac{1}{F(L)}
\leq { \exp(O(1) \cdot \log(T/\eps) ) } 
\leq \poly(T, 1/\eps).
$$
Hence, the sample complexity is given by
$$
2^T \cdot \poly(T, 1/\eps, 1/F(L)) \cdot \log(L/(Q_{F,\eps}\cdot\tau))
\leq 2^T \cdot \poly(T, 1/\eps) \cdot \log(1/\tau).
$$
\end{proof}

\bibliographystyle{alpha}
\bibliography{allrefs.bib}

\newcommand{\etalchar}[1]{$^{#1}$}
\begin{thebibliography}{MMR{\etalchar{+}}17}

\bibitem[ABM19]{altschuler2019best}
J.~M. Altschuler, V.-E. Brunel, and A.~Malek.
\newblock Best arm identification for contaminated bandits.
\newblock {\em J. Mach. Learn. Res.}, 20(91):1--39, 2019.

\bibitem[Ans60]{Ans60}
F.~J. Anscombe.
\newblock {Rejection of outliers}.
\newblock {\em Technometrics}, 2(2):123 -- 147, 1960.

\bibitem[BDH{\etalchar{+}}20]{BakshiDHKKK20}
A.~Bakshi, I.~Diakonikolas, S.~B. Hopkins, D.~Kane, S.~Karmalkar, and P.~K.
  Kothari.
\newblock Outlier-robust clustering of gaussians and other non-spherical
  mixtures.
\newblock In {\em 61st {IEEE} Annual Symposium on Foundations of Computer
  Science, {FOCS} 2020}, pages 149--159, 2020.

\bibitem[BDJ{\etalchar{+}}22]{BD+20-gmm}
A.~Bakshi, I.~Diakonikolas, H.~Jia, D.M. Kane, P.~Kothari, and S.~Vempala.
\newblock Robustly learning mixtures of \emph{k} arbitrary gaussians.
\newblock In {\em {STOC} '22: 54th Annual {ACM} {SIGACT} Symposium on Theory of
  Computing}, pages 1234--1247, 2022.
\newblock Full version available at https://arxiv.org/abs/2012.02119.

\bibitem[Ber06]{Bernholt}
T.~Bernholt.
\newblock Robust estimators are hard to compute.
\newblock Technical report, University of Dortmund, Germany, 2006.

\bibitem[BK20]{BK20}
A.~Bakshi and P.~Kothari.
\newblock Outlier-robust clustering of non-spherical mixtures.
\newblock {\em CoRR}, abs/2005.02970, 2020.

\bibitem[BLKS21]{bogunovic2021stochastic}
I.~Bogunovic, A.~Losalka, A.~Krause, and J.~Scarlett.
\newblock Stochastic linear bandits robust to adversarial attacks.
\newblock In {\em International Conference on Artificial Intelligence and
  Statistics}, pages 991--999. PMLR, 2021.

\bibitem[BMGS17]{blanchard2017machine}
P.~Blanchard, E.~M.~El Mhamdi, R.~Guerraoui, and J.~Stainer.
\newblock Machine learning with adversaries: Byzantine tolerant gradient
  descent.
\newblock {\em Advances in Neural Information Processing Systems}, 30, 2017.

\bibitem[BP21]{BakshiP21}
A.~Bakshi and A.~Prasad.
\newblock Robust linear regression: optimal rates in polynomial time.
\newblock In {\em {STOC} '21: 53rd Annual {ACM} {SIGACT} Symposium on Theory of
  Computing}, pages 102--115. {ACM}, 2021.

\bibitem[CKMY22]{chen2022online}
S.~Chen, F.~Koehler, A.~Moitra, and M.~Yau.
\newblock Online and distribution-free robustness: Regression and contextual
  bandits with huber contamination.
\newblock In {\em 2021 IEEE 62nd Annual Symposium on Foundations of Computer
  Science (FOCS)}, pages 684--695. IEEE, 2022.

\bibitem[CSX17]{chen2017distributed}
Y.~Chen, L.~Su, and J.~Xu.
\newblock Distributed statistical machine learning in adversarial settings:
  Byzantine gradient descent.
\newblock {\em Proceedings of the ACM on Measurement and Analysis of Computing
  Systems}, 1(2):1--25, 2017.

\bibitem[DG92]{Donoho92}
D.~L. Donoho and M.~Gasko.
\newblock Breakdown properties of location estimates based on halfspace depth
  and projected outlyingness.
\newblock {\em Ann. Statist.}, 20(4):1803--1827, 12 1992.

\bibitem[DHKK20]{DHKK20}
I.~Diakonikolas, S.~B. Hopkins, D.~Kane, and S.~Karmalkar.
\newblock Robustly learning any clusterable mixture of gaussians.
\newblock {\em CoRR}, abs/2005.06417, 2020.

\bibitem[DK21]{DK19-survey}
I.~Diakonikolas and D.~M. Kane.
\newblock Robust high-dimensional statistics.
\newblock In T.~Roughgarden, editor, {\em Beyond the Worst-Case Analysis of
  Algorithms}, chapter~17, pages 382--402. Cambridge University Press, 2021.
\newblock An extended version appeared at http://arxiv.org/abs/1911.05911 under
  the title ``Recent Advances in Algorithmic High-Dimensional Robust
  Statistics''.

\bibitem[DK23]{DK23-book}
I.~Diakonikolas and D.~Kane.
\newblock {\em Algorithmic High-Dimensional Robust Statistics}.
\newblock {Cambridge University Press}, 2023.
\newblock Available at https://sites.google.com/view/ars-book/.

\bibitem[DKK{\etalchar{+}}16]{DKKLMS16}
I.~Diakonikolas, G.~Kamath, D.~M. Kane, J.~Li, A.~Moitra, and A.~Stewart.
\newblock Robust estimators in high dimensions without the computational
  intractability.
\newblock In {\em Proceedings of FOCS'16}, pages 655--664, 2016.
\newblock Journal version in \emph{SIAM Journal on Computing}, 48(2), pages
  742-864, 2019.

\bibitem[DKK{\etalchar{+}}19]{DKK+19-sever}
I.~Diakonikolas, G.~Kamath, D.~Kane, J.~Li, J.~Steinhardt, and A.~Stewart.
\newblock Sever: {A} robust meta-algorithm for stochastic optimization.
\newblock In {\em Proceedings of the 36th International Conference on Machine
  Learning, {ICML} 2019}, pages 1596--1606, 2019.

\bibitem[DKK{\etalchar{+}}22]{DKKLT21}
I.~Diakonikolas, D.~M. Kane, D.~Kongsgaard, J.~Li, and K.~Tian.
\newblock Clustering mixture models in almost-linear time via list-decodable
  mean estimation.
\newblock In {\em {STOC} '22: 54th Annual {ACM} {SIGACT} Symposium on Theory of
  Computing, 2022}, pages 1262--1275, 2022.
\newblock Full version available at https://arxiv.org/abs/2106.08537.

\bibitem[DKS18]{DKS18-list}
I.~Diakonikolas, D.~M. Kane, and A.~Stewart.
\newblock List-decodable robust mean estimation and learning mixtures of
  spherical gaussians.
\newblock In {\em Proceedings of the 50th Annual {ACM} {SIGACT} Symposium on
  Theory of Computing, {STOC} 2018}, pages 1047--1060, 2018.
\newblock Full version available at https://arxiv.org/abs/1711.07211.

\bibitem[DKS19]{DKS19}
I.~Diakonikolas, W.~Kong, and A.~Stewart.
\newblock Efficient algorithms and lower bounds for robust linear regression.
\newblock In {\em Proceedings of the Thirtieth Annual {ACM-SIAM} Symposium on
  Discrete Algorithms, {SODA} 2019}, pages 2745--2754, 2019.

\bibitem[DL88]{DonLiu88a}
D.~L. Donoho and R.~C. Liu.
\newblock The "{{Automatic}}" {{Robustness}} of {{Minimum Distance
  Functionals}}.
\newblock {\em The Annals of Statistics}, 16(2):552--586, 1988.

\bibitem[GKT19]{gupta2019better}
A.~Gupta, T.~Koren, and K.~Talwar.
\newblock Better algorithms for stochastic bandits with adversarial
  corruptions.
\newblock In {\em Conference on Learning Theory}, pages 1562--1578. PMLR, 2019.

\bibitem[HL18]{HL18-sos}
S.~B. Hopkins and J.~Li.
\newblock Mixture models, robustness, and sum of squares proofs.
\newblock In {\em Proceedings of the 50th Annual {ACM} {SIGACT} Symposium on
  Theory of Computing, {STOC} 2018}, pages 1021--1034, 2018.

\bibitem[HR09]{Huber09}
P.~J. Huber and E.~M. Ronchetti.
\newblock {\em {Robust statistics}}.
\newblock Wiley New York, 2009.

\bibitem[HRRS86]{HampelEtalBook86}
F.~R. Hampel, E.~M. Ronchetti, P.~J. Rousseeuw, and W.~A. Stahel.
\newblock {\em Robust statistics. The approach based on influence functions}.
\newblock Wiley New York, 1986.

\bibitem[Hub64]{Huber64}
P.~J. Huber.
\newblock Robust estimation of a location parameter.
\newblock {\em Ann. Math. Statist.}, 35(1):73--101, 03 1964.

\bibitem[KKM18]{KlivansKM18}
A.~R. Klivans, P.~K. Kothari, and R.~Meka.
\newblock Efficient algorithms for outlier-robust regression.
\newblock In {\em Conference On Learning Theory, {COLT} 2018}, pages
  1420--1430, 2018.

\bibitem[KMA{\etalchar{+}}21]{kairouz2021advances}
P.~Kairouz, H.~B. McMahan, B.~Avent, A.~Bellet, M.~Bennis, A.~N. Bhagoji,
  K.~Bonawitz, Z.~Charles, G.~Cormode, R.~Cummings, et~al.
\newblock Advances and open problems in federated learning.
\newblock {\em Foundations and Trends{\textregistered} in Machine Learning},
  14(1--2):1--210, 2021.

\bibitem[KPK19]{kapoor2019corruption}
S.~Kapoor, K.~Patel, and P.~Kar.
\newblock Corruption-tolerant bandit learning.
\newblock {\em Machine Learning}, 108(4):687--715, 2019.

\bibitem[KSS18]{KSS18-sos}
P.~K. Kothari, J.~Steinhardt, and D.~Steurer.
\newblock Robust moment estimation and improved clustering via sum of squares.
\newblock In {\em Proceedings of the 50th Annual {ACM} {SIGACT} Symposium on
  Theory of Computing, {STOC} 2018}, pages 1035--1046, 2018.

\bibitem[LM21]{LM20-gmm}
A.~Liu and A.~Moitra.
\newblock Settling the robust learnability of mixtures of gaussians.
\newblock In {\em {STOC} '21: 53rd Annual {ACM} {SIGACT} Symposium on Theory of
  Computing}, pages 518--531. {ACM}, 2021.
\newblock Full version available at https://arxiv.org/abs/2011.03622.

\bibitem[LRV16]{LaiRV16}
K.~A. Lai, A.~B. Rao, and S.~Vempala.
\newblock Agnostic estimation of mean and covariance.
\newblock In {\em Proceedings of FOCS'16}, 2016.

\bibitem[LSP19]{lamport2019byzantine}
L.~Lamport, R.~Shostak, and M.~Pease.
\newblock The byzantine generals problem.
\newblock In {\em Concurrency: the works of leslie lamport}, pages 203--226.
  Association for Computing Machinery, 2019.

\bibitem[LXC{\etalchar{+}}19]{li2019rsa}
L.~Li, W.~Xu, T.~Chen, G.~B. Giannakis, and Q.~Ling.
\newblock Rsa: Byzantine-robust stochastic aggregation methods for distributed
  learning from heterogeneous datasets.
\newblock In {\em Proceedings of the AAAI Conference on Artificial
  Intelligence}, volume~33, pages 1544--1551, 2019.

\bibitem[MMR{\etalchar{+}}17]{mcmahan2017communication}
B.~McMahan, E.~Moore, D.~Ramage, S.~Hampson, and B.~A. y~Arcas.
\newblock Communication-efficient learning of deep networks from decentralized
  data.
\newblock In {\em Artificial intelligence and statistics}, pages 1273--1282.
  PMLR, 2017.

\bibitem[MTCD21]{mukherjee2021mean}
A.~Mukherjee, A.~Tajer, P.-Y. Chen, and P.~Das.
\newblock Mean-based best arm identification in stochastic bandits under reward
  contamination.
\newblock {\em Advances in Neural Information Processing Systems},
  34:9651--9662, 2021.

\bibitem[PKH19]{pillutla2019robust}
K.~Pillutla, S.~M. Kakade, and Z.~Harchaoui.
\newblock Robust aggregation for federated learning.
\newblock {\em arXiv preprint arXiv:1912.13445}, 2019.

\bibitem[PSBR20]{PSBR18}
A.~Prasad, A.~S. Suggala, S.~Balakrishnan, and P.~Ravikumar.
\newblock Robust estimation via robust gradient estimation.
\newblock {\em Journal of the Royal Statistical Society: Series B (Statistical
  Methodology)}, 82(3):601--627, 2020.
\newblock Also available at http://arxiv.org/abs/1802.06485.

\bibitem[SX19]{su2019securing}
L.~Su and J.~Xu.
\newblock Securing distributed gradient descent in high dimensional statistical
  learning.
\newblock {\em Proceedings of the ACM on Measurement and Analysis of Computing
  Systems}, 3(1):1--41, 2019.

\bibitem[TLL18]{lykouris2018stochastic}
V.~T.~Lykouris, V.~Mirrokni and R.~Paes Leme.
\newblock Stochastic bandits robust to adversarial corruptions.
\newblock In {\em Proceedings of the 50th Annual ACM SIGACT Symposium on Theory
  of Computing}, pages 114--122, 2018.

\bibitem[Tuk60]{Tukey60}
J.~W. Tukey.
\newblock A survey of sampling from contaminated distributions.
\newblock In {\em Contributions to probability and statistics: Essays in Honor
  of {Harold Hotelling}}, pages 448--485. Princeton, New Jersey: Princeton
  University, 1960.

\bibitem[Tuk75]{Tukey75}
J.W. Tukey.
\newblock Mathematics and picturing of data.
\newblock In {\em Proceedings of ICM}, volume~6, pages 523--531, 1975.

\bibitem[XCCL21]{xie2021crfl}
C.~Xie, M.~Chen, P.~Chen, and B.~Li.
\newblock Crfl: Certifiably robust federated learning against backdoor attacks.
\newblock In {\em International Conference on Machine Learning}, pages
  11372--11382. PMLR, 2021.

\bibitem[Yat85]{Yatracos85}
Y.~G. Yatracos.
\newblock {Rates of convergence of minimum distance estimators and Kolmogorov's
  entropy}.
\newblock {\em Annals of Statistics}, 13:768--774, 1985.

\bibitem[YS22]{yao2022robust}
T.~Yao and S.~Sundaram.
\newblock Robust online and distributed mean estimation under adversarial data
  corruption.
\newblock {\em arXiv preprint arXiv:2209.09624}, 2022.

\end{thebibliography}
\newpage
\appendix

\section*{Appendix}

\section{Block Model Extension} \label{sec:block}
A significant limitation of the method discussed in Section~\ref{sec:gen-product} is that it requires every coordinate of the unknown distribution to be independent of the others. 
In this section, we relax the constrain by allowing coordinates that are revealed in the same round to have correlations. 
In other words, we only require the coordinates to be \emph{round-wise} independent.
Formally, we consider the following definition of an $F$-tail bound block distribution, which can be viewed as a generalization of \Cref{def:f-tail-bound}.
\begin{definition}[$F$-tail bound $(T \times d)$-block distribution]
Let $F$ be some monotonically decreasing function $F: \R^+ \mapsto [0,1]$.
We say $X$ is an $F$-tail bound $(T \times d)$-block distribution if $X$ is the product of  $T$ distributions $X_1, \cdots, X_T$ where each $X_t$
is a $d$-dimensional distribution satisfying that
$
\Pr \lp[ \snorm{2}{ X_t  }  \geq q\rp] \leq F(q).
$
\end{definition}
Similar to the requirement of estimating means of product distributions, we require the quantity
$$
Q_{F,\eps} = \int_{0}^{\infty} \min\lp(\eps, \sqrt{ \eps F(q) } \rp) dq \;.
$$
to be convergent.
Our main result in the section is an algorithm which can robustly estimate the mean for such block-wise independent distributions in the online setting, which can be viewed as a generalization of \Cref{thm:f-tail}.
\begin{theorem} \label{thm:block-extension}
Let $\eps, \tau \in (0,1)$, $F$ be some monotonically decreasing function $F: \R^+ \mapsto [0,1]$ such that
$Q_{F,\eps} \eqdef \int_{0}^{\infty} \min\lp(\eps, \sqrt{ \eps F(q) } \rp) dq$ is convergent.
Suppose $X$ is an $F$-tail bound $(T \times d)$-block distribution.
Then, for sufficiently small $\epsilon$, there exists an algorithm which estimates the mean of $X$ under $\eps$ corruption in the $T$-round online setting with error $ O\lp( Q_{F, \eps}  \rp)$, failure probability $\tau$  and sample complexity 
$$
n \geq 
2^{O(Td^2)} \cdot \poly(1/\eps, 1/F(L)) \cdot \log(L/(Q_{F, \eps} \cdot\tau))\, ,
$$
where $L \eqdef \inf_{z} \int_z^{\infty} F(q) dq \leq \frac{1}{\sqrt{T}} Q_{F, \eps}$.
\end{theorem}
Recall that when $X$ has pair-wise independent coordinates, we reduce the problem into the case that the algorithm only outputs $1$ coordinate in each round by manually simulating the process of revealing $1$ coordinate at a time.
Then, we further reduce the problem into estimating binary product distribution in the online setting. 
Naturally, one would wonder whether the trick can be applied here directly. 
Unfortunately, if one goes through exactly the same reduction procedure, the task will be reduced into estimating means of correlated binary distributions. Noticeably, it is information theoretically impossible to achieve dimension-independent error (independent of $T$) even in the offline setting for this task.

To circumvent the issue, we will estimate the $d$-dimensional mean vector $\mu^*_{t}$ corresponding to the $d$ coordinates revealed at the $t$-th round at once.
The high-level idea is to estimate $\mu^*_{t}$ projected along many different directions and then summarize the information together into an estimation with a Linear Program.  
We begin by drawing $2^{O(d)} \cdot \log (T)$  many unit vectors in $\R^d$ uniformly at random. 
Denote the set of random unit vectors as $\mathcal V$.
At the $t$-th round, we try to estimate 
$v^T \mu^*_{t} $ for each $v \in \mathcal V$. 
Denote the estimation result as 
$ \mu(v)_t $. 
Then, our final estimate $\mu_t$ is  the solution to the program
\begin{align} \label{eq:lp}
\min_{\mu_t \in \R^d} \max_{v \in \mathcal V} \abs{ v^T \mu_t - \mu(v)_t }.    
\end{align}
Formally, the solution to the program will give us the following guarantee.
\begin{claim} \label{clm:lp-error}
Suppose $\mathcal V$ is a set of random unit vectors of size at least $2^{O(d)} \cdot \log(T / \tau)$
For all $t \in [T]$, it holds
$\snorm{2}{ \mu_t - \mu_t^* }
\leq O(1) \cdot \max_{v \in \mathcal V} \abs{v^T \mu_t^* - \mu(v)_t }.
$ with probability at least $ 1 - \tau $.
\end{claim}
\begin{proof}
For a fixed round $t$, the following happens with probability at least $1 - \tau/(10T)$.
\begin{align} \label{eq:random-set-property}
\sup_{ v \in \R^d } v^T \cdot \lp( \mu_{t} - \mu^*_{t} \rp)   
\leq O(1) \cdot
 \max_{ v \in \mathcal V } v^T \cdot \lp( \mu_{t} - \mu^*_{t} \rp).
\end{align}
By union bound, this holds for all rounds with probability at least $1 - \tau$. In the following analysis, we condition on the inequality holds for all $t$. 

By the definition of the program, we claim that $\mu_{t}$, the solution to the program, must satisfy
\begin{align} \label{eq:lp-error}
\max_{v \in \mathcal V} \abs{v\cdot \lp( \mu_{t} - \mu^*_{t} \rp)} 
\leq 2 \cdot 
\max_{v \in \mathcal V} \abs{  \mu_{t}( v ) - v^T \mu^*_{t} }.
\end{align}
This is true since, by triangle's inequality, we can write 
\begin{align*}
\max_{v \in \mathcal V} \abs{v^T \cdot \lp( \mu_{t} - \mu^*_{t} \rp)} 
\leq
\max_{v \in \mathcal V} \abs{v^T \cdot \mu_{t} -   \mu(v)_{t} }
+
\max_{v \in \mathcal V} \abs{v^T \cdot \mu^*_{t}-  \mu(v)_{t} }.
\end{align*}
Notice that $\max_{v \in \mathcal V} \abs{v^T \cdot \mu_{t} -   \mu(v)_{t} }$ is at most $\max_{v \in \mathcal V} \abs{v^T \cdot \mu^*_{t}-  \mu(v)_{t} }$ since $\mu_t$ the vector which minimizes the expression. Equation~\eqref{eq:lp-error} then follows. Our claim then follows from Equations~\eqref{eq:random-set-property} and ~\eqref{eq:lp-error}.
\end{proof}

We have then reduced the task into
computing $\mu(v)_t$ - estimator for 
the mean of $v^T X_t$. 
Since $X_t$ and $X_t'$ are independent for $t \neq t'$, $v^T X_1, \cdots v^T X_T$ therefore forms an $F$-tail bound product distribution.
This suggests that the techniques illustrated in the last section can be made of good use.
Using techniques from Section~\ref{ssec:nonpar},  we can simultaneously compute estimators $\mu(v)$ for all $v \in \mathcal V$ satisfying that
\begin{align} \label{eq:max-sum}
\max_{v \in \mathcal V} \sum_{t=1}^T \lp( \mu(v)_t - v^T \mu^*_t \rp)^2 \leq O(1) \cdot \int_{0}^{\infty} \min\lp(\eps, \sqrt{ \eps F(q) } \rp) dq.   
\end{align}
However, this turns out to be insufficient. 
As stated in Claim~\ref{clm:lp-error}, for the final output $\mu_t$ to be closed to $\mu_t^*$, we need
$\max_{v \in \mathcal V} \lp( v^T \mu_t^* - \mu(v)_t \rp)^2$ to be small. 
In other words, what we really need is 
\begin{align} \label{eq:sum-max}
\sum_{t=1}^T \max_{v \in \mathcal V} \lp( v^T \mu_t^* - \mu(v)_t \rp)^2 \leq O(1) \cdot \int_{0}^{\infty} \min\lp(\eps, \sqrt{ \eps F(q) } \rp) dq.    
\end{align}
It is not hard to see the left hand side of Equation~\eqref{eq:sum-max} can be much larger than that of Equation~\eqref{eq:max-sum}, making the guarantees obtained by applying Algorithm~\ref{alg:meta-alg} insufficient in a blackbox manner. Fortunately, it is possible to modify the algorithm such that Equation~\eqref{eq:sum-max} is satisfied. 
\begin{lemma}
\label{lem:correlated-product}
Let $\eps,\tau \in (0, 1)$, $F$ be some monotonically decreasing function $F: \R^+ \mapsto [0,1]$ such that
$Q_{F,\eps} \eqdef \int_{0}^{\infty} \min\lp(\eps, \sqrt{ \eps F(q) } \rp) dq$ is convergent.
Suppose $X$ is an $F$-tail bound $(T \times d)$-block distribution with unknown mean vector $\mu^* \in \R^{Td}$.
Let $\mathcal V$ be a set of unit vectors in $\R^d$.
Suppose the number of samples is at least
$$
n \geq 
\abs{\mathcal V}^{T(d+1)} \cdot \poly(1/\eps, 1/F(L)) \cdot \log(L/(Q_{F, \eps}\cdot\tau))
$$
where $L \eqdef  \inf_{z} \lp( \int_z^{\infty} F(q) dq \leq \frac{1}{\sqrt{T}} Q_{F, \eps} \rp)$.
Then, for sufficiently small $\eps$, there exists an algorithm 
\textbf{Projection-Estimation} (Algorithm~\ref{alg:correlated-product})
which outputs estimators $\mu(v) \in \R^{Td}$ for each $v \in \mathcal V$ in the online setting such that
$$
\sum_{t=1}^T  \max_{v \in \mathcal V} 
\lp( \mu(v)_t - v^T \mu^*_t  \rp)^2
\leq O(1) \cdot \int_{0}^{\infty} \min\lp(\eps, \sqrt{ \eps F(q) } \rp) dq
$$
with probability at least $1 - \tau$.
\end{lemma}
\subsection{Proof of \Cref{lem:correlated-product}}
We follow the same procedure as Section~\ref{ssec:nonpar} to reduce the task of estimating $v^T \mu^*_t$ into estimating binary product distributions. 
In particular, 
we can define the binary variables
\begin{align} \label{eq:binary-variables}
Y( q,v )_t \eqdef
\begin{cases}
&\mathbbm 1 \{ v^T \cdot X_{t} > q\} \text{ for } q>0\, ,\\
&\mathbbm 1 \{ v^T \cdot X_{t} < q\} \, ,
\text{ for } q<0 .
\end{cases}
\end{align}
for $q \in \R$ and $v \in \mathcal V \subseteq \R^d$.
Then, similar to \Cref{thm:f-tail}, 
for a fixed $v \in \mathcal V$,
we can reduce estimating $v^T \mu_t^*$  into estimating 
$ \E\lp[Y(v, q)_t\rp] $ for many appropriately chosen $q$. 
It is easy to see that
$Y(v,q)_1, \cdots, Y(v,q)_T$ form a binary product distribution.
If one runs \textbf{Binary-Product-Estimation} for each pair of $(v,q)$ in parallel, it is easy to compute estimators satisfying that
$$
\max_{v \in \mathcal V} \sum_{t=1}^T \lp(\E [Y(v,q)_t] - \tilde Y(v,q)_t \rp)^2  \leq \min \lp( \eps^2,  \eps F(q) \rp).
$$
While this is enough to achieve the guarantees in Equation~\eqref{eq:max-sum}, for the proof of Lemma~\ref{lem:correlated-product}, it turns out we need the following stronger guarantee.
$$
 \sum_{t=1}^T \max_{v \in \mathcal V}\lp(\E [Y(v,q)_t] - \tilde Y(v,q)_t \rp)^2  \leq \min \lp( \eps^2,  \eps F(q) \rp).
$$
This is made possible with the routine \textbf{Correlated-Binary-Estimation} (Algorithm~\ref{alg:correlated-binary}).
\begin{algorithm}[ht] 
\caption{Correlated-Binary-Estimation}
\label{alg:correlated-binary}
\begin{algorithmic}[1]
\STATE \textbf{Input:} Threshold parameter $q \in R$,
unit vector sets $\mathcal V \in \R^d$,
round number $T$, $n$ samples $x^{(1)}, \cdots, x^{(n)}$ from $X$ such that the coordinate $x^{(i)}_t$ is revealed at the $t$-th round.
\STATE Initialize the group $S_0^{(1)} = \{ x^{(1)}, \cdots, x^{(n)} \}$.
Set $\gamma = F(q)$.
\FOR{$t=1,2,...,T$}
   \STATE In the $t$-th round, $x_t^{(1)}, \cdots, x_t^{(n)}$ are revealed.
   \\ \COMMENT{Convert into samples of $Y(v,q)$} 
   \STATE For all $v \in \mathcal V$, compute
   $
   y(v,q)_t^{(i)} = \mathbbm \{ \abs{v^T x_{t}^{(i)}} > \abs{q} \}.
   $
   \\ \COMMENT{Add noises to $y(v,q)_t^{(i)}$}
   \FOR{$i=1\cdots n$}
   \STATE Sample $u$ uniformly from $(0,1)$.
   \IF{$u \leq \tau/4$}
   \STATE Set $y(v,q)_t^{(i)} = 1$ for all $v \in \mathcal V$,
   \ELSIF{ $u \leq 1/2$}
   \STATE Set $y(v,q)_t^{(i)} = 0$ for all $v \in \mathcal V$,
   \ENDIF
   \ENDFOR
   \\ \COMMENT{Divide groups based on all $y(v,q)_{t'}^{(i)}$}
   \STATE Create the group partition $\lp\{S_1^{(t+1)}, S_2^{(t+1)} \cdots S_{m(t+1)}^{(t+1)} \rp\}$ in the $(t+1)$-th round such that two samples $j, j'$end up in the same group if and only if $y(v,q)_{t'}^{(j)} = y(v,q)_{t'}^{(j')} $ for all $t' \leq t$ and $v \in \mathcal V$.
   \\ \COMMENT{Compute group estimations}
   \FOR{each group $S_i^{(t)}$, $v \in \mathcal V$}
   \STATE Compute the group estimation
   $\mu(v)_{t}^{(i)} \eqdef \min \lp(   \gamma, \frac{1}{ |S_i^{(t)}| }\sum_{j \in S_i^{(t)}  } y(v,q)^{(j)}_{t}  \rp)$.
   \ENDFOR
   \STATE Set $\mu(v, q)_t$ to be the weighted median over $\mu(v, q)_t^{(i)}$ where the weights are given by $\abs{ S_i^{(t)} }$.
    \STATE \textbf{Output}: $\mu(v,q)_t$ for each $v$.
\ENDFOR 
\end{algorithmic}
\end{algorithm}
\begin{lemma}
\label{lem:correlated-binary}
Let $\eps,\tau \in (0, 1)$, $F$ be some monotonically decreasing function $F: \R^+ \mapsto [0,1]$ such that
$Q_{F,\eps} \eqdef \int_{0}^{\infty} \min\lp(\eps, \sqrt{ \eps F(q) } \rp) dq$ is convergent.
Suppose $X$ is an $F$-tail bound $(T \times d)$-block distribution.
Let $\mathcal V$ be a set of unit vectors in $\R^d$.
Suppose the number of samples is at least
$$
n \geq 
\abs{\mathcal V}^{T(d+1)} \cdot \poly(1/\eps, 1/F(L)) \cdot \log(L/(Q_{F, \eps}\cdot\tau))
$$
where $L \eqdef  \inf_{z} \lp( \int_z^{\infty} F(q) dq \leq \frac{1}{\sqrt{T}} Q_{F, \eps} \rp)$.
Fixing $q \in \R$ and let $Y(v,q)$ be defined as in Equation~\eqref{eq:binary-variables} for $v \in \mathcal V$.
Then, for sufficiently small $\eps$, there exists an algorithm which outputs estimators $\tilde Y(v,q)$ for each $v \in \mathcal V$ in the online setting such that
$$
\sum_{t=1}^T  \max_{v \in \mathcal V} 
\lp( \E[Y(v,q)_t] - \tilde Y(v,q)_t  \rp)^2
\leq O(1) \cdot  \min\lp(\eps^2, \eps F(q) \rp)
$$
with probability at least $1 - \tau$.
\end{lemma}
At a high level, we still follow the framework of \textbf{Binary-Product-Estimation}: We will divide the samples into groups based on the coordinates revealed so far and the final estimations will be the (weighted) median of the group estimations. The major difference is that now the group division is based on the labels of multiple binary product distributions. 
We focus on the case $q > 0$ since the argument when $q<0$ is symmetric.
Denote $\gamma = F(q)$.
Since we are now only interested in estimating $Y(v,q)$ for a fixed $q$. 
We next discuss the steps of Algorithm~\ref{alg:correlated-binary} in details.
\paragraph{Sample Conversion} 
At the $t$-th round, the algorithm receives $x^{(i)}_{t} \in \R^{d}$.
We will first convert it into data points for $Y(v,q)_t$. 
For each $v \in \mathcal V$,
we compute the indicators
$
y(v,q)_t^{(i)} = \mathbbm 1\{v^T x^{(i)}_{t} > q\}.
$
Then, $y(v,q)^{(i)} \in \R^{d}$ for  $i \in [n]$  can be viewed as \iid~samples drawn from the distribution $Y(v,q)$.
\paragraph{Label Noise}
We will manually add noise to the indicators $y(v,q)_t^{(i)}$.
For each sample $i \in [n]$, we simultaneously change $y(v,q)_t^{(i)}$ for all $v \in \mathcal V$ to $1$ with probability $\gamma / 4$, to $0$ with probability $1/2 - \gamma/4$, and leaves them unchanged otherwise.
Then, $y(v,q)^{(i)} \in \R^{d}$ for each $i \in [n]$  can be viewed as \iid~samples drawn from the binary product distribution $Y'(v,q)$ satisfying that $\E[Y'(v,q)_t] = \E[Y(v,q)_t]/2 + \gamma/4$.
This allows us in the following analysis to assume that $\E[Y(v,q)_t] \in [\gamma/4, 3 \cdot \gamma/4]$.

\paragraph{Group Division} At the beginning of the $t$-th round, the algorithm divides the samples into many groups based on the values of $y(v,q)_{t'}^{(i)}$ for $t' < t$.
In particular, two samples $i,j$ end in the same group in the $t$-th round if and only if 
$y(v,q)_{t'} ^{(i)} = y(v,q)_{t'}^{(j)}$ for all $v \in \mathcal V$ and $t' < t$. Denote $m(t)$ as the number of groups at the beginning of the $t$-th round. 
Naively, it seems like at the $t$-th round there can be as many as 
$2^{\abs{\mathcal V}}$ groups. A more careful computation shows that $ m(t) \leq \abs{\mathcal V}^{t\cdot(d+1)}$.

\begin{lemma}
At the $t$-th round, there can be at most $m(t) \leq \abs{\mathcal V}^{t\cdot(d+1)}$ many groups.
\end{lemma}
\begin{proof}
We will show that 
$m(t+1) \leq m(t) \cdot \abs{\mathcal V}^{(d+1)} $. 
Then, the argument follows from induction.
At the $t$-th round, consider the half-spaces in $\R^d$ parametrized by the sample points $x^{(i)}_{t}$.
$$
\mathcal H^{(t)} = \{ z^T \cdot x^{(i)}_{t} \geq q | i \in [n] \}.
$$
The indicator $y(v,q)_t^{(i)}$ can essentially be viewed as the classification of the point $v \in \mathcal V$ by the half-space $z^T \cdot x^{(i)}_{t} \geq q$. 
Now, for each sample $i$, we associate it with a set $ \mathcal V^{(i)}_t \subseteq \mathcal V$ that includes all vectors $v \in \mathcal V$ which its corresponds half-space classifies as positive. Namely,
$ \mathcal V^{(i)}_t = \lp\{v \in \mathcal V| v^T \cdot x^{(i)}_{t} \geq q  \rp\} $.
Essentially, two sample points $j, j' \in S_i^{(t)}$ will end up in the same group in the $(t+1)$-th round if and only if $\mathcal V^{(j)}_t = \mathcal V^{(j')}_t$.
It is well known that the VC dimension of half-spaces in $\R^d$ is $d+1$. Then, by Sauer's Lemma, we know there can be at most 
$ \abs{ \mathcal V }^{d+1} $ many distinct subsets $\mathcal V^{(i)}_t$.
Therefore, each group $S_i^{(t)}$ splits into at most $\abs{ \mathcal V }^{(d+1)}$ many child groups at the $(t+1)$-th round. 
\end{proof}

\paragraph{Group estimations and outputs}
Denote $m(t)$ as the number of groups at the At the $t$-th round.
Within each group $i \in [m(t)]$, for each $v \in \mathcal V$, we compute the group estimation  $\mu(v,q)^{(i)}_{t} \eqdef \min \lp(   \gamma, \frac{1}{ |S_i^{(t)}| }\sum_{j \in S_i^{(t)}  } y(v,q)^{(j)}_{t}  \rp)$.
Then, for each $v \in \mathcal V$,  the final output $ \tilde Y(v, q)_t $ is given by the weighted median over all $\mu(v,q)^{(i)}_{t}$, i.e. the median of the distribution $U_v$ such that $\Pr \lp[  U_v = \mu(v,q)_{t}^{(i)} \rp] \propto \abs{ S_i^{(t)} }$.

Similar to the proof of \Cref{lem:binary-product-estimation},  our proof consists of three steps. First, we show that the sample mean of the clean samples within each group is well concentrated around the true mean in each round. Second, we show that there exists a potential function such that it increases significantly whenever the algorithm incurs significant errors. Lastly, we upper bound the potential function and uses that to conclude that the total errors incurs must be bounded.

\begin{lemma} \label{lem:concentration-block}
Fix $q \in \R^+$ and denote $\gamma = F(q)$.
Let $\hat \mu(v,q)^{(i)}_t$ be the empirical mean of the group $S_i^{(t)}$ computed from only the clean samples.
In particular, let $\mathcal C$ denote the set of un-corrupted samples. 
We define $ \hat \mu(v,q)^{(i)}_t \eqdef 
\frac{1}{ \abs{ \mathcal C \cap S_i^{(t)} } }
\sum_{ i \in \mathcal C \cap S_i^{(t)}  } 
y(v,q)_t^{(i)}.
$
Assume that $n \geq \abs{\mathcal V}^{ T \cdot (d+1) } \cdot \poly(1/\eps, 1/\gamma) \cdot \log(1/\tau)$. 
Denote $m(t)$ as the number of groups at the $t$-th round.
With probability at least $1-\tau$, for all $t$ and any group satisfying that 
$ |S_i^{(t)} \cap \mathcal C | \geq n \cdot \eps / m(t)$,  it holds $\abs{\hat \mu(v,q)^{(i)}_t - \E\lp[Y(v,q)_t\rp]  } \leq \min \lp( \eps, \gamma \rp) / T$ for all $v \in \mathcal V$.
\end{lemma}
\begin{proof}
The proof is almost identical to that of \Cref{lem:concentration}. The only difference is that in \Cref{lem:concentration} there are at most $2^{t-1}$ groups in the $t$-th round. Now, there can be as many as $m(t) = \abs{\mathcal V}^{ t \cdot (d+1) }$ groups. Therefore, we need the number of samples to be at least $n \geq \abs{\mathcal V}^{ T \cdot (d+1) } \cdot \poly(1/\eps, 1/\gamma) \cdot \log(1/\tau)$.
\end{proof}
We will use the same potential function as in the proof of \Cref{lem:binary-product-estimation} with $\gamma \eqdef F(q)$. 
In particular, we have
\begin{align*}
\Phi(t) \eqdef \frac{1}{ N } \sum_{i=1}^{m(t)} g_{\gamma}\lp( \eps_i^{(t)} \rp) \cdot \abs{ S_i^{(t)} } \, ,
\end{align*}
where $g_{\delta}: [0,1] \mapsto \R^+$ is the same piecewise function
\begin{equation} \label{eq:piecewise}
g_{\gamma}(x)=
    \begin{cases}
        x^2 & \text{if } x < 10 \cdot \eps / \gamma \, ,\\
        20 \frac{\eps}{\gamma} \cdot x - 100 \lp( \frac{\eps}{\gamma}\rp)^2 & \text{otherwise.}
    \end{cases}
\end{equation}
\begin{lemma} \label{lem:potential-increase-2}
Fix $q \in \R^+$.
Denote $\gamma = F(q)$ and
$
\eta \eqdef 
\max_{v \in \mathcal V} \abs{ \tilde Y(v,q)_t - \E\lp[ Y(v,q)_t \rp]}.
$
Then, we have
$
\Phi(t+1) - \Phi(t) \geq 
\Omega(\eta^2/\gamma)
$ 
as long as 
$\eta \geq 2 \cdot \min \lp( \eps, \gamma \rp)/T$.
\end{lemma}
\begin{proof}
The key idea is the following notion of the ``intermediate potentials'' between two rounds. For that, we need to first define the ``intermediate groups''.
Let 
$v^* \eqdef \argmax_{v \in \mathcal V} \lp( \tilde Y(v,q)_t - \E\lp[ Y(v,q)_t \rp] \rp)^2$. 
Consider the groups obtained by splitting the groups at the beginning of the $t$-th round
solely based on the label of
$y(q^*,v)_t^{(j)}$.
In particular, for a group $S_{i}^{(t)}$, we define the intermediate child groups
$$
S_{2\cdot i}^{(t+1/2)}
= \lp\{ j \in S_{i}^{(t)} \text{ such that } 
y(q^*,v)_t^{(j)} = 0
\rp\} \, ,
S_{2\cdot i + 1}^{(t+1/2)}
= \lp\{ j \in S_{i}^{(t)} \text{ such that } 
y(q^*,v)_t^{(j)} = 1
\rp\}.
$$
We denote $\eps_i^{(t + 1/2)}$ as the corresponding adversarial densities of these intermediate groups.
Then, the intermediate potential is then defined as
\begin{align}
\Phi(t+1/2) \eqdef \frac{1}{ N } \sum_{i=1}^{2 \cdot m(t)} g_{\gamma}\lp( \eps_i^{(t+1/2)} \rp) \cdot \abs{ S_i^{(t+1/2)} } \, ,
\end{align}
where $g_{\gamma}: [0,1] \mapsto \R^+$ is the same piecewise function
used in Equation~\eqref{eq:piecewise}.
Then, we will show 
(i) $\Phi(t+1) \geq \Phi(t+1/2)$, and
(ii)
$\Phi(t+1/2) - \Phi(t) \geq \Omega(\eta^2/\gamma)$ if $\eta \geq 2 \cdot \min\lp( \eps, \gamma\rp)/T$. It is easy to see that combining the two claims then gives our lemma.

We start with claim (i). 
Notice that the groups $S_i^{(t+1)}$ for $i \in [m(t+1)]$ at the beginning of the $(t+1)$-th round can be viewed as the child groups obtained by further splitting the intermediate groups $S_i^{(t+1/2)}$ based on the remaining labels $y(v,q)_t^{(i)}$ for $v \neq v^*$. 
Then, by convexity of $g$, these additional splits will never decrease the potential. 
Hence, the claim follows. 

We then turn to claim (ii). 
Consider the groups in the $t$-th round satisfying the following conditions (i) $\eps_i^{(t)} \leq 5 \eps$ (ii) the estimation $\mu(v, q^*)_t^{(i)}$ is off from $\E\lp[ Y(v, q^*)_t \rp]$ by at least 
$\eta$, which is by our assumption at least $ 2 \cdot \min(\eps, \gamma)/T$, and (iii) the number of clean samples is at least $\abs{ S_i^{(t)} \cap \mathcal C } \geq n \cdot \eps / m(t)$. 
Denote the set of groups satisfying the conditions as $G$.
Then, following almost identical argument as in \Cref{lem:potential-increment-var}, it can be shown that 
$
\frac{1}{n} \sum_{i \in G} \abs{S_i^{(t)}} \geq 1/5.
$
and for all $i \in G$ we have
$$
    \lp( \frac{ |S_{2i}^{(t+1/2)}| }{ n }\rp) \cdot
    g_{\gamma}\lp( \eps_{2i}^{(t+1/2)} \rp)
    +
    \lp( \frac{ |S_{2i+1}^{(t+1/2)}| }{ n } \rp) \cdot 
    g_{\gamma}\lp( \eps_{2i+1}^{(t+1/2)} \rp)
    - \lp(\frac{ \abs{  S_i^{(t)} } }{ n }\rp)  \cdot
    g \lp( \eps_i^{(t)} \rp)
\geq 
\frac{ \abs{S_{i}^{(t)}} }{ n } \cdot \Omega(\eta^2 / \gamma).
$$
The claim then follows.
\end{proof}
Lastly, we note the potential function in the setting shares the same bound as in \Cref{lem:modified-potential-bound}. Its proof is also identical to that of \Cref{lem:modified-potential-bound} since the argument only relies on the fact that $g_{\gamma}$ is convex and the equality $\sum_i \frac{ \abs{ S_i^{(t)} } }{n} \cdot \eps_i^{(t)} = \eps$.
\begin{claim}
\label{lem:modified-potential-bound-2}
$ \Phi(t) \leq O(1) \cdot \min\lp(\eps, \eps^2/\gamma \rp) $ for all $t \in [T]$.
\end{claim}

Now, we can conclude the proof of \Cref{lem:correlated-binary}.
\begin{proof}[Proof of~\Cref{lem:correlated-binary}]
Denote $\gamma = F(q)$.
From \Cref{lem:potential-increase-2} and \Cref{lem:concentration-block}, we know that
\begin{align*}
\sum_{t=1}^T \max_{v \in \mathcal V} \lp( \tilde Y(v,q)_t - \E\lp[ Y(v,q)_t \rp] \rp)^2 
&\leq \sum_{t=1}^T  O \lp( \min\lp(\eps, \gamma\rp)^2/T^2 \rp) +  O(\gamma) \cdot \lp( \Phi(t) - \Phi(t-1) \rp)  \\
&\leq O(\min\lp(\eps, \gamma\rp)^2/T) + O(\gamma) \cdot \Phi(T).    
\end{align*}
By \Cref{lem:modified-potential-bound-2}, we know $\Phi(T) \leq O(1) \cdot \min(\eps, \eps^2/\gamma )$.
Substituting that into the equation above then gives the desired bound on $\sum_{t=1}^T \max_{v \in \mathcal V} \lp( \tilde Y(v,q)_t - \E\lp[ Y(v,q)_t \rp] \rp)^2 $.
\end{proof}
Given these more powerful estimators $\tilde Y(v,q)$, the rest of the step is identical to that of  Algorithm~\ref{alg:meta-alg}. 
We provide the  pseudocode for the algorithm \textbf{Projection-Estimation} in Lemma~\ref{lem:correlated-product} below for completeness.
\begin{algorithm}[h] 
\caption{Projection-Estimation}
\label{alg:correlated-product}
\begin{algorithmic}[1]
\STATE \textbf{Input:} Threshold parameter $q \in R$,
unit vector sets $\mathcal V \in \R^d$,
round number $T$, $n$ samples $x^{(1)}, \cdots, x^{(n)}$ from $X$ such that the coordinate $x^{(i)}_t$ is revealed at the $t$-th round.
\STATE Set $Q_{F,\eps} = \int_{q=0}^{\infty} \min\lp(\eps, \sqrt{ \eps \cdot F(q) } \rp) dq$, $L = \inf_z \lp( \int_{q=z}^{\infty} F(q) dq  \leq Q_{F, \eps} / \sqrt{T} \rp)$,
$m = \floor{ L \cdot \sqrt{T} / Q_{F, \eps} }$.
\STATE Choose $q_0, \cdots, q_m$ such that the points partition $[0, L]$ into intervals of equal size.
\FOR{$i=1,\cdots,m$}
\STATE Initialize a process $\mathcal A_{q_i}$ which runs $\textbf{Correlated-Binary-Estimation}$ with the parameters $q = q_i$ (and respectively for $\mathcal A_{-q_i}$).
\ENDFOR
\FOR{$t=1,2,...,T$}
   \STATE In the $t$-th round, $x_t^{(1)}, \cdots, x_t^{(n)}$ are revealed.
   \FOR{$i=1,\cdots,m$}
   \STATE $\tilde Y(v, q_i) \gets \mathcal A_{q_i}\lp(x^{(1)}_t, \cdots, x^{(n)}_t\rp)$.
   \STATE $\tilde Y(v, -q_i) \gets \mathcal A_{-q_i}\lp(x^{(1)}_t, \cdots, x^{(n)}_t\rp)$.
   \ENDFOR
   \STATE For all $v \in \mathcal V$, compute
   $\mu(v)_t = 
   \sum_{i=1}^m  
   \tilde Y(v, q_{i})
   \cdot L/m
   - 
   \sum_{i=1}^m 
   \tilde Y(v, -q_{i})
   \cdot L/m.
   $
\STATE \textbf{Output}: $\mu(v)_t$ for all $v \in \mathcal V$.
\ENDFOR 
\end{algorithmic}
\end{algorithm}
\begin{proof}[Proof of \Cref{lem:correlated-product}]
At the $t$-th round, the estimator for $v^T \mu_t^*$ is given by
$$
 \mu(v)_t = \sum_{i=1}^m  \tilde Y(v, q_{i})_t
   \cdot L/m
   - \sum_{i=1}^m \tilde Y(v, -q_{i})_t
   \cdot L/m.
$$
For analysis purpose, we will also define the variables $\hat \mu_{t}(  v )$ which are similar to $ \mu_{t}(  v )$ but computed with the exact values of $\E\lp[ Y(v, q_j)_t \rp]$
$$
\hat \mu(v)_t = \sum_{i=1}^m  \E \lp[ Y(v, q_{i})_t \rp]
   \cdot L/m
   - \sum_{i=1}^m  \E \lp[ Y(v, -q_{i})_t \rp]
   \cdot L/m.
$$
Notice that $\hat \mu(v)_t$ can be viewed as the Riemann sum approximation of the integral
$$
v^T \mu_t^*
= 
\int_{ q=0 }^{\infty} \Pr[ v_i^T X_{t} \geq q ] dq
- 
\int_{ -\infty }^{ 0} \Pr[ v_i^T X_{t} \leq q ] dq.
$$
By our assumption of $X$, we have the tail bound $\Pr \lp[  v^T \cdot X_{t} \geq q\rp] \leq F(q)$. 
Hence, by \Cref{lem:Riemann}, it holds
$$
\abs{  v^T \mu_t^* - \hat \mu(v)_t  } \leq O(L/m) = O( Q_{F, \eps}/\sqrt{T} ).
$$
Using the inequality $(a+b)^2 \leq 2 (a^2 + b^2)$, we can then show
\begin{align} \label{eq:fundamental-inequality}
\sqrt{ \sum_{t=1}^T \max_{v \in \mathcal V} 
\lp(  \mu(v)_t -  v^T \mu^*_t \rp)^2
}
&=
\sqrt{ \sum_{t=1}^T \max_{v \in \mathcal V} 
\lp(  \mu(v)_t - \hat \mu(v)_t + \hat \mu(v)_t -  v^T \mu^*_t \rp)^2
} \nonumber \\
&\leq 
\sqrt{2} \cdot 
\sqrt{ \sum_{t=1}^T \max_{v \in \mathcal V} 
\lp( \mu(v)_t - \hat \mu_{t}(v) \rp)^2}
+ O( Q_{F, \eps} ).
\end{align}
It then remains to upper bound the first term.
Utilizing the fact that $ \mu( v )_t$ and $\hat \mu(  v )_t$ are both Riemann sums, we can then write
\begin{align*}
&\sqrt{ \sum_{t=1}^T \max_{v \in \mathcal V} 
\lp( \tilde \mu(v)_t - \hat \mu_{t}(v) \rp)^2
}  \\
&=
\sqrt{ \sum_{t=1}^T 
\max_{v \in \mathcal V}
\lp( 
\sum_{i=1}^m \lp( \E\lp[ Y(v, q_i  )_t \rp] -  \tilde Y(v, q_i  )_t\rp) \cdot L/m
- \sum_{i=1}^m \lp( \E\lp[ Y(v, -q_i  )_t \rp] -  \tilde Y(v, -q_i  )_t \rp) \cdot L/m
\rp)^2
} 
\\
&\leq
\sqrt{ \sum_{t=1}^T 
\max_{v \in \mathcal V}
\lp( 
\sum_{i=1}^m \lp( \E \lp[ Y(v, q_i )_t \rp] - \tilde Y(v, q_i  )_t \rp) \cdot L/m
\rp)^2
} +
\sqrt{ \sum_{t=1}^T 
\max_{v \in \mathcal V}
\lp( 
\sum_{i=1}^m \lp( \E \lp[ Y(v, -q_i )_t \rp] - \tilde Y(v, -q_i  )_t \rp) \cdot L/m
\rp)^2}.
\end{align*}
  We now focus on the first term as the argument for bounding the second term is similar. For convenience, we will denote the vector $v \in \mathcal V$ which maximizes the expression for each round $t$ as~$v^*(t)$.
\begin{align*}
&\sqrt{ \sum_{t=1}^T 
\lp( 
\sum_{i=1}^m \lp( Y(v^*(t), q_i  )_t - \tilde Y(v^*(t), q_i  )_t \rp) \cdot L/m
\rp)^2
} \\
&\leq 
\sum_{i=1}^m
\sqrt{ \sum_{t=1}^T 
\lp( Y(v^*(t), q_i  )_t - \tilde Y(v^*(t), q_i  )_t \rp)^2 \cdot \lp(L/m\rp)^2
} \\ 
&=
\sum_{i=1}^m
\frac{L}{m} \cdot 
\sqrt{ \sum_{t=1}^T
\lp( Y(v^*(t), q_i  )_t - \tilde Y(v^*(t), q_i)_t  \rp)^2
} \\
&\leq \sum_{i=1}^m
\frac{L}{m} \cdot 
\min \lp( \eps, \sqrt{  F\lp(q_i\rp) \cdot \eps} \rp) \\
&\leq
\int_{q=0}^{\infty} \min \lp( \eps, \sqrt{  F\lp(q\rp) \cdot \eps} \rp) dq = Q_{F, \eps} \, ,
\end{align*}
where the first inequality is the triangle's inequality, the second inequality is by the guarantees of our estimators 
$\tilde Y(v_{i^*(t)}, q)$ (\Cref{lem:correlated-binary}) and the last inequality is by the fact that $\min \lp( \eps, \sqrt{  F\lp(q\rp) \cdot \eps} \rp)$ is monotonically decreasing. The argument for upper bounding the second term involving $Y(v, -q_i)$ is symmetric. This then gives us
$$
\sqrt{ \sum_{t=1}^T \max_{v \in \mathcal V} 
\lp( \tilde \mu(v)_t - \hat \mu_{t}(v) \rp)^2
} \leq O \lp( Q_{F, \eps} \rp).
$$
Combining this with Equation~\eqref{eq:fundamental-inequality} then allows us to conclude our proof.
\end{proof}

\subsection{Proof of \Cref{thm:block-extension}}
Then, we are ready to conclude the proof.
Essentially, we use the algorithm in \Cref{lem:correlated-product} to output the estimators $\mu(v)_t$. Then, the final output $\mu_t$ is simply the solution to the program specified in Equation~\eqref{eq:lp}.
\begin{proof}[Proof of \Cref{thm:block-extension}]
By \Cref{clm:lp-error}, with probability at least $1 - \tau$, it holds
$$
\snorm{2}{ \mu_t - \mu_t^* }
\leq O(1) \cdot \max_{v \in \mathcal V} \abs{v^T \mu_t^* - \mu(v)_t }.
$$
for all $t \in [T]$.
By \Cref{lem:correlated-product},
with probability at least $1 - \tau$
it holds 
$$
\sum_{t=1}^T  \max_{v \in \mathcal V} 
\lp( \mu(v)_t - v^T \mu^*_t  \rp)^2
\leq O(1) \cdot \int_{0}^{\infty} \min\lp(\eps, \sqrt{ \eps F(q) } \rp) dq.
$$
By union bound, the above two inequalities are simultaneously true with probability at least $1 - 2\tau$.
Condition on that, 
we then have
$$
\sum_{t=1}^T \snorm{2}{ \mu_t - \mu_t^* }^2
\leq O(1) \cdot \sum_{t=1}^T  \max_{v \in \mathcal V} 
\lp( \mu(v)_t - v^T \mu^*_t  \rp)^2
\leq O(1) \cdot \int_{0}^{\infty} \min\lp(\eps, \sqrt{ \eps F(q) } \rp) dq.
$$
Setting $\tau = 1/20$ then concludes the proof. 
\end{proof}
\section{Lower Bound Against the Filter Algorithm} \label{sec:lower-bound}

Here we establish the following result: 

\begin{lemma}
Fix $\eps \in (0, 1)$ and $T \in \mathbb Z^+$ satisfying $\log T \ll 1/\eps$.
Let $C$ be a set of samples in $\R^T$ whose mean is $\mu^*$ and whose covariance is bounded above by a constant multiple of $I$.
Then, there exists a set $X$ which is an $\eps$-corrupted version of $C$ and a sequence of subsets $X^{(T)} \subseteq X^{(T-1)} \cdots X^{(1)} \subseteq X^{(0)} = X$ satisfying
\begin{enumerate}
    \item For $t = 1 \cdots T$, the covariance of the samples in each $X^{(t)}$, after truncated to the first $t$ coordinates, is bounded above by a constant multiple of $I$. 
    \item The set $X^{(t)} \backslash X^{(t+1)}$ consists of only corrupted samples.
    \item Define $\mu \in \R^T$ as the vector such that
    $\mu_t$ equals to the $t$-th coordinate of the mean of $X^{(t)}$.
    It holds $\snorm{2}{\mu_t - \mu^*} = \Omega( \eps \log T )$.
\end{enumerate}
\end{lemma}
\begin{proof}
We state our construction for $X, X^{(1)}, \cdots, X^{(T)}$ only for $\mu^* = 0$ as one can easily obtain the constructions for other $\mu^*$ by applying a shift to all the sample points.

Consider the sets $B_1, \cdots, B_T$ each of size
$\frac{\eps}{T (1-\eps)} \cdot \abs{C}$. The set $B_{i}$ is made entirely of the point
$$ 
\sqrt{T} \; \lp(\frac{1}{i}, \frac{1}{i-1}, \cdots, \frac{1}{2}, 1, 0, \cdots, 0
\rp).
$$

We will set $X$ to be the union of $C$ and all $B_i$, 
and $X^{(t)}$ to be the union of $C$ and $\bigcup_{i=T-t}^T B_i$.
We first argue the covariance of the samples in each $X^{(t)}$, after truncated to the first $t$ coordinates, is bounded above by some constant multiples of $I$. 
Since we have the covariance of $C$ is bounded by some constant multiples of $I$, it suffices to argue for all $t$, we have
\begin{align} \label{eq:key-var-bound}
    \frac{1}{\abs{X^{(t)}}} \sum_{v \in \bigcup_{i=T-t}^{T} B_i  } v_{[t]} v_{[t]}^\top \cl \kappa \; I \, ,
\end{align}
for some constant $\kappa$.
Notice that the left hand side of Equation~\eqref{eq:key-var-bound} is exactly the matrix 
$$
\kappa_t \sum_{n=m}^{T} (1/n, 1/(n-1), \ldots, 1/(n-m+1)) (1/n, 1/(n-1), \ldots, 1/(n-m+1))^\top.
$$
for $m = T-t$ and $\kappa_t = O(1)$.
We claim the matrix is indeed bounded above by some constant multiples of $I$ and defer the proof to \Cref{clm:linear-algebra-bound}.

It is easy to see that we remove only coruppted points while going from $X^{(t)}$ to $X^{(t+1)}$ so the second property in the claim is satisfied. It suffices to show $\mu$, where $\mu_t$ is defined to be the $t$-th coordinate of the mean of $X^{(t)}$, is far from $\mu^* = 0$ in $\ell_2$ distance. 
By the definition of $
\mu$, we have
\begin{align*}
\snorm{2}{\mu}^2 
= \sum_{t=1}^T
\frac{\eps^2}{T} \; \lp(\sum_{i=1}^{T-t+1} \frac{1}{i} \rp)^2
\geq 
\frac{T}{2} \; 
\frac{\eps^2}{T} \; \lp(\sum_{i=1}^{T/2} \frac{1}{i} \rp)^2
\geq \Omega\lp( \eps^2 \log^2 T \rp) \;.
\end{align*}
This concludes the proof.
\end{proof}

\begin{lemma} \label{clm:linear-algebra-bound}
For all $m, T \in \mathbb Z^+$ such that $m < T$, 
the matrix
$$
\sum_{n=m}^{T} (1/n, 1/(n-1), \ldots, 1/(n-m+1)) (1/n, 1/(n-1), \ldots, 1/(n-m+1))^\top
$$
is bounded above by a constant multiple of $I$.
\end{lemma}
\begin{proof}
Note that the matrix in question is $BB^T$ where $B$ is the $m \times (T-m)$ matrix with entries $B_{i,j} = 1/(i+m-j)$. Therefore, it is enough to show that the singular values of $B$ are bounded. Let $N > T$ be one less than a power of $2$. By reversing the columns of $B$ and adding extra rows and columns, we get an $N\times N$ matrix $A$ with entries $A_{i,j} = 1/(i+j-1)$. We note that the singular values of $B$ are at most the singular values of $A$, so it suffices to bound the singular values of $A$. By the Perron-Frobenius theorem, the largest singular value of $A$ is equal to the eigenvalue of the unique eigenvector $v$ with non-negative entries. Note that if we replace $A$ by a matrix $A'$ which is entry-wise larger than $A$, we have that $v^\top Av \leq v^\top A' v$. Therefore, the largest singular vector of $A'$ is bigger than the largest singular vector of $A$. In particular, if we define $\{n\}$ to be the largest power of $2$ which is at most $n$, we will use the matrix
$$
(A')_{i,j} := 1/{\max(i,j)}.
$$
Let $e_i$ be the $ith$ standard basis vector. For integers $0\leq k < \log_2(N)$, we define the unit vectors
$$
v_k = 2^{-k/2} \sum_{i=2^k}^{2^{k+1}-1} e_i.
$$
We note that all of the entries of $A'$ whose row is in the support of $v_k$ and whose column is in the support of $v_\ell$ is $2^{-\max(k,\ell)}$. From this it is not hard to see that
$$
A' = \sum_{k,\ell} 2^{k/2+\ell/2-\max(k,\ell)} v_k v_\ell^\top = \sum_{k,\ell} 2^{-|k-\ell|/2} v_k v_\ell^\top.
$$
From this, we can see that $A'$ has the same singular values as the $\log_2(N+1)\times \log_2(N+1)$ matrix $\tilde A$ with $\tilde A_{k,\ell} = 2^{-|k-\ell|/2}$. However, it is easy to see that this is $O(1)$ since it is a symmetric matrix where the sum of the absolute values of the entries in each row are $O(1)$. In particular, this means that if $v$ is an eigenvector of $\tilde A$ with eigenvalue $\lambda$ we have that $\tilde A v = \lambda v$. Taking the $\ell^\infty$ norm of both sides, we have that $\lambda \|v\|_\infty = O(1) \|v\|_\infty$.

This completes our proof.
\end{proof}

\end{document}